\pdfoutput=1
\documentclass[ejs]{imsart}

\usepackage{color}
\usepackage{graphicx,subfigure,amsmath,amssymb,amsfonts,amsthm,bm,epsfig,epsf,url,alg,dsfont}
\usepackage[small,bf]{caption}

\usepackage{amscd}
\usepackage{color}
\definecolor{darkred}{RGB}{150,0,0}
\definecolor{darkgreen}{RGB}{0,150,0}
\definecolor{darkblue}{RGB}{0,0,200}

\usepackage{float}
\floatstyle{ruled}
\newfloat{algo}{thp}{lop}
\floatname{algo}{Algorithm}

\newtheorem{theorem}{Theorem}
\newtheorem{lemma}{Lemma}
\newtheorem{proposition}{Proposition}

\renewcommand{\thetable}{\arabic{table}}

\DeclareMathOperator{\dist}{dist}

\DeclareMathOperator{\diam}{diam}

\DeclareMathOperator{\Sp}{\mathrm{span}}

\newcommand{\bbR}{\mathbb R}

\newcommand{\beq}{\begin{equation}}
\newcommand{\eeq}{\end{equation}}

\newcommand{\pr}[1]{\mathbb{P}\left( #1 \right)}
\newcommand{\expect}[1]{\mathbb{E}\left( #1 \right)}

\def\ba{\mathbf{a}}
\def\bb{\mathbf{b}}

\def\be{\mathbf{e}}

\def\bq{\mathbf{q}}
\def\br{\mathbf{r}}
\def\bs{\mathbf{s}}

\def\bu{\mathbf{u}}
\def\bv{\mathbf{v}}
\def\bw{\mathbf{w}}
\def\bx{\mathbf{x}}
\def\by{\mathbf{y}}
\def\bz{\mathbf{z}}

\def\bA{\mathbf{A}}
\def\bD{\mathbf{D}}
\def\bI{\mathbf{I}}
\def\bQ{\mathbf{Q}}
\def\bU{\mathbf{U}}
\def\bV{\mathbf{V}}
\def\bW{\mathbf{W}}
\def\bZ{\mathbf{Z}}

\def\brW{\mathbf{\mathring{W}}}
\def\brZ{\mathbf{\mathring{Z}}}

\def\rD{\mathring{D}}
\def\rG{\mathring{G}}

\def\rW{\mathring{W}}

\def\cA{\mathcal{A}}

\def\cS{\mathcal{S}}
\def\cX{\mathcal{X}}

\def\E{\mathbb{E}}

\def\eps{\epsilon}
\def\vol{{\rm vol}}

\def\dhat{\hat{d}}

\def\rhat{\hat{r}}
\def\shat{\hat{s}}
\def\tauhat{\hat{\tau}}

\def\ud{d}
\def\uK{K}
\def\ur{r}

\newcommand{\argmax}{\arg\!\max}

\renewcommand{\bdoi}[1]{}

\def\sep{{\rm sep}}

\arxiv{stat.ML/1001.1323}

\def\sN{{\scriptscriptstyle N}}
\def\sK{{\scriptscriptstyle K}}

\begin{document}

\begin{frontmatter}

\title{Spectral Clustering Based on Local Linear Approximations}
\runtitle{Higher Order Spectral Clustering}

\begin{aug}
\author{Ery Arias-Castro\thanks{corresponding author}\ead[label=e1]{eariasca@ucsd.edu}},
\address{Department of Mathematics, University of California, San Diego, \\
\printead{e1}}
\and
\author{Guangliang Chen\ead[label=e2]{glchen@math.duke.edu}},
\address{Department of Mathematics, Duke University, \\
\printead{e2}}
\and
\author{Gilad Lerman\ead[label=e3]{lerman@umn.edu}},
\address{Department of Mathematics, University of Minnesota, Twin Cities, \\
\printead{e3}}

\runauthor{Arias-Castro, Chen and Lerman}
\end{aug}

\begin{abstract}
In the context of clustering, we assume a generative model where each cluster is the result of sampling points in the neighborhood of an embedded smooth surface; the sample may be contaminated with outliers, which are modeled as points sampled in space away from the clusters.  We consider a prototype for a higher-order spectral clustering method based on the residual from a local linear approximation.  We obtain theoretical guarantees for this algorithm and show that, in terms of both separation and robustness to outliers, it outperforms the standard spectral clustering algorithm (based on pairwise distances) of Ng, Jordan and Weiss (NIPS '01).  The optimal choice for some of the tuning parameters depends on the dimension and thickness of the clusters.  We provide estimators that come close enough for our theoretical purposes.  We also discuss the cases of clusters of mixed dimensions and of clusters that are generated from smoother surfaces.  In our experiments, this algorithm is shown to outperform pairwise spectral clustering on both simulated and real data.
\end{abstract}






\end{frontmatter}

\section{Introduction}

In a number of modern applications, the data appear to cluster near some low-dimensional structures.  In the particular setting of manifold learning~\cite{Tenenbaum00ISOmap,Roweis00LLE,Belkin03,survey-kernel-spectral,DG05}, the data are assumed to lie near manifolds embedded in Euclidean space.  When multiple manifolds are present, the foremost task is separating them, meaning the recovery of the different components of the data associated with the different manifolds.
Manifold clustering naturally occurs in the human visual cortex, which excels at grouping points into clusters of various shapes~\cite{Neumann2007189,Field1993173}.  It is also relevant for a number of modern applications.  For example, in cosmology, galaxies seem to cluster forming various geometric structures such as one-dimensional filaments and two-dimensional walls~\cite{galaxy-nonrandom,MarSaa}. In motion segmentation, feature vectors extracted from moving objects and tracked along different views cluster along affine or algebraic surfaces~\cite{Ma07,1530127,vidal2006unified,AtevKSCC}.  In face recognition, images of faces in fixed pose under varying illumination conditions cluster near low-dimensional affine subspaces~\cite{Ho03,Basri03,Epstein95}, or along low-dimensional manifolds when introducing additional poses and camera views.

In the last few years several algorithms for multi-manifold clustering were introduced; we discuss them individually in Section~\ref{subsec:other_methods}.  We focus here on spectral clustering methods, and in particular, study a prototypical multiway method relying on local linear approximations, with precursors appearing in~\cite{spectral_applied,Agarwal06,Shashua06,Agarwal05,Govindu05}.  We refer to this method as Higher-Order Spectral Clustering (HOSC).  We establish theoretical guarantees for this method within a standard mathematical framework for multi-manifold clustering.  Compared with all other algorithms we are aware of, HOSC is able to separate clusters that are much closer together; equivalently, HOSC is accurate under much lower sampling rate than any other algorithm we know of.  Roughly speaking, a typical algorithm for multi-manifold clustering relies on local characteristics of the point cloud in a way that presupposes that all points, or at least the vast majority of the points, in a (small enough) neighborhood are from a single cluster, except in places like intersections of clusters.  In contrast, though HOSC is also a local method, it can work with neighborhoods where two or more clusters coexist.

\subsection{Higher-Order Spectral Clustering (HOSC)}
\label{sec:spectral}

We introduce our higher-order spectral clustering algorithm in this section, tracing its origins to the spectral clustering algorithm of Ng et al.~\cite{Ng02} and the spectral curvature clustering of Chen and Lerman~\cite{spectral_applied,spectral_theory}.

Spectral methods are based on building a neighborhood graph on the data points and partitioning the graph using its Laplacian~\cite{survey-kernel-spectral,1288832}, which is closely related to the extraction of connected components.  The version introduced by Ng et al.~\cite{Ng02} is an emblematic example---we refer to this approach as SC.  It uses an affinity based on pairwise distances.  Given a scale parameter $\eps > 0$ and a kernel $\phi$, define
\begin{equation}
\label{eq:pair-affinity}
\alpha(\bx_1,\bx_2) = \left\{\begin{array}{ll}
\phi(\|\bx_1 - \bx_2\|/\eps), & \bx_1 \neq \bx_2; \\
0, & \bx_1 = \bx_2.
\end{array} \right.
\end{equation}
($\| \cdot \|$ denotes the Euclidean norm.)
Standard choices include the heat kernel $\phi(s) = \exp(-s^2)$ and the simple kernel $\phi(s) = {\bf 1}_{\{|s| < 1\}}$.
Let $\bx_1, \dots, \bx_N \in \bbR^D$ denote the data points.
SC starts by computing all pairwise affinities $\bW = (W_{ij})$, with $W_{ij} = \alpha(\bx_i, \bx_j)$, for $i,j = 1, \dots, N$.  It then computes the matrix $\mathbf{Z} = (Z_{ij}): Z_{ij} = W_{ij}/(D_i D_j)^{1/2}$, where $D_i = \sum_{1\leq j\leq N} W_{ij}$ is the degree of the $i$th point in the graph with similarity matrix $\bW$.  Note that $\bI - \bZ$ is the corresponding normalized Laplacian.  Providing the algorithm with the number of clusters $K$, SC continues by extracting the top $\uK$ eigenvectors of $\bZ$, obtaining a matrix $\bU \in \bbR^{N \times K}$, and after normalizing its rows, uses them to embed the data into $\bbR^K$.  The algorithm concludes by applying $K$-means to the embedded points.  See Algorithm~\ref{algo:NJW} for a summary.

\begin{algo}[ht]
\caption{Spectral Clustering (SC)~\cite{Ng02} }
\centering
\begin{tabular}{p{4in}}
\textbf{Input:}\\
\hspace*{.3in} $\bx_1,\bx_2,...,\bx_\sN$: the data points\\
\hspace*{.3in} $\eps$: the affinity scale\\
\hspace*{.3in} $\uK$: the number of clusters\\

\textbf{Output:}\\
\hspace*{.3in} A partition of the data into $\uK$ disjoint clusters\\[.1in]

\textbf{Steps:}\\
{\bf 1:} Compute the affinity matrix $\bW = (W_{ij})$, with $W_{ij} = \alpha(\bx_i, \bx_j)$. \\
{\bf 2:} Compute the $\mathbf{Z} = (Z_{ij}): Z_{ij} = W_{ij}/(D_i D_j)^{1/2}$, where $D_i = \sum_j W_{ij}$. \\
{\bf 3:} Extract $\bU = [\bu_1, \dots, \bu_{\uK}]$, the top $\uK$ eigenvectors of $\bZ$. \\
{\bf 4:} Renormalize each {\it row} of $\mathbf{U}$ to have unit norm, obtaining a matrix $\bV$. \\
{\bf 5:} Apply $K$-means to the row vectors of $\mathbf{V}$ in $\mathbb{R}^{\uK}$ to find $\uK$ clusters. \\
{\bf 6:} Accordingly group the original points into $\uK$ disjoint clusters. \\
\end{tabular}
\label{algo:NJW}
\end{algo}


Spectral methods utilizing multiway affinities were proposed to better exploit additional structure present in the data.
The spectral curvature clustering (SCC) algorithm of Chen and Lerman~\cite{spectral_applied, spectral_theory} was designed for the case of hybrid linear modeling where the manifolds are assumed to be affine, a setting that arises in motion segmentation~\cite{Ma07}.  Assuming that the subspaces are all of dimension $d$---a parameter of the algorithm, SCC starts by computing the (polar) curvature of all $(d+2)$-tuples, creating an $N^{\otimes (d+2)}$-tensor.  The tensor is then flattened into a matrix $\bA$ whose product with its transpose, $\bW=\bA\bA'$, is used as an affinity matrix for the spectral algorithm SC.
(In practice, the algorithm is randomized for computational tractability.)  Kernel spectral curvature clustering (KSCC)~\cite{AtevKSCC} is a kernel version of SCC designed for the case of algebraic surfaces.

The SCC algorithm (and therefore KSCC) is not localized in space as it fits a parametric model that is global in nature.  The method we study here may be seen as a localization of SCC, which is appropriate in our nonparametric setting since the manifolds resemble affine surfaces locally.  This type of approach is mentioned in publications on affinity tensors~\cite{Agarwal06,Shashua06,Agarwal05,Govindu05} and is studied here for the first time, to our knowledge.  As discussed in Section~\ref{sec:discussion}, all reasonable variants have similar theoretical properties, so that we choose one of the simplest versions to ease the exposition.  Concretely, we consider a multiway affinity that combines pairwise distances between nearest neighbors and the residual from the best $\ud$-dimensional local linear approximation.  Formally, given a set of $m \geq d+2$ points, $\bx_1, \dots, \bx_{m}$, define
\begin{equation} \label{eq:Lambda}
\Lambda_{\ud}(\bx_1, \dots, \bx_{m}) = \min_{L \in \cA_{\ud}} \, \max_{j=1,\dots,m} \, \dist(\bx_j, L),
\end{equation}
where $\dist(\bx, S) := \inf_{\bs \in S} \|\bx -\bs\|$ for a subset $S \subset \bbR^D$ and $\cA_{\ud}$ denotes the set of $\ud$-dimensional affine subspaces in $\bbR^D$.
In other words, $\Lambda_{\ud}(\bx_1, \dots, \bx_{m})$ is the width of the thinnest tube (or band) around a $\ud$-dimensional affine subspace that contains $\bx_1, \dots, \bx_{m}$.
(In our implementation, we use the mean-square error; see Section~\ref{sec:numerics}.)
Given scale parameters $\eps>\eta>0$ and a kernel function $\phi$, define the following affinity: $\alpha_{\ud}(\bx_1, \dots, \bx_{m}) = 0$ if $\bx_1, \dots, \bx_{m}$ are not distinct; otherwise:
\begin{equation}
\label{eq:linear-affinity}
\alpha_{\ud}(\bx_1, \dots, \bx_{m}) =
\phi\left(\frac{\diam(\bx_1, \dots, \bx_{m})}{\eps}\right) \cdot \phi\left(\frac{\Lambda_{\ud}(\bx_1, \dots, \bx_{m})}{\eta}\right),
\end{equation}
where $\diam(\bx_1, \dots, \bx_{m})$ is the diameter of $\{\bx_1, \dots, \bx_{m}\}$.
See Figure~\ref{fig:alpha} for an illustration.

\begin{figure}[htbp]
\centering
\includegraphics[width=.30\linewidth]{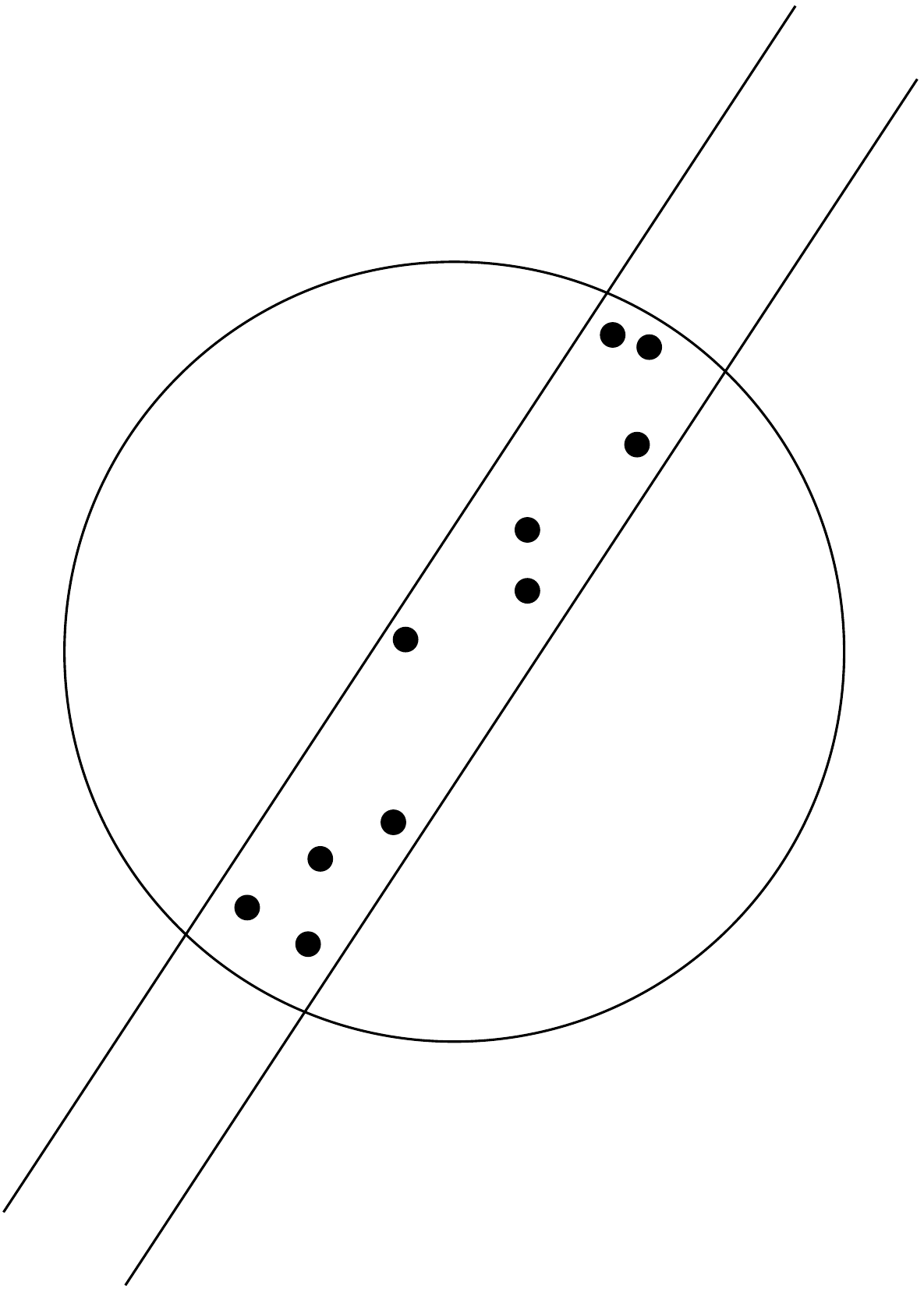} \
\includegraphics[width=.30\linewidth]{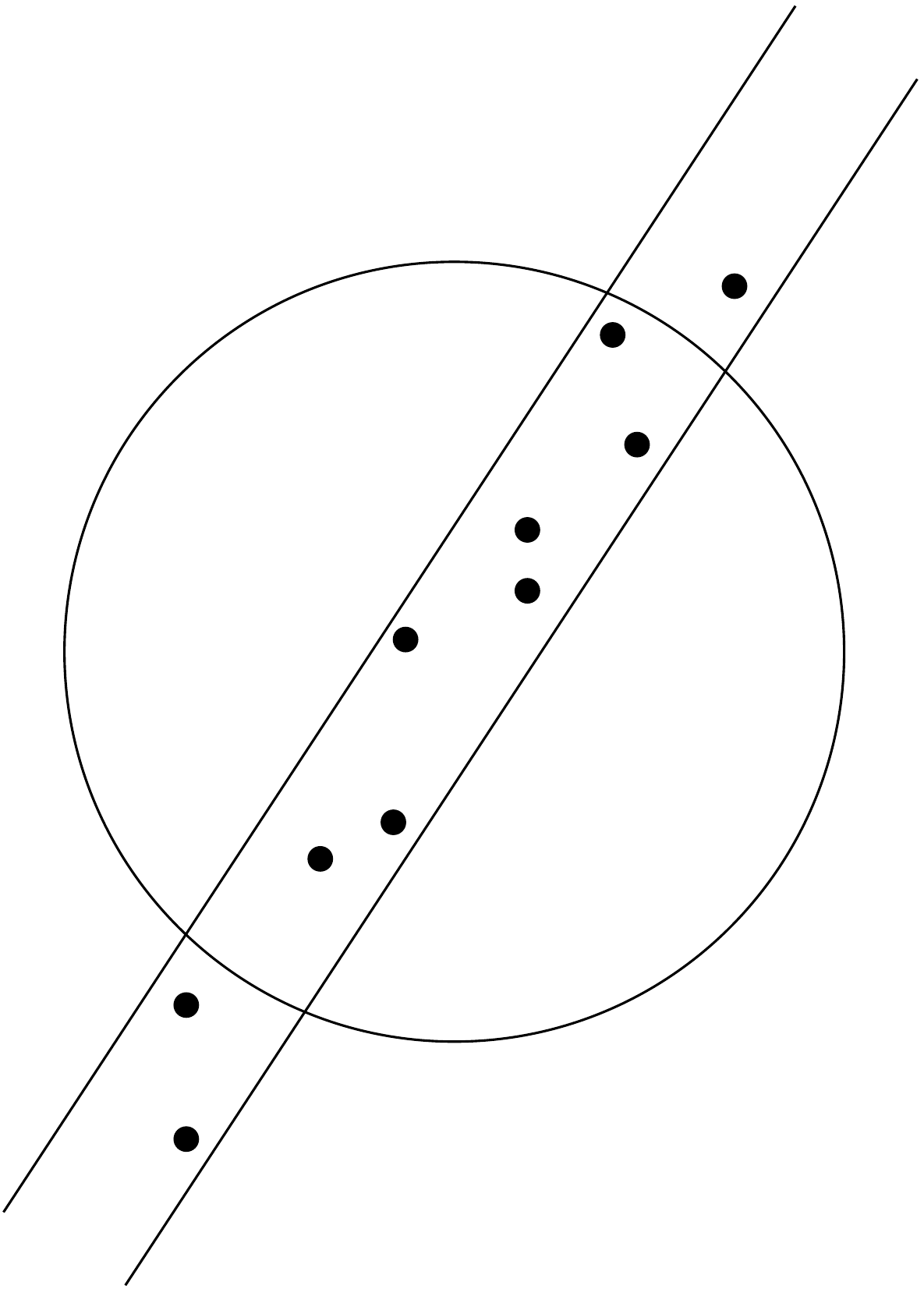} \
\includegraphics[width=.30\linewidth]{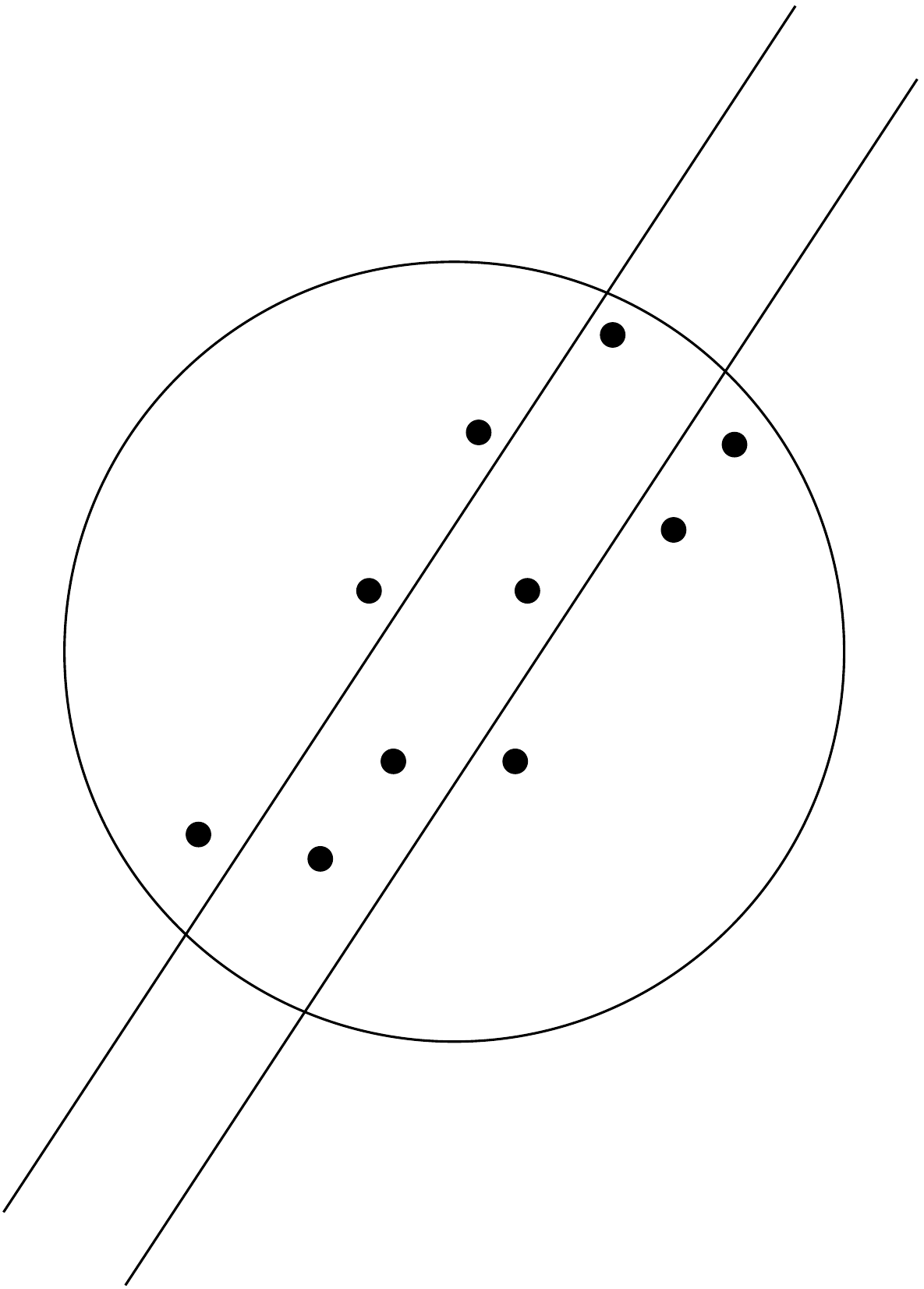}
\caption{The circle is of radius $\eps/2$ and the band is of half-width $\eta$.  Assuming we use the simple kernel, the $m$-tuple on the left has affinity $\alpha_d$ equal to one, while the other two $m$-tuples have affinity equal to zero, the first one for having a diameter exceeding $\eps$ and the second one for being `thicker' than $\eta$.}
\label{fig:alpha}
\end{figure}
%
%

Given data points $\bx_1, \dots, \bx_\sN$ and approximation dimension $\ud$, we compute all $m$-way affinities, and then obtain pairwise similarities by clique expansion~\cite{Agarwal05} (note that several other options are possible~\cite{spectral_applied,Shashua06,Govindu05}):
\begin{equation}
\label{eq:W_def}
W_{ij} = \sum_{i_1, \dots, i_{m-2}} \alpha_{\ud}(\bx_{i}, \bx_{j}, \bx_{i_1},\ldots,\bx_{i_{m-2}}).
\end{equation}
Though it is tempting to choose $m$ equal to $\ud+2$, a larger $m$ allows for more tolerance to weak separation and small sampling rate.  The down side is what appears to be an impractical computational burden, since the mere computation of $\bW$ in \eqref{eq:W_def} requires order $O(N^m)$ flops.  In Section~\ref{sec:complexity}, we discuss how to reduce the computational complexity to $O(N^{1 + o(1)})$ flops, essentially without compromising performance.

Once the affinity matrix $\bW$ is computed, the SC algorithm is applied.  We call the resulting procedure higher-order spectral clustering (HOSC), summarized in Algorithm~\ref{algo:linear}.
Note that HOSC is (essentially) equivalent to SC when $\eta \geq \eps$, and equivalent to SCC when $\eps = \infty$.

\begin{algo}[ht]
\caption{Higher Order Spectral Clustering (HOSC)}
\centering
\begin{tabular}{p{4in}}
\textbf{Input:}\\
\hspace*{.3in} $\bx_1,\bx_2,...,\bx_\sN$: the data points\\
\hspace*{.3in} $\ud, m$: the approximation dimension and affinity order \\
\hspace*{.3in} $\eps, \eta$: the affinity scales\\
\hspace*{.3in} $\uK$: the number of clusters\\

\textbf{Output:}\\
\hspace*{.3in} A partition of the data into $\uK$ disjoint clusters\\[.1in]

\textbf{Steps:}\\
{\bf 1:} Compute the affinity matrix $\bW = (W_{ij})$ according to (\ref{eq:W_def}). \\
{\bf 2:} Apply SC (Algorithm~\ref{algo:NJW}).
\end{tabular}
\label{algo:linear}
\end{algo}

\subsection{Generative Model}
\label{sec:setting}

It is time to introduce our framework.  We assume a generative model where the clusters are the result of sampling points near surfaces embedded in an ambient Euclidean space, specifically, the $D$-dimensional unit hypercube $(0,1)^D$.
%
For a surface $S \subset (0,1)^D$ and $\tau > 0$, define its $\tau$-neighborhood as
\[
B(S, \tau) = \{\bx \in (0,1)^D: \dist(\bx, S) < \tau\}. 
\]
The reach of $S$ is the supremum over $\tau > 0$ such that, for each $\bx \in B(S, \tau)$, there is a unique point realizing $\inf\{\|\bx - \bs\|: \bs \in S\}$~\cite{MR0110078}.  It is well-known that, for $C^2$ submanifolds, the reach bounds the radius of curvature from below~\cite[Lem.~4.17]{MR0110078}.
For a connection to computational geometry, the reach coincides with the condition number introduced in~\cite{1349695} for submanifolds without boundary.
Let $\vol_d(S)$ denote the $d$-dimensional Hausdorff measure, 
and $\partial S$ the boundary of $S$ within $(0,1)^D$.
For an integer $1 \leq d \leq D-1$ and a constant $\kappa \geq 1$, let $\cS_d^2(\kappa)$ be the class of $d$-dimensional, connected, $C^2$ submanifolds $S \subset (0,1)^D$ of $1/\kappa \leq \diam(S) \leq \kappa$ and ${\rm reach}(S) \geq 1/\kappa$, and if $S$ has a boundary, $\partial S$ is a $(d-1)$-dimensional $C^2$ submanifold with ${\rm reach}(\partial S) \geq 1/\kappa$.
Given surfaces $S_1, \dots, S_\sK \in \cS_d^2(\kappa)$ and $\tau < 1/\kappa$, we generate clusters $\cX_1, \dots, \cX_\sK$ by sampling $N_k$ points uniformly at random in $B(S_k, \tau)$, the $\tau$-neighborhood of $S_k$ in $(0,1)^D$, for all $k = 1, \dots, K$.  We call $\tau$ the jitter level.
Except for Section~\ref{sec:intersect}, where we allow for intersections, we assume that the surfaces are separated by a distance of at least $\delta \geq 0$, i.e.
\begin{equation} \label{eq:delta}
\dist(S_k, S_\ell) := \inf_{\bx \in S_k} \inf_{\by \in S_\ell} \|\bx - \by\| \geq \delta, \quad \forall k \neq \ell.
\end{equation}
In that case, by the triangle inequality, the actual clusters are separated by at least $\delta - 2 \tau$, i.e.
\[
\dist(\cX_k, \cX_\ell) \geq \delta -2\tau.
\]
We assume that the clusters are comparable in size by requiring that $N_k \leq \zeta N_\ell$ for all $k \neq \ell$, for some finite constant $\zeta$.  Let $\bx_1, \dots, \bx_\sN$ denote the data points thus generated.  See Figure~\ref{fig:setting} for an illustration.

\begin{figure}[htbp]
\centering
\includegraphics[width=.45\linewidth]{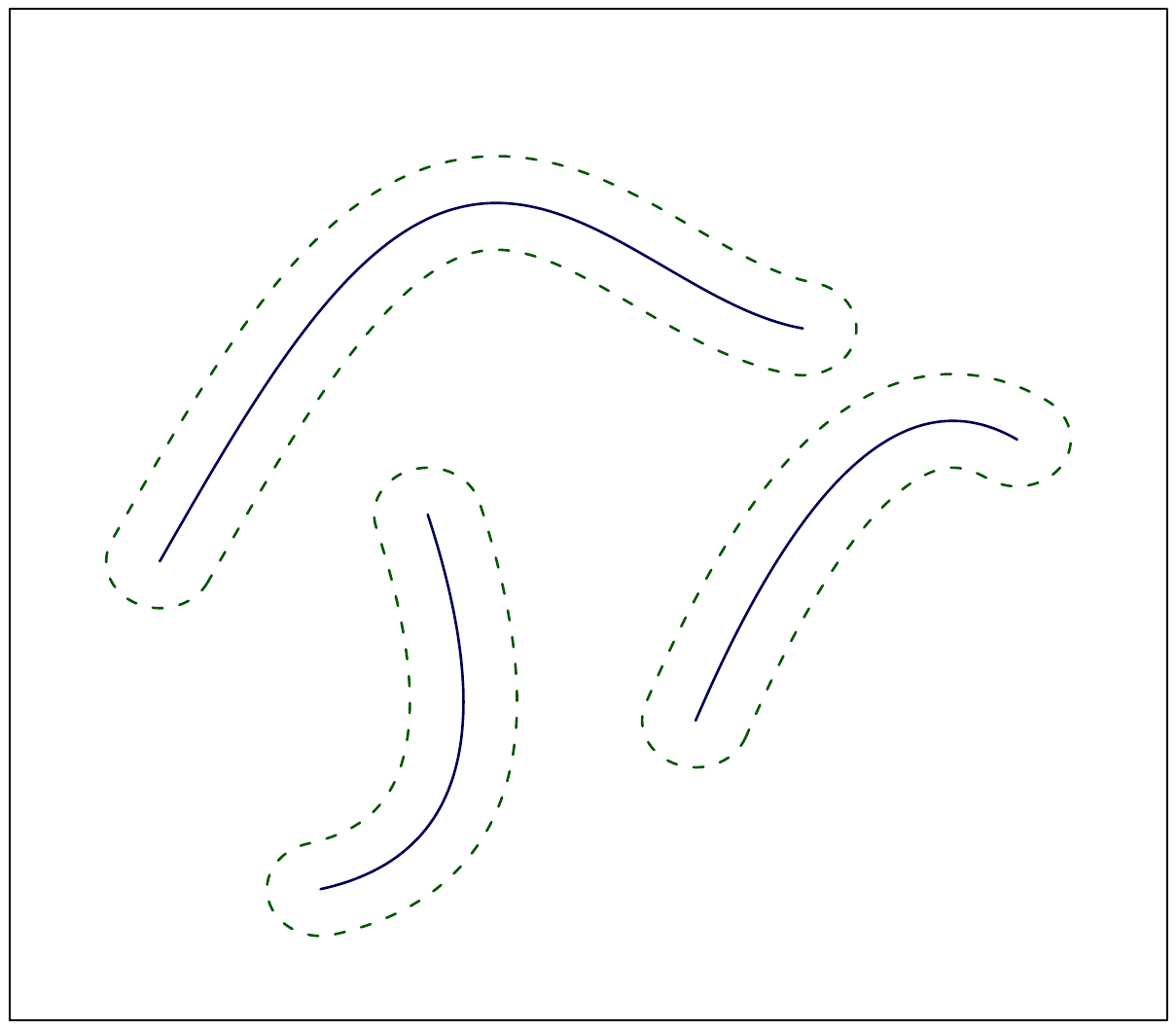} \
\includegraphics[width=.45\linewidth]{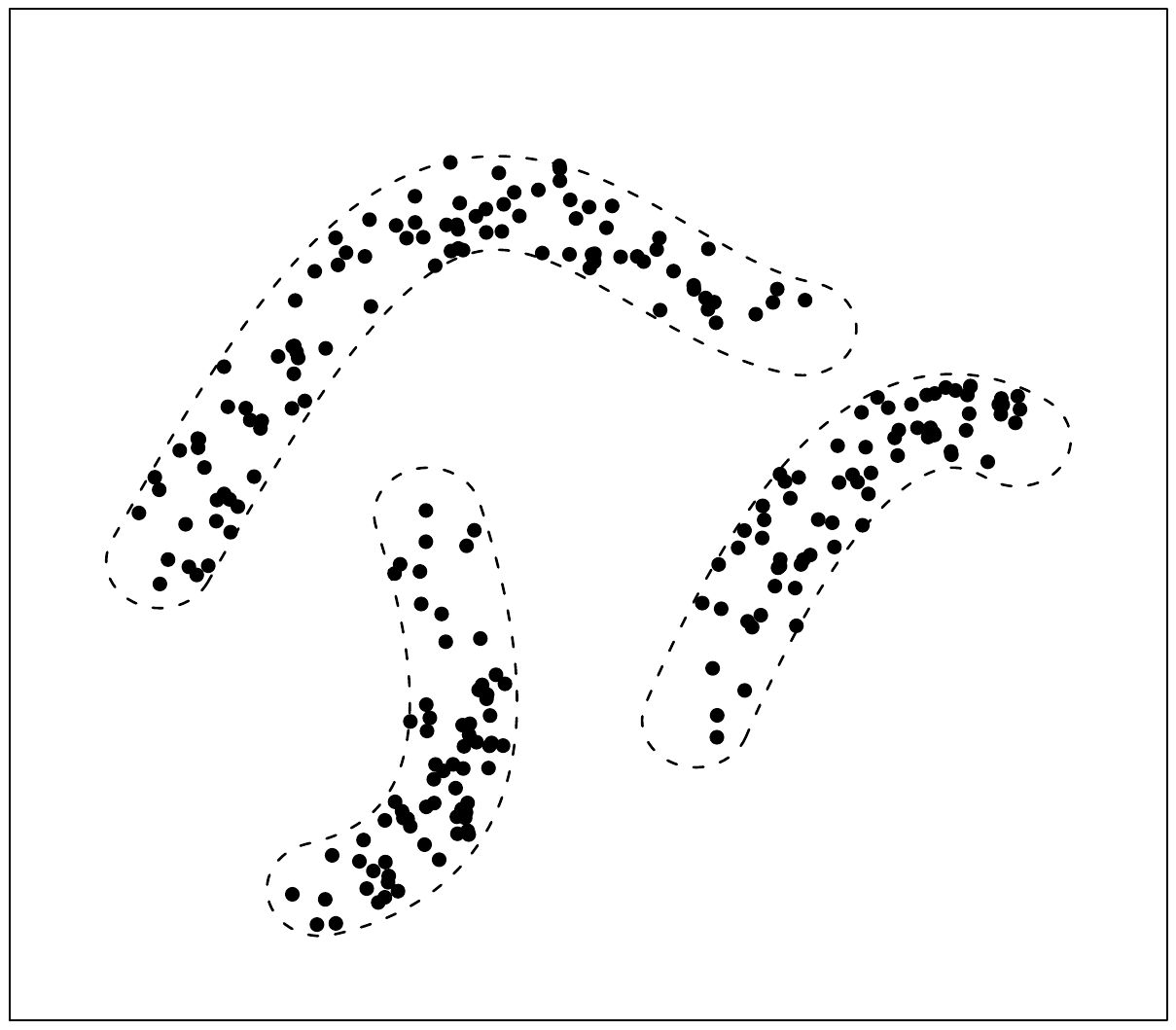} \
\vspace{-.3in}
\caption{This figure illustrates the generative model.  Left: Three surfaces (here curves) with their $\tau$-neighborhood.  The curves are separated by at least $\delta$.  Right: Points sampled within the tubular neighborhoods of the surfaces.}
\label{fig:setting}
\end{figure}

Given data $\cX := \{\bx_1, \dots, \bx_\sN\}$, we aim at recovering the clusters $\cX_1, \dots, \cX_\sK$.
Formally, a clustering algorithm is a function taking data $\cX$, and possibly other tuning parameters, and outputs a partition of $\cX$.  We say that it is `perfectly accurate' if the output partition coincides with the original partition of $\cX$ into $\cX_1, \dots, \cX_\sK$.
Our main focus is on relating the sample size $N$ and the separation requirement in \eqref{eq:delta} (in order for HOSC to cluster correctly), and in particular we let $\tau$ and $\delta$ vary with $N$.  This dependency is left implicit.  In contrast, we assume that $d, K, \zeta$ are fixed.  Also, we assume that $d, \tau, K$ are known throughout the paper (except for Section~\ref{sec:param} where we consider their estimation).
Though our setting is already quite general, we discuss some important extensions in Section~\ref{sec:discussion}.

We will also consider the situation where outliers may be present in the data.  By outliers we mean points that were not sampled near any of the underlying surfaces.  We consider a simple model where outliers are points sampled uniformly in $(0,1)^D \setminus \bigcup_k B(S_k, \delta_0)$ for some $\delta_0 > 0$, in general different from $\delta$.  That is, outliers are at least a distance $\delta_0$ away from the surfaces.  We let $N_0$ denote the number of outliers, while $N$ still denotes the total number of data points, including outliers.  See Figure~\ref{fig:outliers} for an illustration.

\begin{figure}[htbp]
\centering
\includegraphics[width=.5\linewidth]{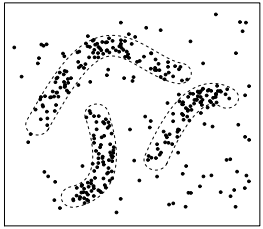} \
\vspace{-.3in}
\caption{This figure illustrates the generative model with outliers included in the data.}
\label{fig:outliers}
\end{figure}

\subsection{Performance in terms of Separation and Robustness}

\subsubsection{Performance of SC}

A number of papers analyze SC under generative models similar to ours~\cite{pairwise,vonLuxburg08,pelletier11,LDS_NIPS_06}, and the closely related method of extracting connected components of the neighborhood graph~\cite{pairwise,1519716,brito,maier2007cluster}.  The latter necessitates a compactly supported kernel $\phi$ and may be implemented via a union-of-balls estimator for the support of the density~\cite{MR579432}.
Under the weaker (essentially Lipschitz) regularity assumption
\begin{equation} \label{eq:S-vol}
C^{-1}\, \eps^d \leq \vol_d(B(\bs, \eps) \cap S) \leq C\, \eps^d, \quad \forall \eps \in (0, 1/C), \, \forall \bs \in S,
\end{equation}
Arias-Castro~\cite{pairwise} shows that SC with a compactly supported kernel is accurate if
\begin{equation} \label{eq:pair-sep}
\delta - 2 \tau \gg \sep_\sN := \left(\frac{\log N}{N}\right)^{1/d} \vee \, \tau^{1 - d/D} \left(\frac{\log N}{N}\right)^{1/D}.
\end{equation}
($a \vee b$ denotes the maximum of $a$ and $b$ and $a_\sN \gg b_\sN$ if $a_\sN/b_\sN \to \infty$ as $N \to \infty$).
With the heat kernel, the same result holds up to a $\sqrt{\log N}$ multiplicative factor.
See also~\cite{1519716,maier2007cluster}, which prove a similar result for the method of extracting connected components under stronger regularity assumptions.  At the very least, \eqref{eq:pair-sep} is necessary for the union-of-balls approach and for SC with a compactly supported kernel, because $\sep_\sN$ is the order of magnitude of the largest distance between a point and its closest neighbor from the same cluster~\cite{penrose}.
%
Note that \eqref{eq:S-vol} is very natural in the context of clustering as it prevents $S$ from being too narrow in some places and possibly confused with two or more disconnected surfaces.  And, when $C$ in \eqref{eq:S-vol} is large enough and $\kappa$ is small enough, it is satisfied by any surface $S$ belonging to $\cS_d^2(\kappa)$.  Indeed, such a surface resembles an affine subspace locally and \eqref{eq:S-vol} is obviously satisfied for an affine surface.

When outliers may be present in the data, as a preprocessing step, we identify as outliers data points with low connectivity in the graph with affinity matrix $\bW$, and remove these points from the data before proceeding with clustering.  (This is done between Steps 1 and 2 in Algorithm~\ref{algo:NJW}.)   In the context of spectral clustering, this is very natural; see, e.g.,~\cite{spectral_applied,1519716,pairwise}.
Using the pairwise affinity \eqref{eq:pair-affinity}, outliers are properly identified if $\delta_0 -\tau$ satisfies the lower bound in \eqref{eq:pair-sep} and if the sampling is dense enough, specifically~\cite{pairwise},
\begin{equation} \label{eq:N-cond-lb}
N_k \geq (N^{d/D} \vee N \tau^{D-d}) \log(N), \quad \forall k=1,\dots,K.
\end{equation}
When the surfaces are only required to be of Lipschitz regularity as in \eqref{eq:S-vol}, we are not aware of any method that can even detect the presence of clusters among outliers if the sampling is substantially sparser.

\subsubsection{Performance of HOSC}

Methods using higher-order affinities are obviously more complex than methods based solely on pairwise affinities.  Indeed, HOSC depends on more parameters and is computationally more demanding than SC.  One, therefore, wonders whether this higher level of complexity is justified.  We show that  HOSC does improve on SC in terms of clustering performance, both in terms of required separation between clusters and in terms of robustness to outliers.

Our main contribution in this paper is to establish a separation requirement for HOSC which is substantially weaker than~\eqref{eq:pair-sep} when the jitter $\tau$ is small enough.  Specifically, HOSC operates under the separation
\begin{equation} \label{eq:sep}
\delta - 2 \tau \gg (\tau \wedge \sep_\sN) \vee \sep_\sN^2,
\end{equation}
where $a \wedge b$ denotes the minimum of $a$ and $b$, and $\sep_\sN$ is the separation required for SC with a compactly supported kernel, defined in \eqref{eq:pair-sep}.  This is proved in Theorem~\ref{th:linear} of Section~\ref{sec:same}.
In particular, in the jitterless case (i.e.~$\tau = 0$), the magnitude of the separation required for HOSC is (roughly) the square of that for SC at the same sample size; equivalently, at a given separation, HOSC requires (roughly) the square root of the sample size needed by SC to correctly identify the clusters.

\begin{figure}[htbp]
\centering
\includegraphics[width=.32\linewidth]{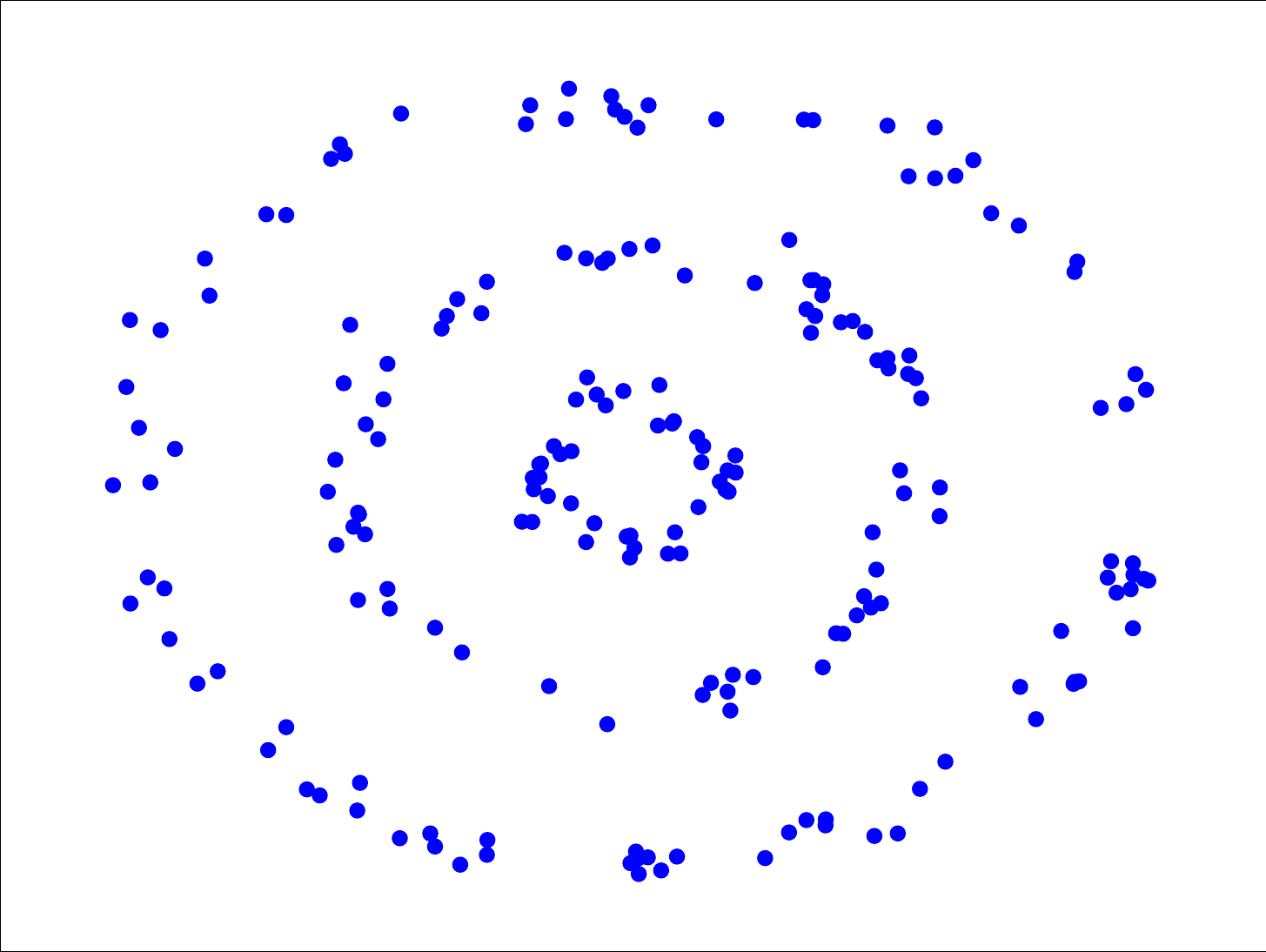}
\includegraphics[width=.32\linewidth]{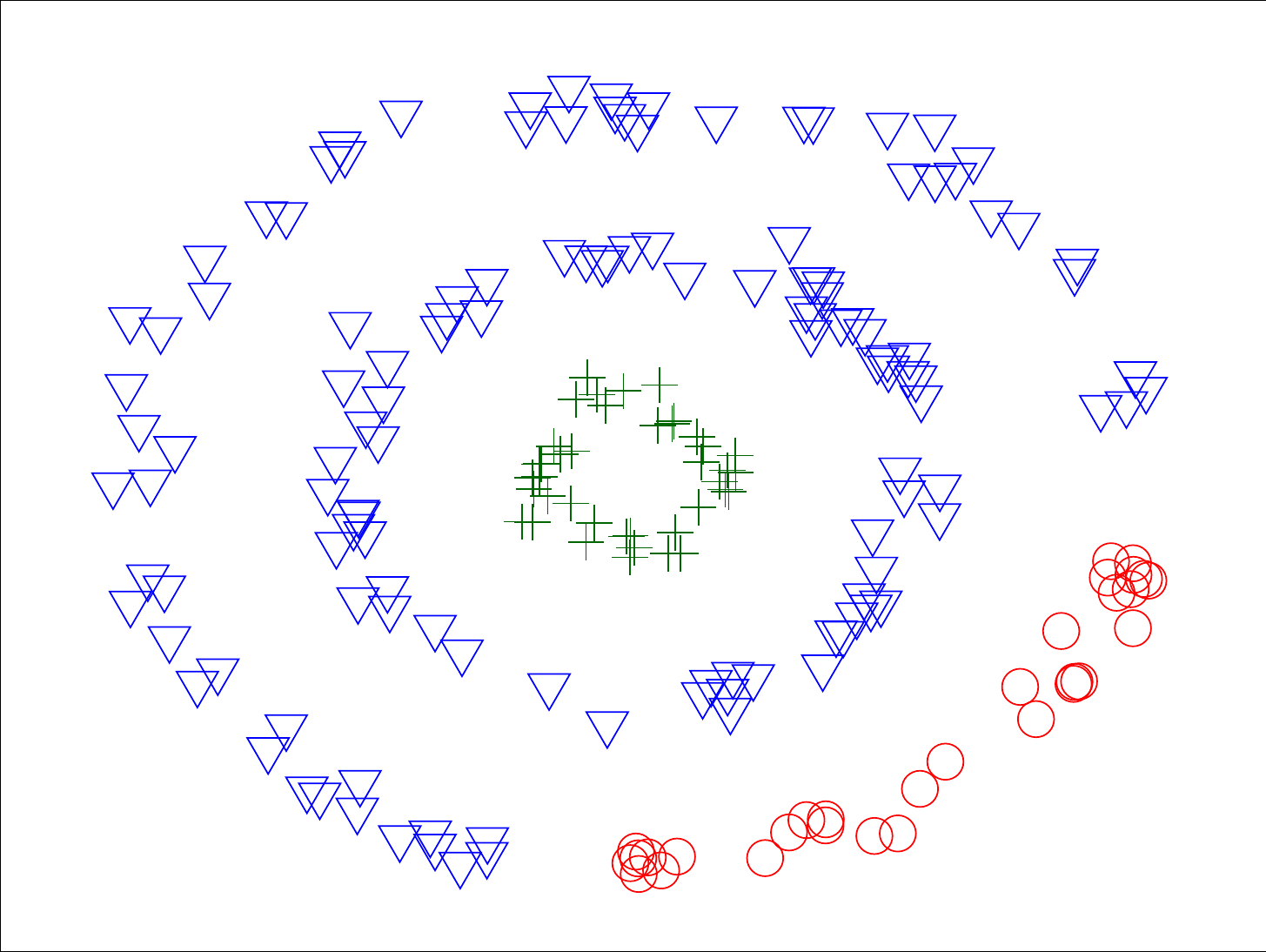}
\includegraphics[width=.32\linewidth]{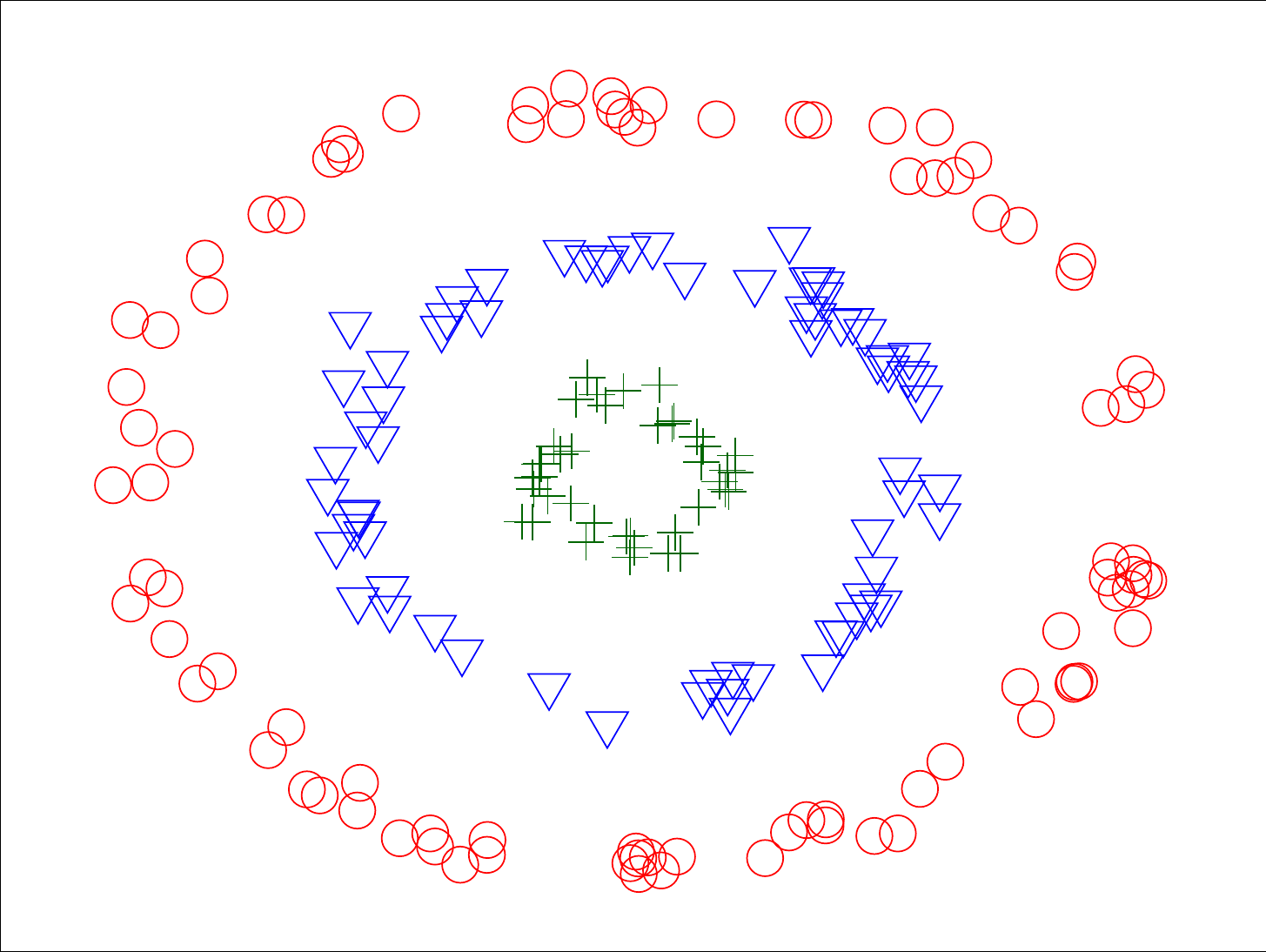}
\caption{Left: data.  Middle:  output from SC.  Right: output from HOSC.  The sampling is much sparser than in the original paper of Ng et al.~\cite{Ng02}, which is why SC fails. This figure is part of Figure~\ref{fig:artificial_data} in Section~\ref{sec:numerics}, which displays more numerical experiments. }
\label{fig:sep}
\end{figure}

That HOSC requires less separation than SC is also observed numerically.  In Figure~\ref{fig:sep} we compare the outputs of SC and HOSC on the emblematic example of concentric circles given in~\cite{Ng02} (here with three circles).
While the former fails completely, the latter is perfectly accurate.
Indeed, SC requires that the majority of points in an $\eps$-ball around a given data point come from the cluster containing that point.  In contrast, HOSC is able to properly operate in situations where the separation between clusters is so small, or the sampling rate is so low, that any such neighborhood is empty of data points except for the one point at the center.
To further illustrate this point, consider the simplest possible setting consisting of two parallel line segments in dimension $D = 2$, separated by a distance $\delta > 0$, specifically, $S_1 := \{(t,0): t \in [0,1]\}$ and $S_2 := \{(t,\delta): t \in [0,1]\}$.  Suppose $N/2$ points are sampled uniformly on each of these line segments.  It is well-known that the typical distance between a point on $S_k$ and its nearest neighbor on $S_k$ is of order $O(1/N)$; see~\cite{penrose}.  Hence, a method computing local statistics requires neighborhoods of radius at least of order $1/N$, for otherwise some neighborhoods are empty.  From \eqref{eq:sep}, HOSC is perfectly accurate when $\delta = (\log N)^3/N^2$, say.  When the separation $\delta$ is that small, typical ball of radius of order $1/N$ around a data point contains about as many points from $S_1$ as from $S_2$ (thus SC cannot work). See Figure~\ref{fig:sep_small} for an illustration.

\begin{figure}[htbp]
\centering
\includegraphics[width=.48\linewidth]{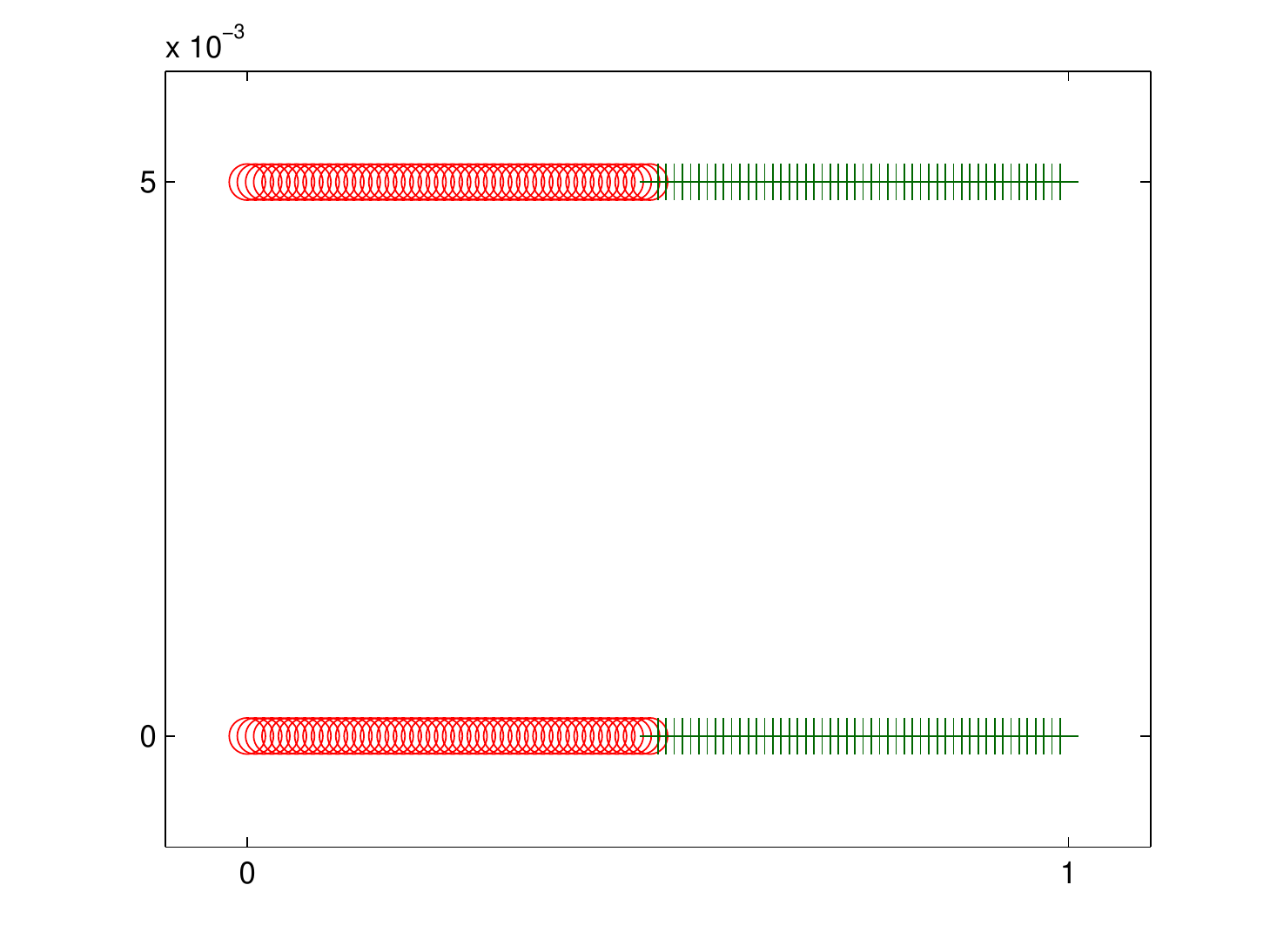}
\includegraphics[width=.48\linewidth]{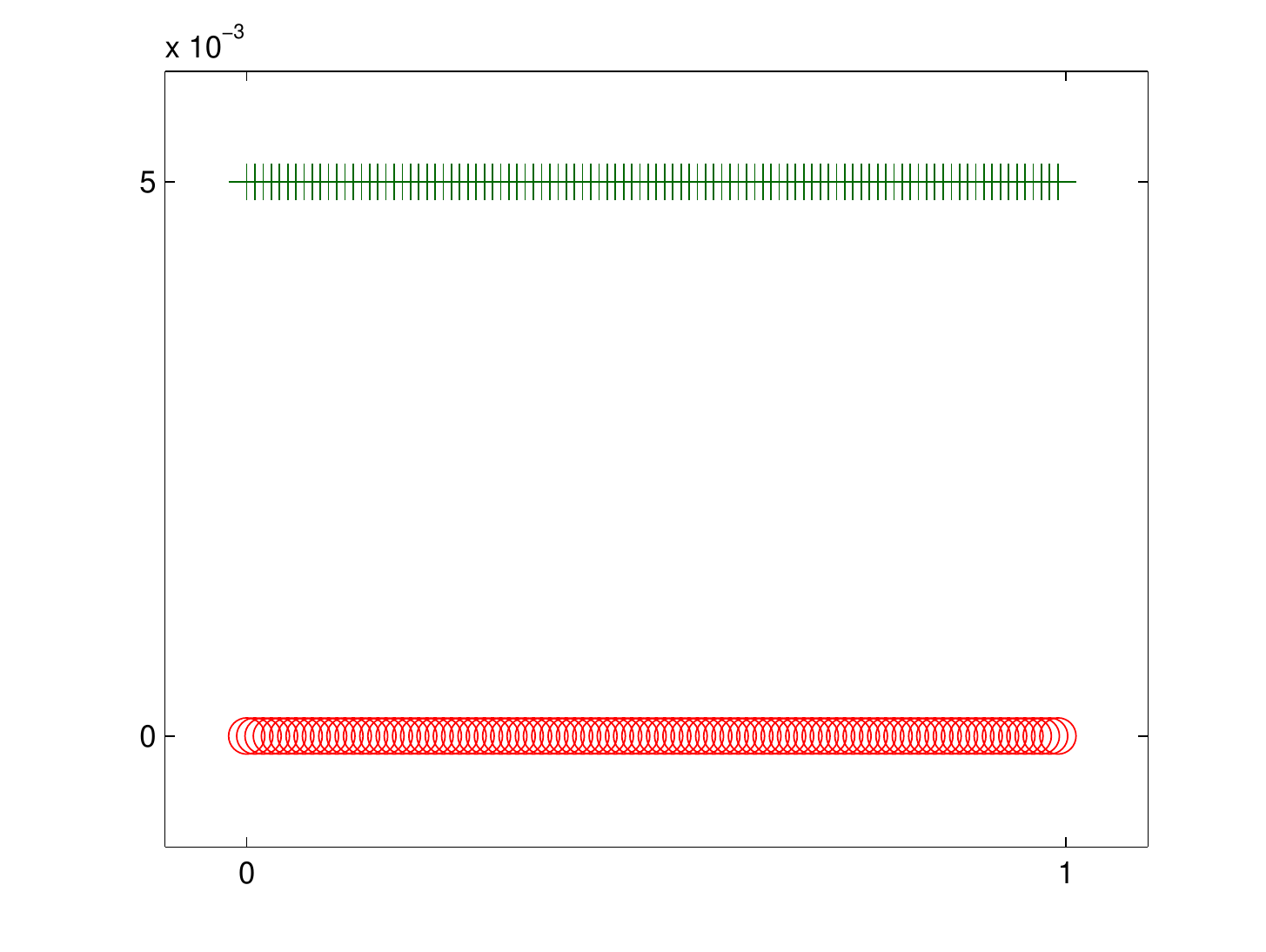}
\caption{Clustering results obtained by SC (left) and HOSC (right) on a data set of two lines with small separation ($\delta=0.005$). 100 points are sampled from each line, equally spaced (at a distance $0.01$).  Note that the inter-point separation on the same cluster is twice as large as the separation between clusters. In this case, SC cannot separate the two lines correctly, as we have argued. In contrast, HOSC performs perfectly when clustering the data, which again agrees with the theory and our expectation. We have also tried increasing the separation $\delta$ from $0.005$ to $0.025$, in which case both SC and HOSC perform correctly.}
\label{fig:sep_small}
\end{figure}

As a bonus, we also show that HOSC is able to resolve intersections in some (very) special cases, while SC is incapable of that.  See Proposition~\ref{prop:intersect} and also Figure~\ref{fig:inter-sim}.

To make HOSC robust to outliers, we do exactly as described above, identifying outliers as data points with low connectivity in the graph with affinity matrix $\bW$, this time computed using the multiway affinity \eqref{eq:linear-affinity}.  The separation and sampling requirements are substantially weaker than \eqref{eq:N-cond-lb}, specifically, $\delta_0 -\tau$ is required to satisfy the lower bound in \eqref{eq:sep} and the sampling
\begin{equation} \label{eq:N-cond-lb-linear}
N_k \gg (N^{d/(2D -d)} \vee N \tau^{D-d}) \log(N), \quad \forall k=1,\dots,K.
\end{equation}
This is established in Proposition~\ref{prop:linear-outliers-2}, and again, we are not aware of any method for detection that is reliable when the sampling is substantially sparser.
For example, when $\tau = 0$ and we are clustering curves ($d=1$) in the plane ($D=2$) (with background outliers), the sampling requirement in \eqref{eq:N-cond-lb} is roughly $N_k \gg N^{1/2} \log(N)$, compared to $N_k \gg N^{1/3} \log(N)$ in \eqref{eq:N-cond-lb-linear}.
In Figure~\ref{fig:sep_out} below we compare both SC and HOSC on outliers detection, using the data in Figure~\ref{fig:sep} but further corrupted with $33.3\%$ outliers.

\begin{figure}[htbp]
\centering
\includegraphics[width=.32\linewidth]{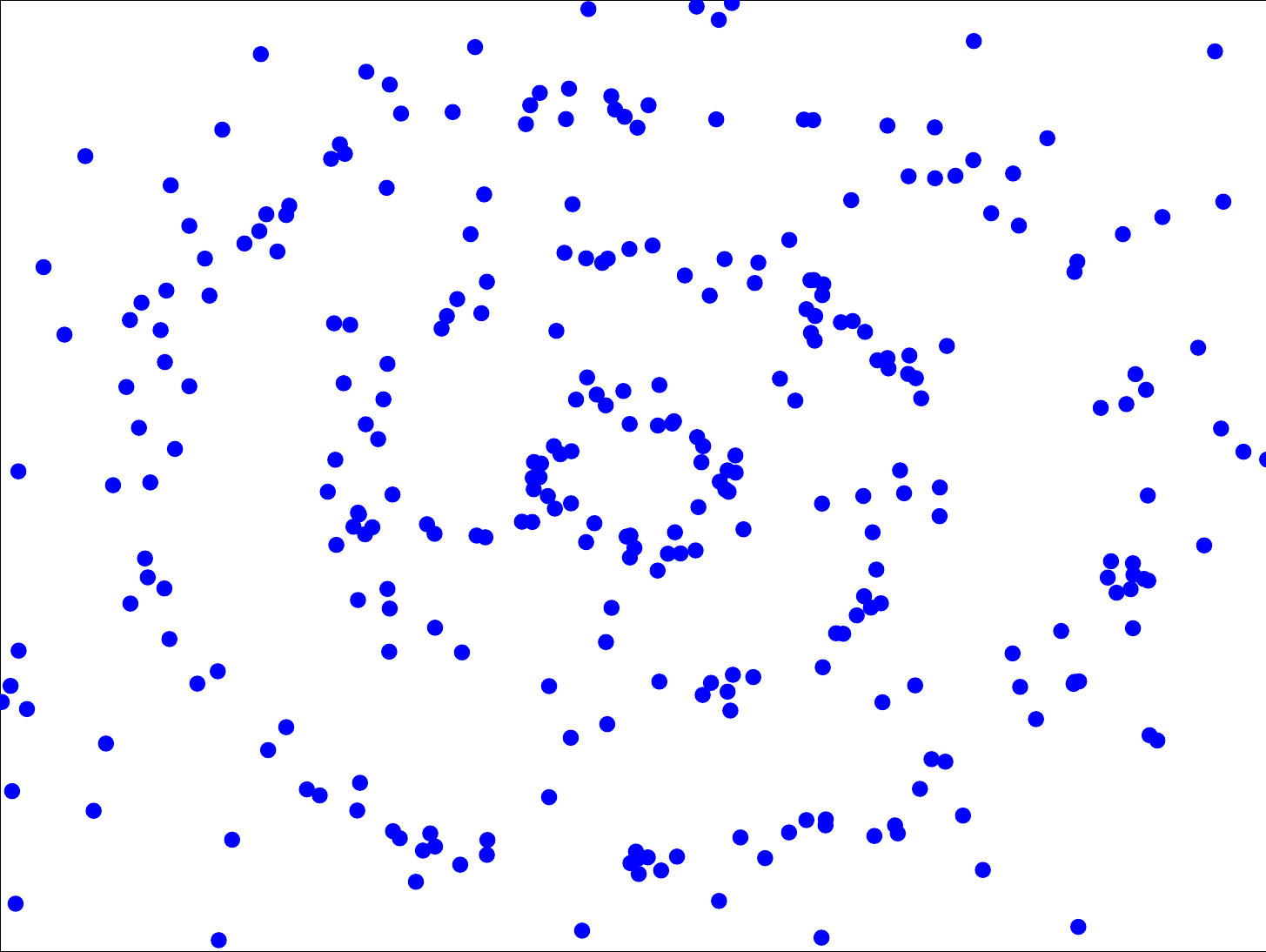}
\includegraphics[width=.32\linewidth]{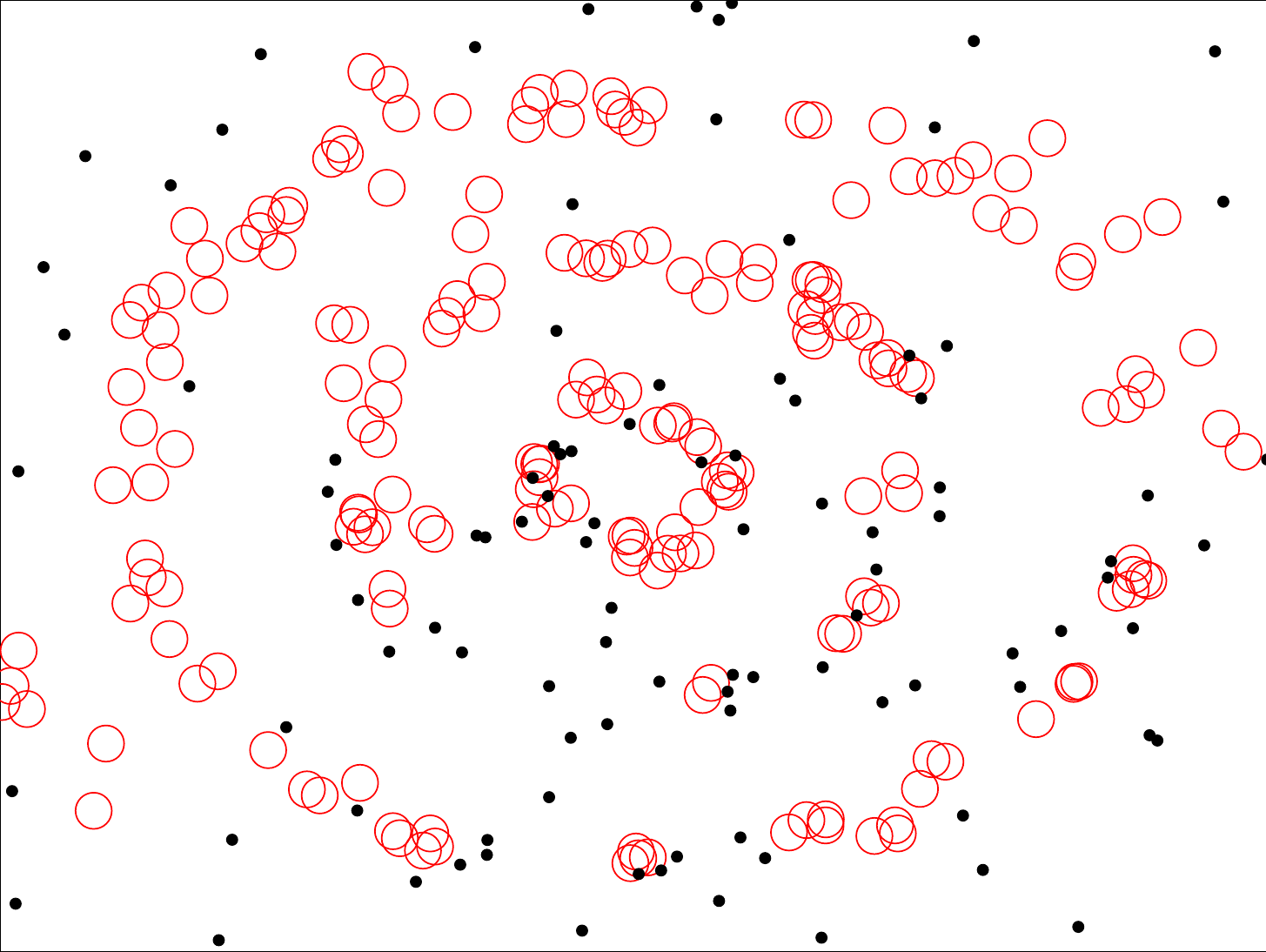}
\includegraphics[width=.32\linewidth]{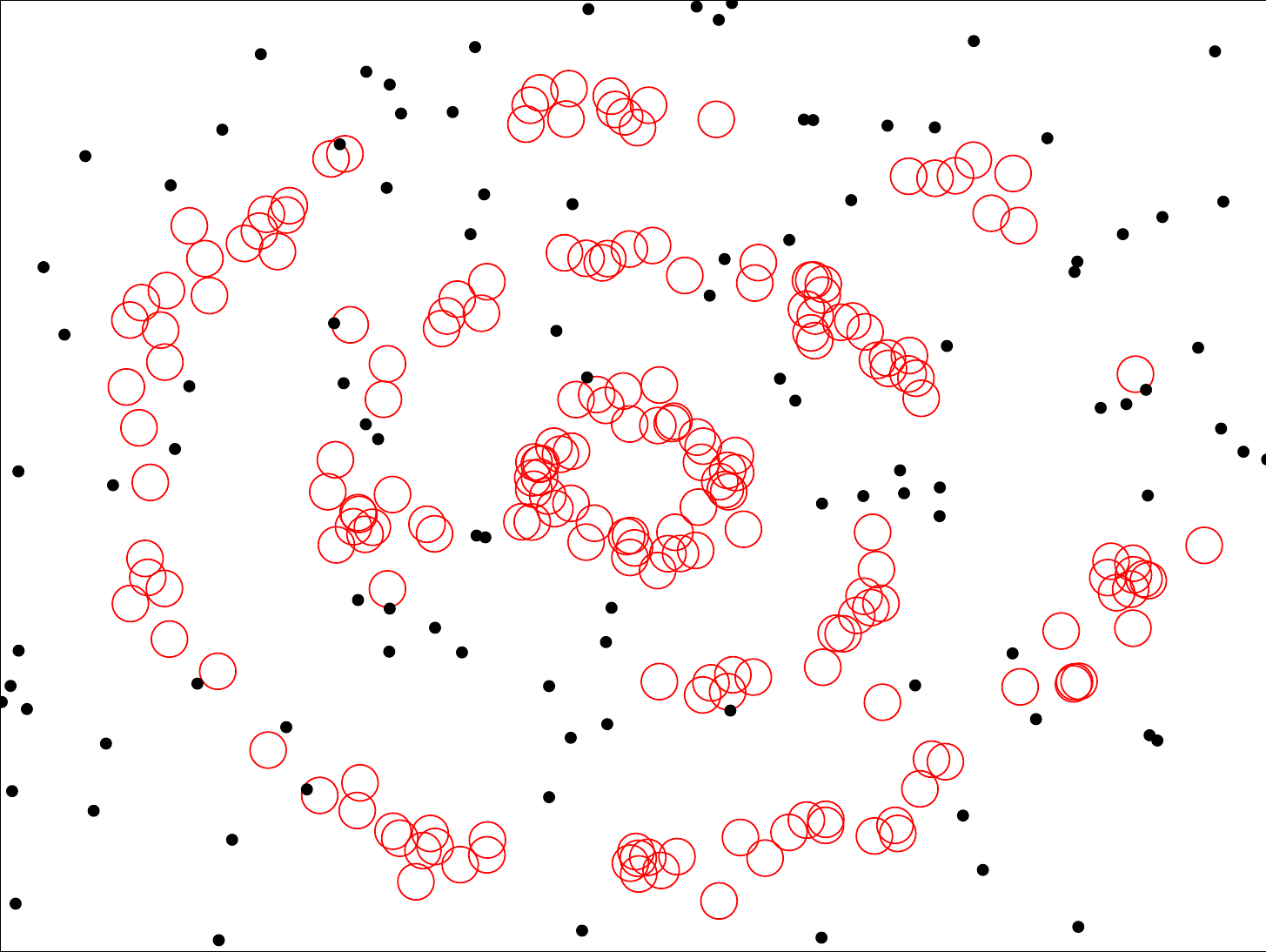}
\caption{Left: data with outliers. Middle: outliers (black dots) detected by SC. Right: outliers (black dots) detected by HOSC. This figure is part of Figure~\ref{fig:artificial_data_outliers} in Section~\ref{sec:numerics}, where more outliers-removal experiments are conducted.
}
\label{fig:sep_out}
\end{figure}

\subsubsection{Other Methods} \label{subsec:other_methods}

We focus on comparing HOSC and SC to make a strong point that higher-order methods may be preferred to simple pairwise methods when the underlying clusters are smooth and the jitter level is small.  In fact, we believe that no method suggested in the literature is able to compete with HOSC in terms of separation requirements.  We quickly argue why.

The algorithm of Kushnir et al.~\cite{kushnir} is multiscale in nature and is rather complex, incorporating local information (density, dimension and principal directions) within a soft spectral clustering approach.  In the context of semi-supervised learning, Goldberg et al.~\cite{goldberg2009multi} introduce a spectral clustering method based on a local principal components analysis (PCA) to utilize the unlabeled points.  Both methods rely on local PCA to estimate the local geometry of the data and they both operate by coarsening the data, eventually applying spectral clustering to a small subset of points acting as representative hubs for other points in their neighborhoods.  They both implicitly require that, for the most part, the vast majority of data points in each neighborhood where the statistics are computed come from a single cluster.
Souvenir and Pless~\cite{souvenir} suggest an algorithm that starts with ISOMAP and then alternates in EM-fashion between the cluster assignment and the computation of the distances between points and clusters---this is done in a lower dimensional Euclidean space using an MDS embedding.  Though this iterative method appears very challenging to be analyzed, it relies on pairwise distances computed as a preprocessing step to derive the geodesic distances, which implicitly assumes that the points in small enough neighborhoods are from the same manifold.
Thus, like the SC algorithm, all these methods effectively rely on neighborhoods where only one cluster dominates.  This is strong evidence that their separation requirements are at best similar to that of SC.
The methods of Haro et al.~\cite{Haro06} and Gionis et al.~\cite{gionis} are solely based on the local dimension and density, and are powerless when the underlying manifolds are of same dimension and sampled more or less uniformly, which is the focus of this paper.
The method of Guo et al.~\cite{energy} relies on minimizing an energy that, just as HOSC, incorporates the diameter and local curvature of $m$-tuples, with $m = 3$ for curves and $m = 4$ for surfaces in 3D, and the minimization is combinatorial over the cluster assignment.  In principle, this method could be analyzed with the arguments we deploy here.  That said, it seems computationally intractable.

\subsection{Computational Considerations}
\label{sec:complexity}

Thus it appears that HOSC is superior to SC and other methods in terms of separation between clusters and robustness to outliers, when the clusters are smooth and the jitter is small.  But is HOSC even computationally tractable?

Assume $K$ and $D$ are fixed.  The algorithm starts with building the neighborhood graph (i.e., computing the matrix $\bW$).  This may be done by brute force in $O(m N^m)$ flops.  Clearly, this first step is prohibitive, in particular since we recommend using a (moderately) large $m$.  However, we may restrict computations to points within distance $\eps$, which essentially corresponds to using a compactly supported kernel $\phi$.  Hence, we could apply a range search algorithm to reduce computations.  Alternatively, at each point we may restrict computations to its $\ell = \omega_\sN \log(N)$ nearest neighbors, with $\omega_\sN \to \infty$, or in a slightly different fashion, adapt the local scaling method proposed in~\cite{Zelnik-Manor04} by replacing $\eps$ in $\alpha_{\ud}(\bx_{i_1}, \dots, \bx_{i_m})$ by $(\eps_{i_1} \cdots \eps_{i_m})^{1/m}$, where $\eps_i$ denotes the distance between $\bx_i$ and its $\ell$th nearest neighbor.  The reason is that the central condition (\ref{eq:eps}) effectively requires that the degree at each point be of order $\log(N)^{m-1}$ (roughly), which is guaranteed if the $\ell$-nearest neighbors are included in the computations; see~\cite{pairwise,maier2007cluster} for rigorous arguments leading to that conclusion.
In low dimensions, $D = O(\log \log N)$, a range search and $\ell$-nearest-neighbor search may be computed effectively with kd-trees in $O(N {\rm poly}(\log N))$ flops.  In higher dimensions, it is essential to use methods that adapt to the intrinsic dimensionality of the data.  Assuming that $d$ is small, the method suggested in~\cite{1143857} has a similar computational complexity.  Hence, the (approximate) affinity matrix $\bW$ can be computed in order $O(N {\rm poly}(\log N)) + O(N \cdot \ell^m)$; assuming $m \leq \log(N)/(\omega_\sN\log\log(N))$, this is of order $O(N^{1 + 1/\omega_\sN})$.
This is within the possible choices for $m$ in Theorem~\ref{th:linear}.

Assume we use the $\ell$-nearest-neighbor approximation to the neighborhood graph, with $\ell = \omega_\sN \log(N)$. Then computing $\bZ$ may be done in $O(N^{1 + 1/\omega_\sN})$ flops, since the affinity matrix $\bW$ has at most $\ell^m = O(N^{1/\omega_\sN})$ non-zero coefficients per row.  Then extracting the leading $K$ eigenvectors of $\bZ$ may be done in $O(K N^{1 + 1/\omega_\sN})$ flops, using Lanczos-type algorithms~\cite{MR1948689}.
Thus we may run the $\ell$-nearest neighbor version of HOSC in $O(N^{1 + 1/\omega_\sN})$ flops, and it may be shown to perform comparably.


We actually implemented the $\ell$-nearest-neighbor variant of HOSC and tried it on a number of simulated datasets and a real dataset from motion segmentation.  The results are presented in Section~\ref{sec:numerics}.  The code is publicly available online~\cite{hosc}.

\subsection{Content}

The rest of the paper is organized as follows.
The main theoretical results are in Section~\ref{sec:same} where we provide theoretical guarantees for HOSC, including in contexts where outliers are present or the underlying clusters intersect. We emphasize that HOSC is only able to separate intersecting clusters under very stringent assumptions.
In the same section we also address the issue of estimating the parameters that need to be provided to HOSC.  In theory at least, they may be chosen automatically.
 In Section~\ref{sec:numerics} we implemented our own version of HOSC and report on some numerical experiments involving both simulated and real data.
Section~\ref{sec:discussion} discusses a number of important extensions, such as when the surfaces self-intersect or have boundaries, which are excluded from the main discussion for simplicity.  We also discuss the case of manifolds of different intrinsic dimensions, suggesting an approach that runs HOSC multiple times with different $d$.  And we describe a kernel version of HOSC that could take advantage of higher degrees of smoothness.  Other extensions are also mentioned, including the use of different kernels.
The proofs are postponed to the Appendix.

\section{Theoretical Guarantees}
\label{sec:same}

Our main result provides conditions under which HOSC is perfectly accurate with probability tending to one in the framework introduced in Section~\ref{sec:setting}.
Throughout the paper, we state and prove our results when the surfaces have no boundary and for the simple kernel $\phi(s) = {\bf 1}_{\{|s| < 1\}}$, for convenience and ease of exposition.  We discuss the case of surfaces with boundaries in Section~\ref{sec:boundary} and the use of other kernels in Section~\ref{sec:extensions}.


\begin{theorem}
\label{th:linear}
Consider the generative model of Section~\ref{sec:setting}.  For $\rho_\sN \to \infty$ slowly (e.g., $\rho_\sN = \log \log N$), assume the parameters of HOSC satisfy
\begin{equation} \label{eq:m}
\log N \geq m \geq \frac{\log N}{\sqrt{\log \rho_\sN}},
\end{equation}
\begin{equation} \label{eq:eps}
\eps \geq \left(\rho_\sN^2 \frac{\log N}{N}\right)^{1/d} \vee \tau^{1 - d/D} \left(\rho_\sN^2 \frac{\log N}{N}\right)^{1/D}.
\end{equation}
and
\begin{equation} \label{eq:eta}
\eta \geq \eps \wedge (\tau + \rho_\sN \eps^2)
\end{equation}
Assume that (\ref{eq:delta}) holds with
\begin{equation} \label{eq:delta-lb}
\delta - 2 \tau > \eps \wedge \rho_\sN \eta.
\end{equation}
Under these conditions, when $N$ is large enough, HOSC is perfectly accurate with probability at least $1 - N^{-\rho_\sN}$.
\end{theorem}

To relate this to the separation requirement stated in the Introduction, the condition \eqref{eq:sep} is obtained from \eqref{eq:delta-lb} by choosing $\eps$ and $\eta$ equal to their respective lower bounds in \eqref{eq:eps} and \eqref{eq:eta}.

We further comment on the theorem.  First, the result holds if $\rho_\sN = \rho$ and $\rho$ is sufficiently large.  We state and prove the result when $\rho_\sN \to \infty$ as a matter of convenience.  Also, by \eqref{eq:m} and \eqref{eq:delta-lb}, the weakest separation requirement is achieved when $m$ is at least of order slightly less than $O(\log N)$ so that $\rho_N$ is of order $O(1)$.  However, as discussed in Section~\ref{sec:complexity}, the algorithm is not computationally tractable unless $m = o(\log N)$.  This is another reason why we focus on the case where $\rho_\sN \to \infty$.
Regarding the constraints \eqref{eq:eps}-\eqref{eq:eta} on $\eps$ and $\eta$, they are there to guarantee that, with probability tending to one, each cluster is `strongly' connected in the neighborhood graph.  Note that the bound on $\eps$ is essentially the same as that required by the pairwise spectral method SC~\cite{pairwise,maier2007cluster}.
In turn, once each cluster is `strongly' connected in the graph, clusters are assumed to be separated enough that they are `weakly' connected in the graph.  The lower bound \eqref{eq:delta-lb} quantifies the required separation for that to happen.  Note that it is specific to the simple kernel.  For example, the heat kernel would require a multiplicative factor proportional to $\sqrt{\log N}$.

So how does HOSC compare with SC?  When the jitter is large enough that $\tau \gg (\log(N)/N)^{1/d}$, we have $\eta \geq \eps$ and the local linear approximation contribution to (\ref{eq:linear-affinity}) does not come into play.  In that case, the two algorithms will output the same clustering (see Figure~\ref{fig:tau_large} for an example).

\begin{figure}[htbp]
\centering
\includegraphics[width=.48\linewidth]{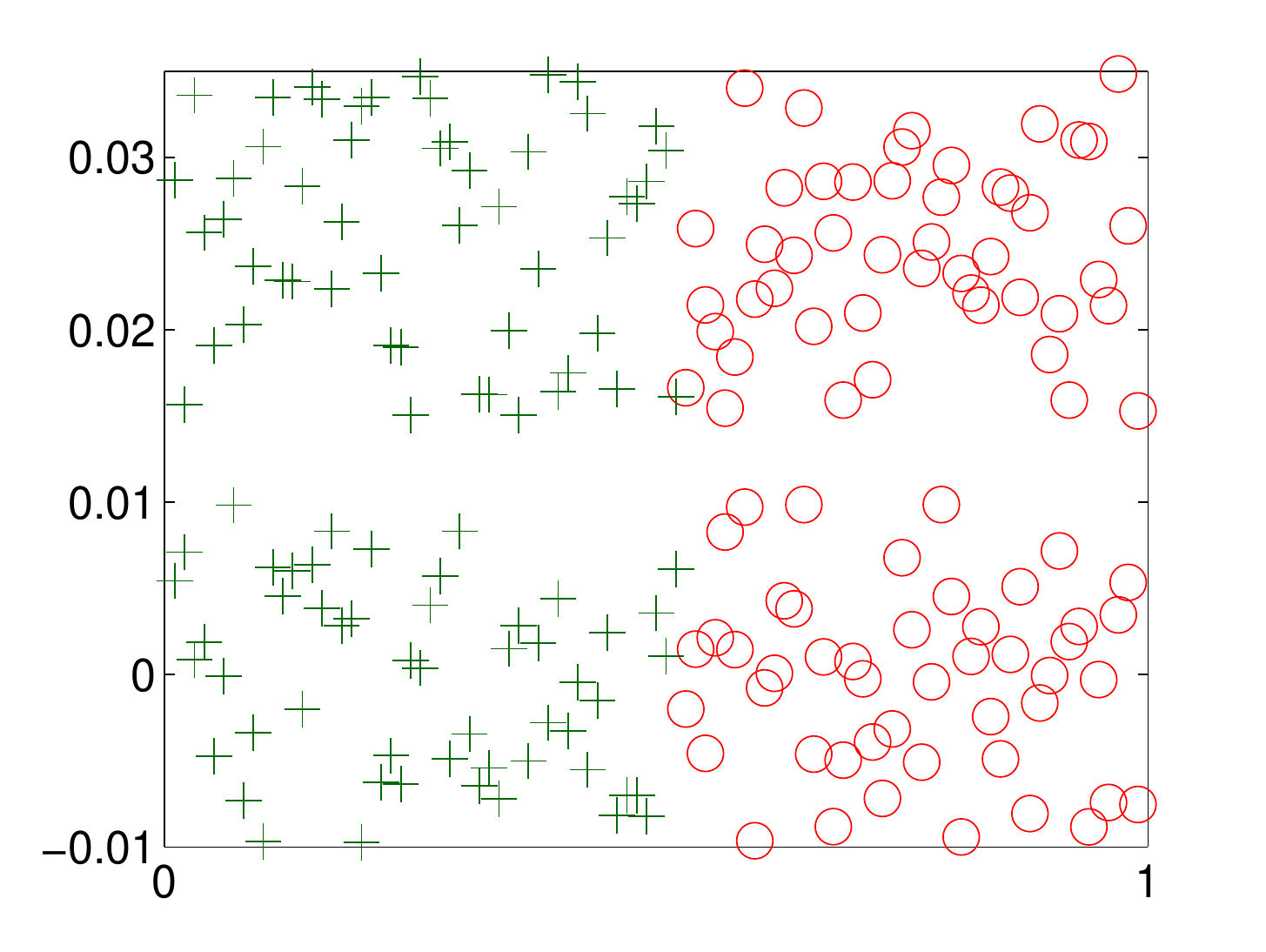}
\includegraphics[width=.48\linewidth]{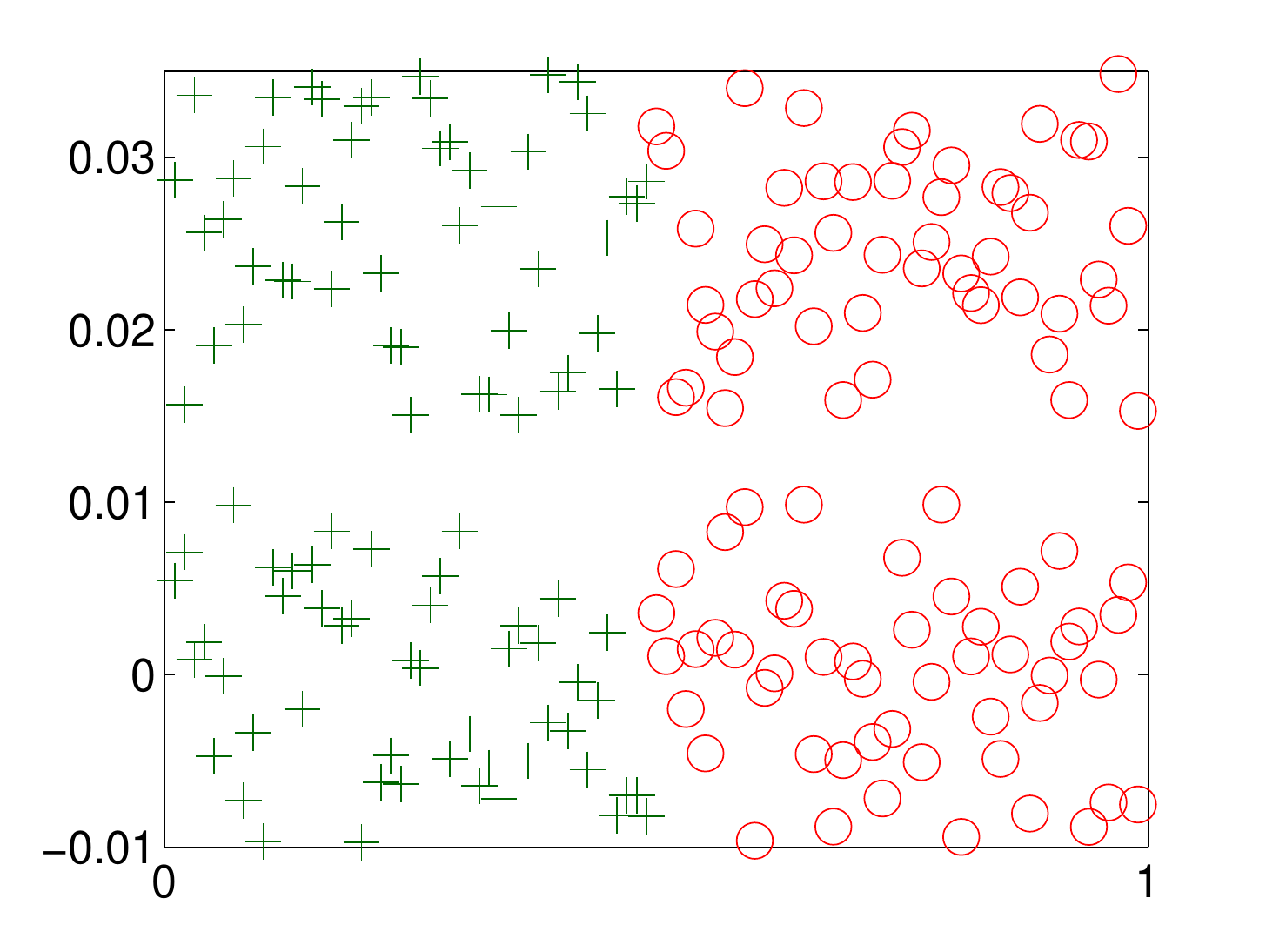}
\caption{Clustering results obtained by SC (left) and HOSC (right) on the data set of Figure~\ref{fig:sep_small}, but with separation $\delta=0.025$ and jitter $\tau=0.01$. In this example, neither SC nor HOSC can successfully separate the two lines. This example supports our claim that when the jitter is large enough (relative to separation), HOSC does not improve over SC and the two algorithms will output the same clustering.}
\label{fig:tau_large}
\end{figure}

When the jitter is small enough that $\tau \ll (\log(N)/N)^{1/d}$, HOSC requires less separation, as demonstrated in Figure~\ref{fig:sep_small}.
Intuitively, in this regime the clusters are sampled densely enough relative to the thickness $\tau$ that the smoothness of the underlying surfaces comes into focus and each cluster, as a point cloud, becomes locally well-approximated by a thin band.
We provide some numerical experiments in Section~\ref{sec:numerics} showing HOSC outperforming SC in various settings.
%

Thus, HOSC improves on SC only when the jitter is small.  This condition is quite severe, though again, we do not know of any other method that can accurately cluster under the weak separation requirement displayed here, even in the jitterless case.  It is possible that some form of scan statistic (i.e., matched filters) may be able to operate under the same separation requirement without needing the jitter to be small, however, we do not know how to compute it in our nonparametric setting---even in the case of hybrid linear modeling where the surfaces are affine, computing the scan statistic appears to be computationally intractable.
At any rate, the separation required by HOSC is essentially optimal when $\tau$ is of order $O(N^{-1/d})$ or smaller.  A quick argument for the case $d=1$ and $D=2$ goes as follows.  Consider a line segment of length one and sample $N$ points uniformly at random in its $\tau$-neighborhood, with $\tau = O(1/N)$.  The claim is that this neighborhood contains an empty band of thickness of order slightly less than $O(1/N^2)$, and therefore cannot be distinguished from two parallel line segments.  Indeed, such band of half-width $\lambda$ inside that neighborhood is empty of sample points with probability $(1 - \lambda/\tau)^N$, which converges to 1 if $N \lambda/\tau \to 0$, and when $\tau = O(1/N)$, this is the case if $\lambda = o(1/N^2)$.


In regards to the choice of parameters, the recommended choices depend solely on $(d, \tau, K)$.  These model characteristics are sometimes unavailable and we discuss their estimation in Section~\ref{sec:param}. Afterwards, we discuss issues such as outliers (Section~\ref{sec:outliers}) and intersection (Section~\ref{sec:intersect}).

\subsection{Parameter Estimation}
\label{sec:param}

In this section, we propose some methods to estimate the intrinsic dimension $d$ of the data, the jitter $\tau$ and the number of clusters $K$.  Though we show that these methods are consistent in our setting, further numerical experiments are needed to determine their potential in practice.

Compared to SC, HOSC requires the specification of three additional parameters.  This is no small issue in practice.  In theory, however, we recommend choosing $d$ and $K$ consistent with their true values, $\eps$ and $\eta$ as functions of $\tau$, and $m$ of order slightly less than $\log(N)$.  The true unknowns are therefore $(d, \tau, K)$.  We provide estimators for $d$ and $K$ that are consistent, and an estimator for $\tau$ that is accurate enough for our purposes.  Specifically, we estimate $d$ and $\tau$  using the correlation dimension~\cite{cor-dim} and an adaptation of our own design.
The number of clusters $K$ is estimated via the eigengap of the matrix $\bZ$.  

\subsubsection{The Intrinsic Dimension and the Jitter Level}

A number of methods have been proposed to estimate the intrinsic dimensionality; we refer the reader to~\cite{levina-bickel} and references therein.  The correlation dimension, first introduced in~\cite{cor-dim}, is perhaps the most relevant in our context, since surfaces may be close together.  Define the pairwise correlation function
$$
{\rm Cor}(\eps)  = \sum_{i} \sum_{j \neq i} {\bf 1}_{\{\| \bx_i - \bx_j \| \leq \eps\}}.
$$
The authors of~\cite{cor-dim} recommend plotting $\log {\rm Cor}(\eps)$ versus $\log \eps$ and estimating the slope of the linear part.  We use a slightly different estimator that allows us to estimate $\tau$ too, if it is not too small.  The idea is to regress $\log {\rm Cor}(\eps)$ on $\log \eps$ and identify a kink in the curve.  See Figure~\ref{fig:cor-eps} for an illustration.

\begin{figure}[htbp]
\centering
\includegraphics[width=.50\linewidth]{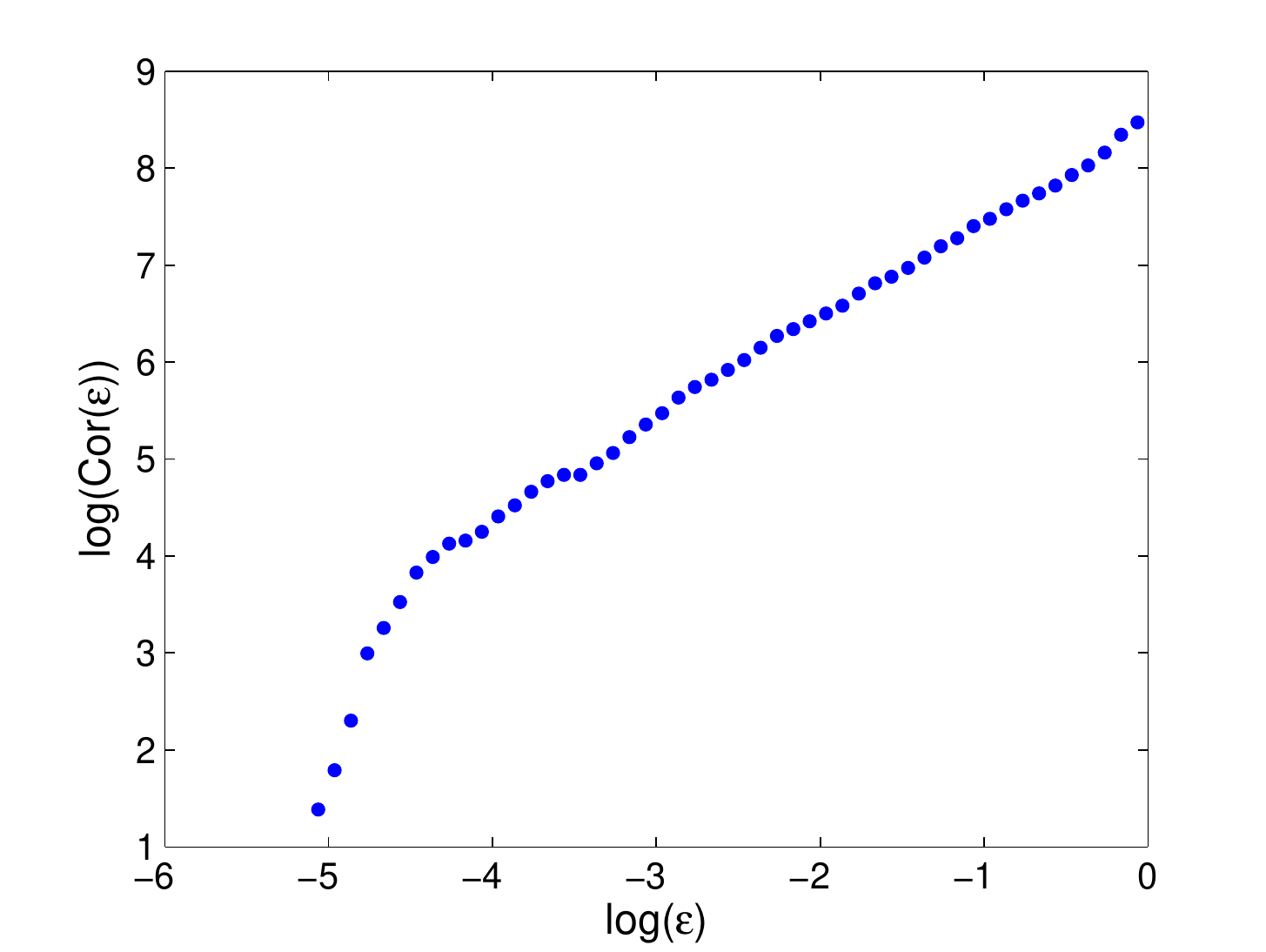}
\caption{A correlation curve for a simulated data set of 240 points sampled from the $\tau$-neighborhood of three disjoint one-dimensional curves ($d=1$) in dimension ten ($D=10)$ crossing all dimensions.  The jitter is $\tau=0.01$.  We see that the linear part of the curve has slope (near) 1, which coincides with the intrinsic dimension of the curves. The kink appears near $\tauhat := \exp(-4.5) = 0.0111$, a close approximation to $\tau$.}
\label{fig:cor-eps}
\end{figure}

Though several (mostly ad hoc) methods have been proposed for finding kinks, we describe a simple method for which we can prove consistency.
Fix $\rho_\sN \to \infty$, with $\rho_\sN \ll \log N$.  Define
$$
r_\sN = -\left[\frac{\log \log(N) -\log N}{d \log \rho_\sN}\right] -2.
$$
Let $A_r = \log {\rm Cor}(\rho_\sN^{-r})$.
If there is $r \in \{3, \dots, r_\sN - 2D-1\}$ such that
$$(A_r - A_{r+1})/\log \rho_\sN > D - 1/2,$$
then let $\rhat \geq 0$ be the smallest such $r$; otherwise, let $\rhat = r_\sN - 2D$.
Define $\tauhat = \rho_\sN^{-\rhat}$; and also $\dhat = D$, if $\rhat = 3$, and $\dhat$ the closest integer to $(A_{3} - A_{\rhat})/(\rhat \log \rho_\sN)$, otherwise.

\begin{proposition} \label{prop:tau-1}
Consider the generative model described in Section~\ref{sec:setting} with $S_1, \dots, S_\sK \in \cS_{d}^2(\kappa)$.  Assume that $\tau \leq \rho_\sN^{-3}$ and, if there are $N_0$ outliers, assume that $N -N_0 \geq N/\rho_\sN$.  Then the following holds with probability at least $1 - N^{-\sqrt{\rho_\sN}}$:  if $\rhat < r_\sN - 2D$, then $\tau \in [\tauhat/\rho_\sN, \rho_\sN \tauhat]$; if $\rhat = r_\sN - 2D$, then $\tau \leq \tauhat$; moreover, if $\rhat > 3$, $\dhat = d$.
\end{proposition}
%

In the context of Proposition~\ref{prop:tau-1}, the only time that $\dhat$ is inconsistent is when $\tau$ is of order $\rho_\sN^{-3}$ or larger, in which case $\dhat = D$; this makes sense, since the region $\bigcup_k B(S_k, \tau)$ is in fact $D$-dimensional if $\tau$ is of order 1.  Also, $\tauhat$ is within a $\rho_\sN$ factor of $\tau$ if $\tau$ is not much smaller than $(\log(N)/N)^{1/d}$.

We now extend this method to deal with a smaller $\tau$.  Consider what we just did.  The quantity ${\rm Cor}(\eps)$ is the total degree of the $\eps$-neighborhood graph built in SC. Fixing $(d, m)$, we now consider the total degree of the $\eta$-neighborhood graph built in HOSC.  Define the multiway correlation function
$${\rm Cor}_{\ud,m}(\eps, \eta) = \sum_{i} D_i^{1/(m-1)}.$$
Similarly, we shall regress $\log{\rm Cor}_{\ud,m}(\eps, \eta)$ on $\log \eta$ and identify a kink in the curve (Figure~\ref{fig:cor-eps-smalltau} displays such a curve).


\begin{figure}[htbp]
\centering
\includegraphics[width=.48\linewidth]{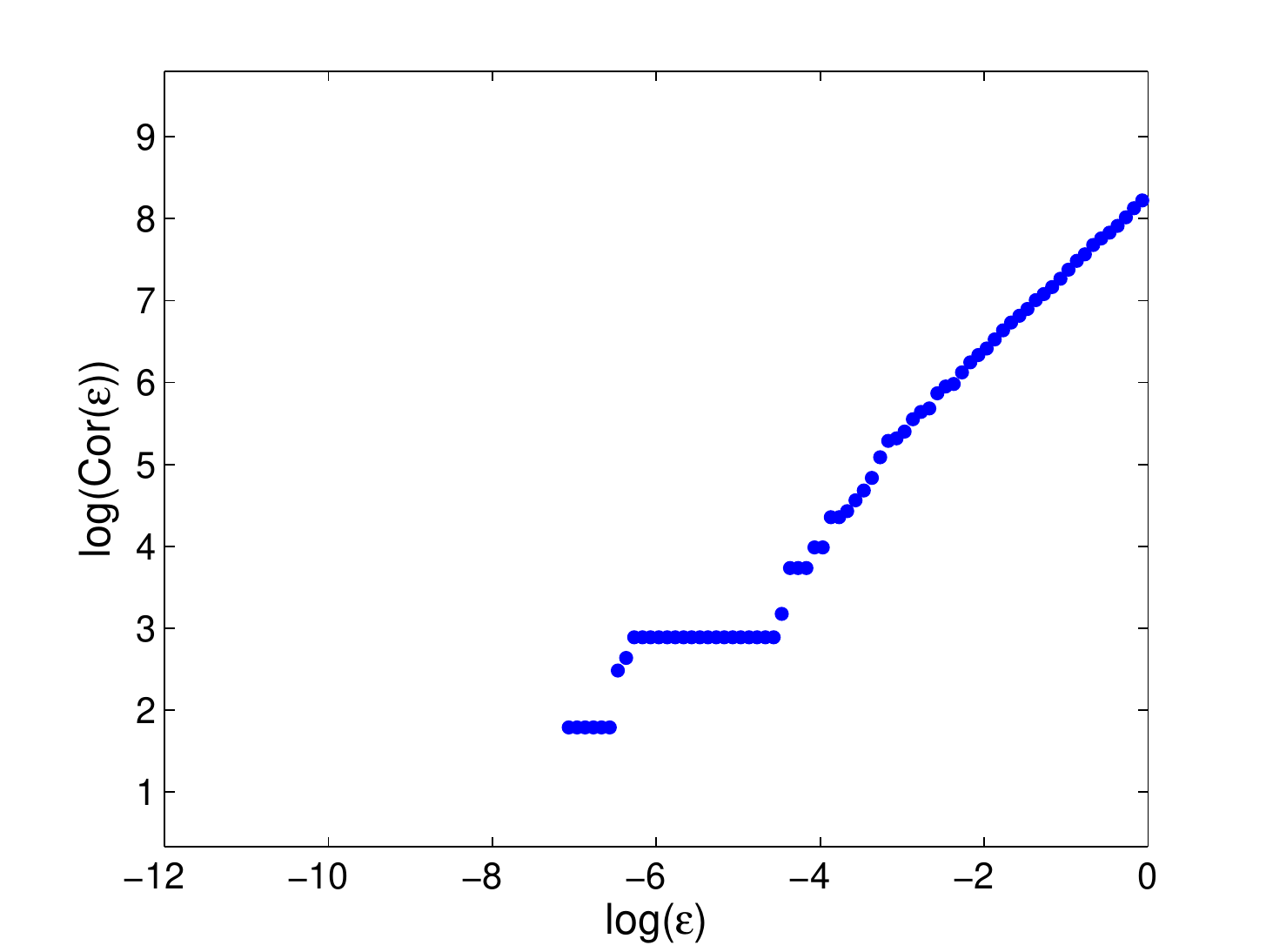}
\includegraphics[width=.48\linewidth]{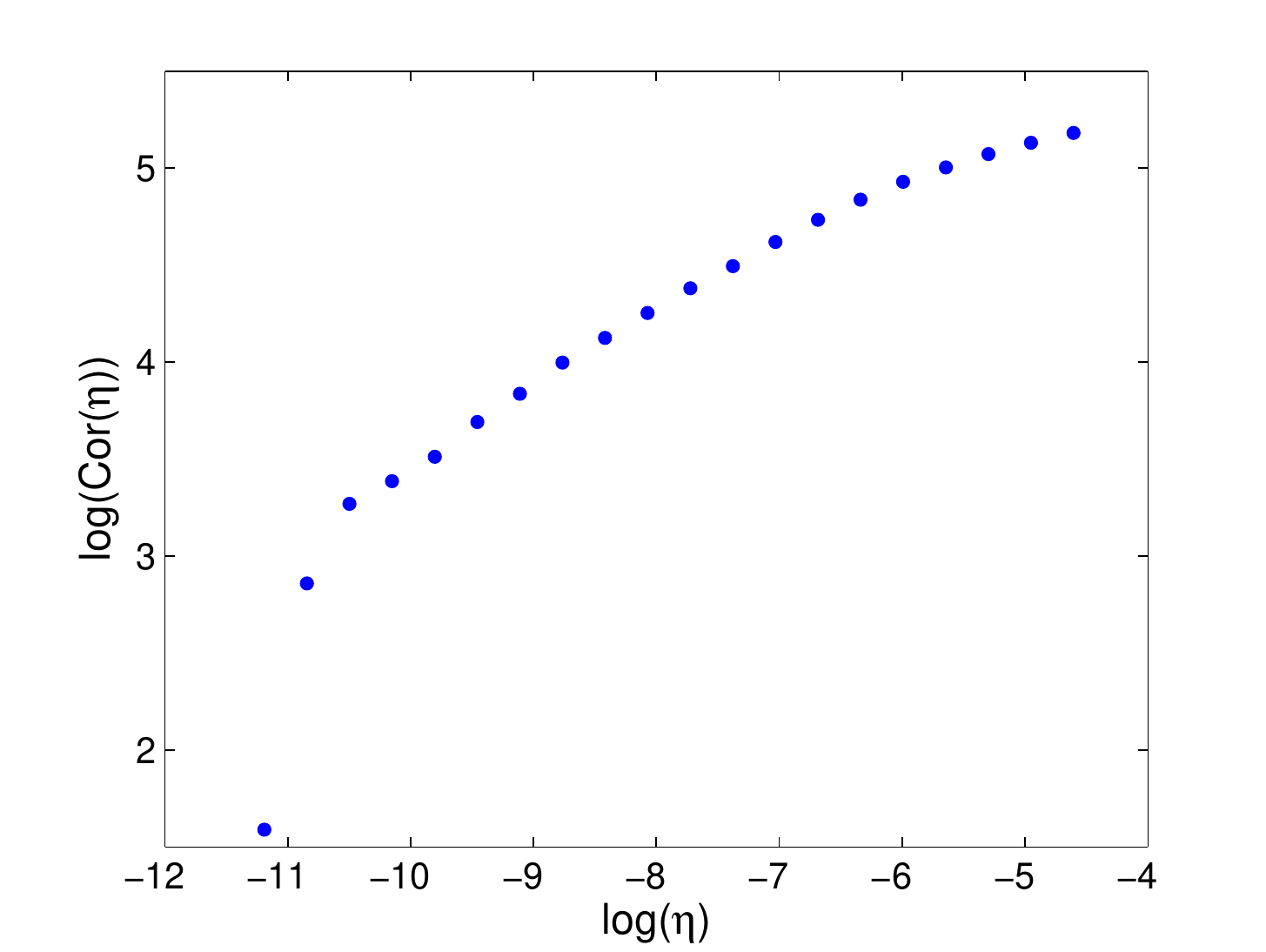}
\caption{Correlation curves corresponding to SC (left) and HOSC (right) for the data set of Figure~\ref{fig:cor-eps}, but with a much smaller $\tau=1e-4$.  We see that the pairwise correlation function works poorly in this case, while the multiway correlation curve has a kink near $\tauhat := \exp(-10.5) = 2.754e-5$, within a factor of $\frac{1}{4}$ of the true $\tau$.}
\label{fig:cor-eps-smalltau}
\end{figure}

Using the multiway correlation function, we then propose an estimator $\tauhat$ as follows.
We assume that the method of Proposition~\ref{prop:tau-1} returned $\rhat = r_\sN - 2D $, for otherwise we know that $\tauhat$ is accurate.  Choose $\ud = \dhat$ and $m \geq \log(N) (\log \rho_\sN)^2$.  Note that this is the only time we require $m$ to be larger than $\log N$.  Let $B_{s} = \log {\rm Cor}_{\ud, m}(\rho_\sN^{-\rhat}, \rho_\sN^{-\rhat-s})$.
If there is $s \in \{0, \dots, \rhat-1\}$ such that
$$(B_s - B_{s+1})/\log \rho_\sN > D - d - 1/2,$$
then let $\shat$ be the smallest one; otherwise, let $\shat = \rhat$.
We then redefine $\tauhat$ as $\tauhat = \rho_\sN^{-\rhat-\shat+1}$.
\begin{proposition} \label{prop:tau-2}
In the context of Proposition~\ref{prop:tau-1}, assume that $\rhat = r_\sN - 2D$.  Then redefining $\tauhat$ as done above, the following holds with probability at least $1 - N^{-\sqrt{\rho_\sN}}$:  if $\shat < \rhat$, then $\tau \in [\tauhat/\rho_\sN, \rho_\sN \tauhat]$; if $\shat = \rhat$, then $\tau \leq \tauhat$.
\end{proposition}
Now, $\tauhat$ comes close to $\tau$ if $\tau$ is not much smaller than $(\log(N)/N)^{2/d}$.  Whether this is the case, or not, the statement of Theorem~\ref{th:linear} applies with $\tauhat$ in place of $\tau$ in \eqref{eq:eta}.

Though our method works in theory, it is definitely asymptotic.  In practice, we recommend using other approaches for determining the location of the kink and the slope of the linear part of the pairwise correlation function (in log-log scale).  Robust regression methods with high break-down points, like least median of squares and least trimmed squares, worked well in several examples.  We do not provide details here, as this is fairly standard, but the figures are quite evocative.

\subsubsection{The Number of Clusters}
\label{sec:K}

HOSC depends on choosing the number of clusters $K$ appropriately.  A common approach consists in choosing $K$ by inspecting the eigenvalues of $\bZ$.  We show that, properly tuned, this method is consistent within our model.

\begin{proposition}
\label{prop:K-choice}
Compute the matrix $\bZ$ in HOSC with the same choice of parameters as in Theorem~\ref{th:linear}, except that knowledge of $K$ is not needed.
Set the number of clusters equal to the number of eigenvalues of $\bZ$ (counting multiplicity) exceeding $1 -  N^{-2}/\rho_\sN$.  Then with probability at least $1 - N^{-\rho_\sN}$, this method chooses the correct number of clusters.
\end{proposition}
We implicitly assumed that $d$ and $\tau$ are known, or have been estimated as described in the previous section.
The proof of Proposition~\ref{prop:K-choice} is parallel to that of~\cite[Prop.~4]{pairwise}, this time using the estimate provided in part \ref{a1} of the proof of Theorem~\ref{th:linear}.  Details are omitted.

Figure~\ref{fig:eigen-graph} illustrates a situation where the number of clusters is correctly chosen by inspection of the eigenvalues, more specifically, by counting the number of eigenvalue $1$ in the spectrum of $\bZ$ (up to numerical error).  This success is due to the fact that the clusters are well-separated, and even then, the eigengap is quite small.

\begin{figure}[htbp]
\centering
\includegraphics[width=.50\linewidth]{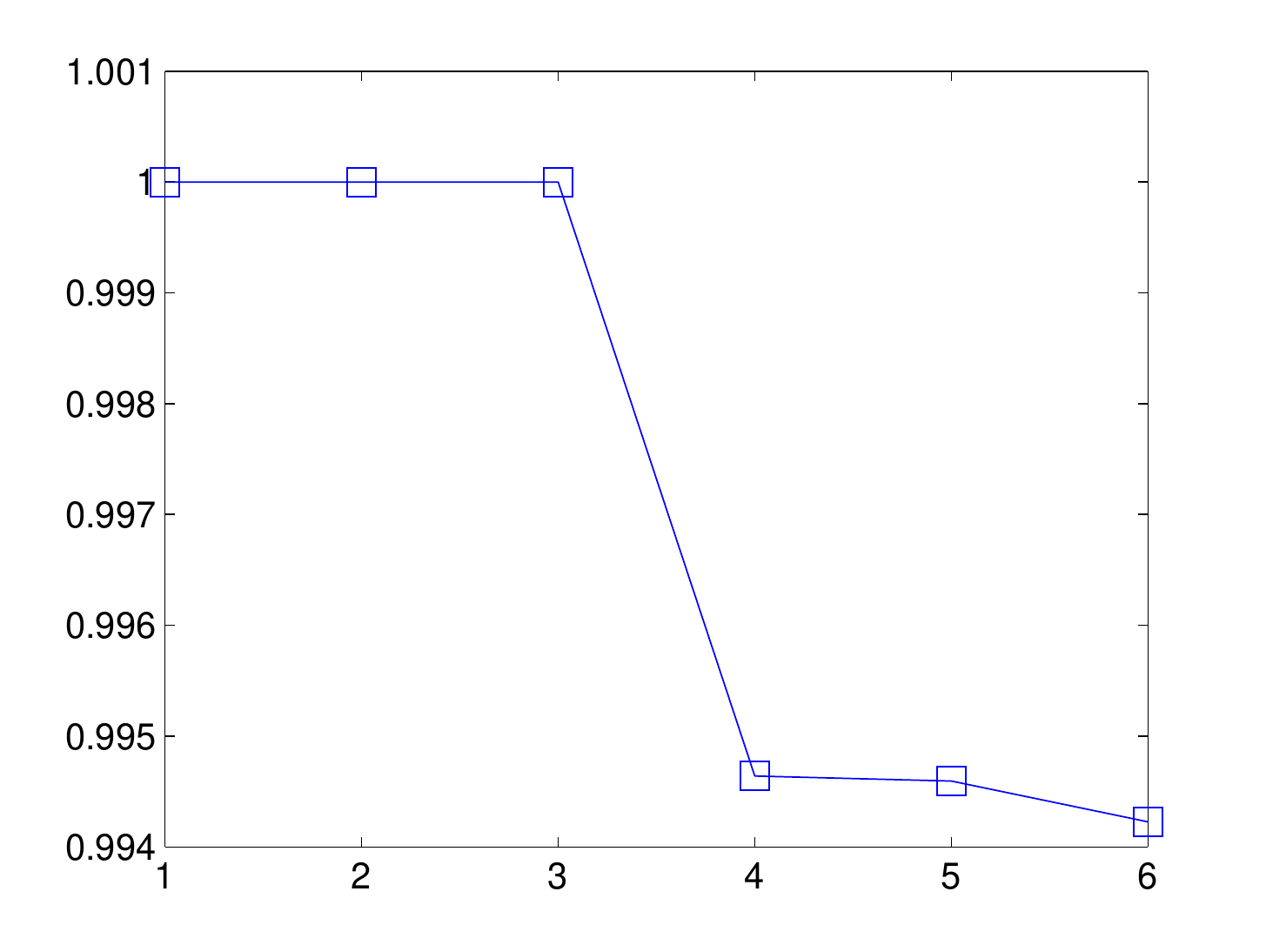}
\caption{The top six eigenvalues of the weight matrix $\mathbf{Z}$ obtained by HOSC in Step 2 for the same data used in Figure~\ref{fig:cor-eps}.  Though in this example the clusters are well-separated, the eigengap is still very small (about 0.005).  
}
\label{fig:eigen-graph}
\end{figure}

We apply this strategy to more data later in Section~\ref{sec:numerics}, and show that it can correctly identify the parameter $K$ in some cases (see Figure \ref{fig:eigs_hosc_syndata}).
In general we do not expect this method to work well when the data has large noise or intersecting clusters,
though we do not know of any other method that works in theory under our very weak separation requirements.

\subsection{When Outliers are Present}
\label{sec:outliers}

So far we have only considered the case where the data is devoid of outliers.  We now assume that some outliers may be included in the data as described at the end of Section~\ref{sec:setting}.  As stated there, we label as outlier any data point with low degree in the neighborhood graph, as suggested in~\cite{spectral_applied,1519716,pairwise}.  Specifically, we compute $\bD$ as in Step 2 of HOSC, and then label as outliers points $\bx_i$ with degree $D_i$ below some threshold.
Let $\rho_\sN \to \infty$ slower than any power of $N$, e.g., $\rho_\sN = \log N$.
We propose two thresholds:
\renewcommand{\theenumi}{(O\arabic{enumi})}
\renewcommand{\labelenumi}{\theenumi}
\begin{enumerate}
\item \label{o1} Identify as outliers points with degree:
$$D_i^{1/(m-1)} \leq \rho_\sN^{-1} \max_j D_j^{1/(m-1)}.$$
\item \label{o2} Identify as outliers points with degree:
$$D_i^{1/(m-1)} \leq \rho_\sN N \eps^{\ud} \eta^{D-\ud}.$$
\end{enumerate}

Taking up the task of identifying outliers, only the separation between outliers and non-outliers is relevant, so that we do not require any separation between the actual clusters.
We first analyze the performance of \ref{o1}, which requires about the same separation between outliers and non-outliers as HOSC requires between points from different clusters in \eqref{eq:delta-lb}.
\begin{proposition}
\label{prop:linear-outliers-1}
Consider the generative model described in Section~\ref{sec:setting}.
Assume that $N -N_0 \geq N/\rho_\sN$ and that \eqref{eq:m}-\eqref{eq:eta} hold.
In terms of separation, assume that $\delta_0 - \tau > \eps \wedge \rho_\sN \eta.$
Then with probability at least $1 - N^{-\rho_\sN}$, the procedure \ref{o1} identifies outliers without error.
\end{proposition}

We now analyze the performance of \ref{o2}, which requires a stronger separation between outliers and non-outliers, but operates under very weak sampling requirements.
\begin{proposition}
\label{prop:linear-outliers-2}
Assume that $m$ is as in \eqref{eq:m}, and
\begin{equation} \label{eq:eps-eta-O2}
\eps = (\rho_\sN \log(N)/N)^{1/(2D-d)}, \quad \eta = (\rho_\sN \log(N)/N)^{2/(2D-d)}.
\end{equation}
In terms of separation, assume that $\delta_0 - \tau >  \eps$.  In addition, suppose that
\begin{equation} \label{eq:eps-lb}
N_k \geq \rho_\sN \log(N) N^{d/(2D-d)} \vee N \tau^{D-d}, \ \forall k = 1, \dots, K.
\end{equation}
Then with probability at least $1 - N^{-\rho_\sN}$, the procedure \ref{o2} identifies outliers without error.
\end{proposition}

If $\delta_0 = \tau$, so that outliers are sampled everywhere but within the $\tau$-tubular regions of the underlying surfaces, then both \ref{o1} and \ref{o2} may miss some outliers within a short distance from some $B(S_k, \tau)$.  Specifically, \ref{o1} (resp.~\ref{o2}) may miss outliers within $\eps \wedge \rho_\sN \eta$ (resp.~within $\eps$) from some $B(S_k, \tau)$.  Using Weyl's tube formula~\cite{MR1507388}, we see that there are order $N_0 (\eps \wedge \rho_\sN \eta)^{D-d}$ (resp.~$N_0 \eps^{D-d}$) such outliers, a small fraction of all outliers.

The sampling requirement (\ref{eq:eps-lb}) is weaker than the corresponding requirement for pairwise methods displayed in \eqref{eq:N-cond-lb}.  In fact, (\ref{eq:eps-lb}) is only slightly stronger than what is required to just detect the presence of a cluster hidden in noise.  We briefly explain this point.  Instead of clustering, consider the task of detecting the presence of a cluster hidden among a large number of outliers.  Formally, we observe the data $\bx_1, \dots, \bx_\sN$, and want to decide between the following two hypotheses: under the null, the points are independent, uniformly distributed in the unit hypercube $(0,1)^D$; under the alternative, there is a surface $S_1 \in \cS_d^2(\kappa)$ such that $N_1$ points are sampled from $B(S_1, \tau)$ as described in Section~\ref{sec:setting}, while the rest of the points, $N-N_1$ of them, are sampled from the unit hypercube $(0,1)^D$, again uniformly.
Assuming that the parameters $d$ and $\tau$ are known, it is shown in~\cite{AriasCastro2009,CTD} that the scan statistic is able to separate the null from the alternative if
\begin{equation} \label{eq:detect}
N_1 \gg N^{d/(2D-d)} \vee N \tau^{D-d}.
\end{equation}
We are not aware of a method that is able to solve this detection task at a substantially lower sampling rate, and (\ref{eq:eps-lb}) comes within a logarithmic factor from (\ref{eq:detect}).
We thus obtain the remarkable result that accurate clustering is possible within a log factor of the best (known) sampling rate that allows for accurate detection in the same setting.

\subsection{When Clusters Intersect}
\label{sec:intersect}

We now consider the setting where the underlying surfaces may intersect.  The additional conditions we introduce are implicit constraints on the dimension of, and the incidence angle at, the intersections.  We suppose there is an integer $0 \leq d_{\rm int} \leq d-1$ and a finite constant $C > 0$ such that
\begin{equation} \label{eq:d-cap}
\vol_{d}(B(S_k \cap S_\ell, \eps) \cap S_k) \leq C \eps^{d-d_{\rm int}}, \ \forall \eps \in (0,1/\kappa), \ \forall k \neq \ell.
\end{equation}
(The subscript $_{\rm int}$ stands for `intersection'.)
In addition, we assume that for some $\theta_{\rm int} \in (0, \pi/2]$,
\begin{equation} \label{eq:delta-cap}
\dist(\bx, S_\ell) \geq \delta \wedge \sin(\theta_{\rm int}) \dist(\bx, S_k \cap S_\ell), \ \forall \bx \in S_k, \ \forall k \neq \ell \text{ with } S_k \cap S_\ell \neq \emptyset.
\end{equation}
(\ref{eq:d-cap}) is slightly stronger than requiring that $S_k \cap S_\ell$ has finite $d_{\rm int}$-dimensional volume.  If the surfaces are affine, it is equivalent to the condition $\dim(S_k \cap S_\ell) \leq d_{\rm int}, \ \forall k \neq \ell.$
(\ref{eq:delta-cap}), on the other hand, is a statement about the minimum angle at which any two surfaces intersect.  For example, if the surfaces are affine within distance $\delta$ of their intersection, then (\ref{eq:delta-cap}) is equivalent to their maximum (principal) angle being bounded from below by $\theta_{\rm int}$.  See Figure~\ref{fig:intersect} for an illustration.

\begin{figure}[htbp]
\centering
\includegraphics[width=.450\linewidth]{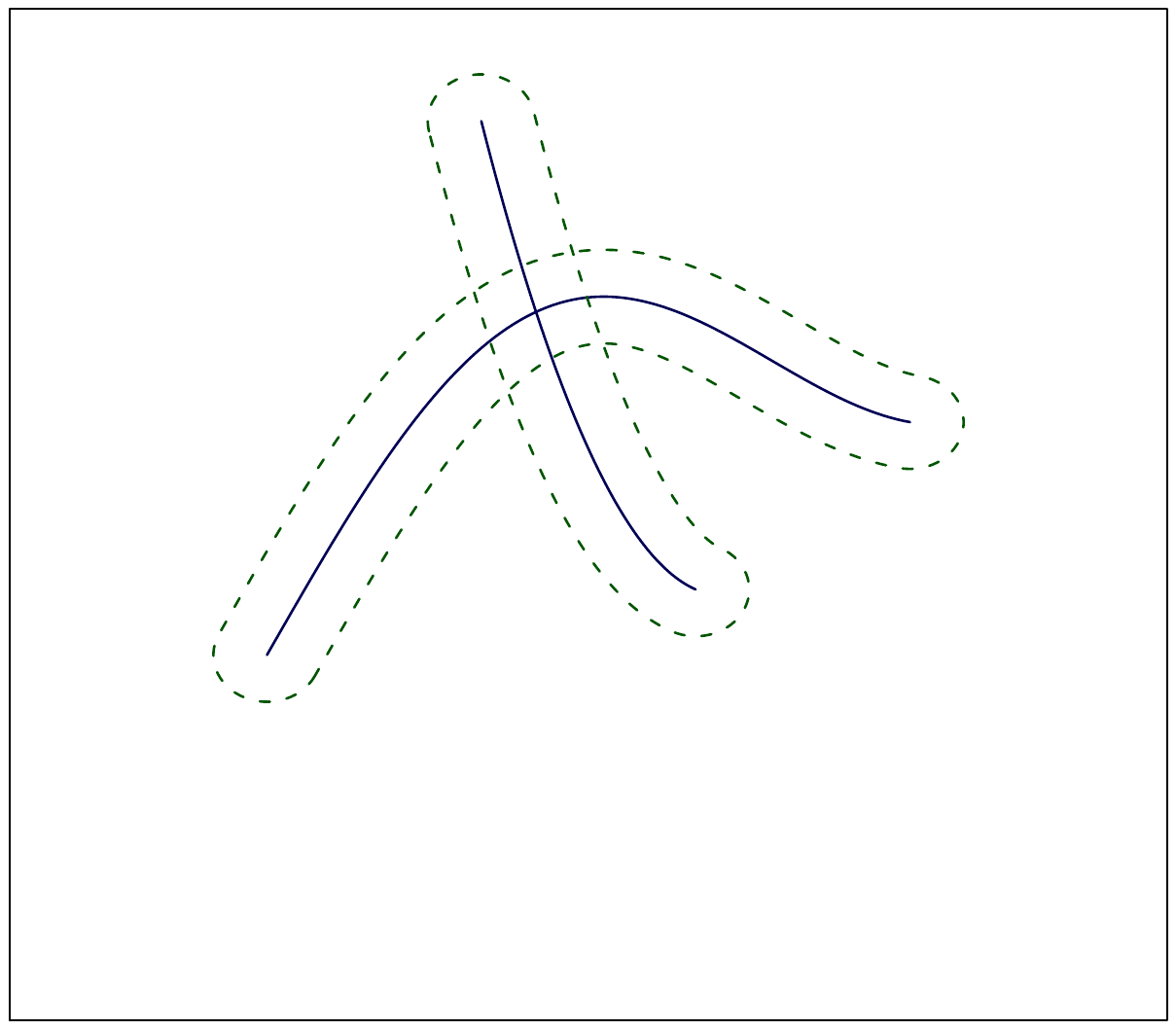} \
\includegraphics[width=.450\linewidth]{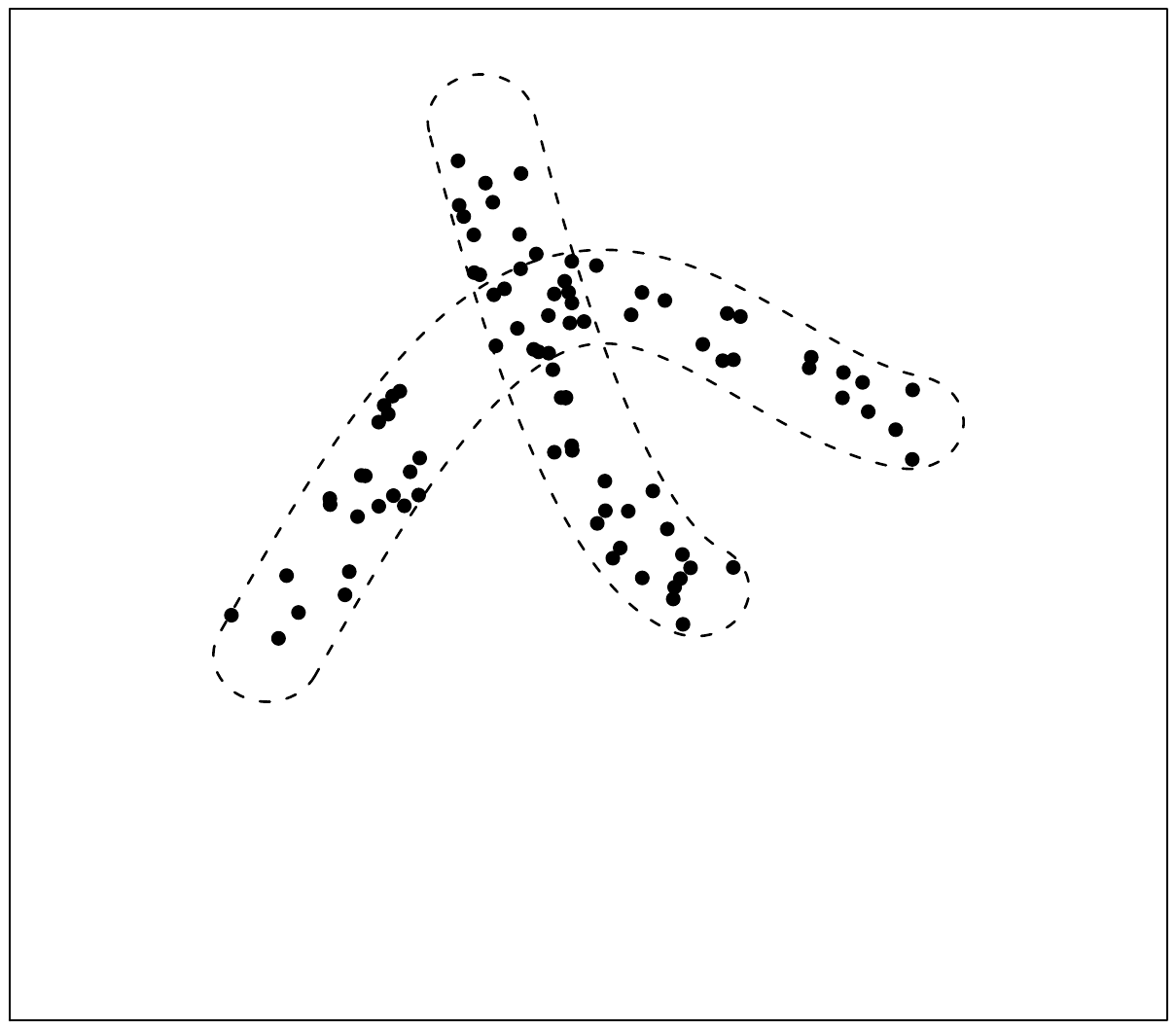} \
\vspace{-.3in}
\caption{Illustration of intersecting surfaces.  Though the human eye easily distinguishes the two clusters, the clustering task is a lot harder for machine learning algorithms.  The main issue is that there are too many data points at the intersection of the two tubular regions.  However, in very special cases HOSC {\em is} able to separate intersecting clusters (see Figure~\ref{fig:inter-sim} for such an example).}
\label{fig:intersect}
\end{figure}

\begin{proposition}
\label{prop:intersect}
Consider the setting of Theorem~\ref{th:linear}, with (\ref{eq:delta}) replaced by (\ref{eq:delta-cap}).  In addition, assume that (\ref{eq:d-cap}) holds.  Define
$$\gamma_\sN := N^2 \eps^{d}  (\eps \wedge \rho_\sN \eta)^{d-d_{\rm int}} (\sin \theta_{\rm int})^{d_{\rm int}-d}.$$
Then there is a constant $C > 0$ such that, with probability at least $1 - C\, \gamma_\sN$, HOSC is perfectly accurate.
\end{proposition}

The most favorable case is when $\tau = 0$ and $\theta_{\rm int} = \pi/2$.  Then with our choice of $\eps$ and $\eta$ in Theorem~\ref{th:linear}, assuming $\rho_\sN$ increases slowly, e.g., $\rho_\sN \prec \log N$, we have $\gamma_\sN \to 0$ if $2 d_{\rm int} < d$, and partial results suggest this cannot be improved substantially.  This constraint on the intersection of two surfaces is rather severe.  Indeed, a typical intersection between two (smooth) surfaces of same dimension $d$ is of dimension $d-1$, and if so, only curves satisfy this condition.  Figure~\ref{fig:inter-sim} provides a numerical example showing the algorithm successfully separating two intersecting one-dimensional clusters.
Thus, even with no jitter and the surfaces intersecting at right angle, HOSC is only able to separate intersecting clusters under exceptional circumstances.  Moreover, even when the conditions of Proposition~\ref{prop:intersect} are fulfilled, the probability of success is no longer exponentially small, but is at best of order $(1/N)^{1-2d_{\rm int}/d}$.
That said, SC does not seem able to properly deal with intersections at all (see also Figure~\ref{fig:inter-sim}).  It essentially corresponds to taking $\eta = \eps$ in HOSC, in which case $\gamma_\sN$ never tends to zero.

\begin{figure}[htbp]
     \centering
     \includegraphics[width=.32\linewidth]{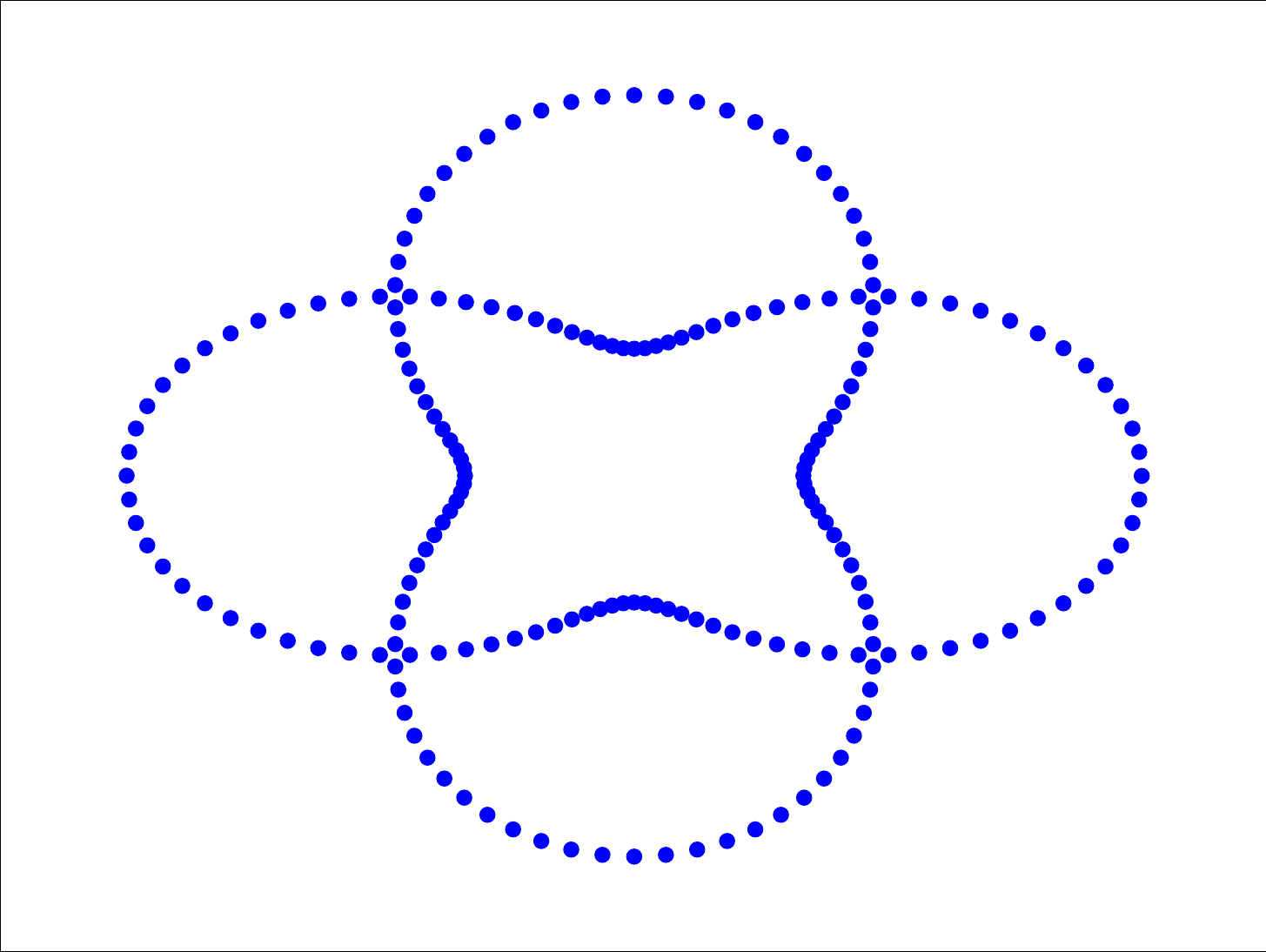}
     \includegraphics[width=.32\linewidth]{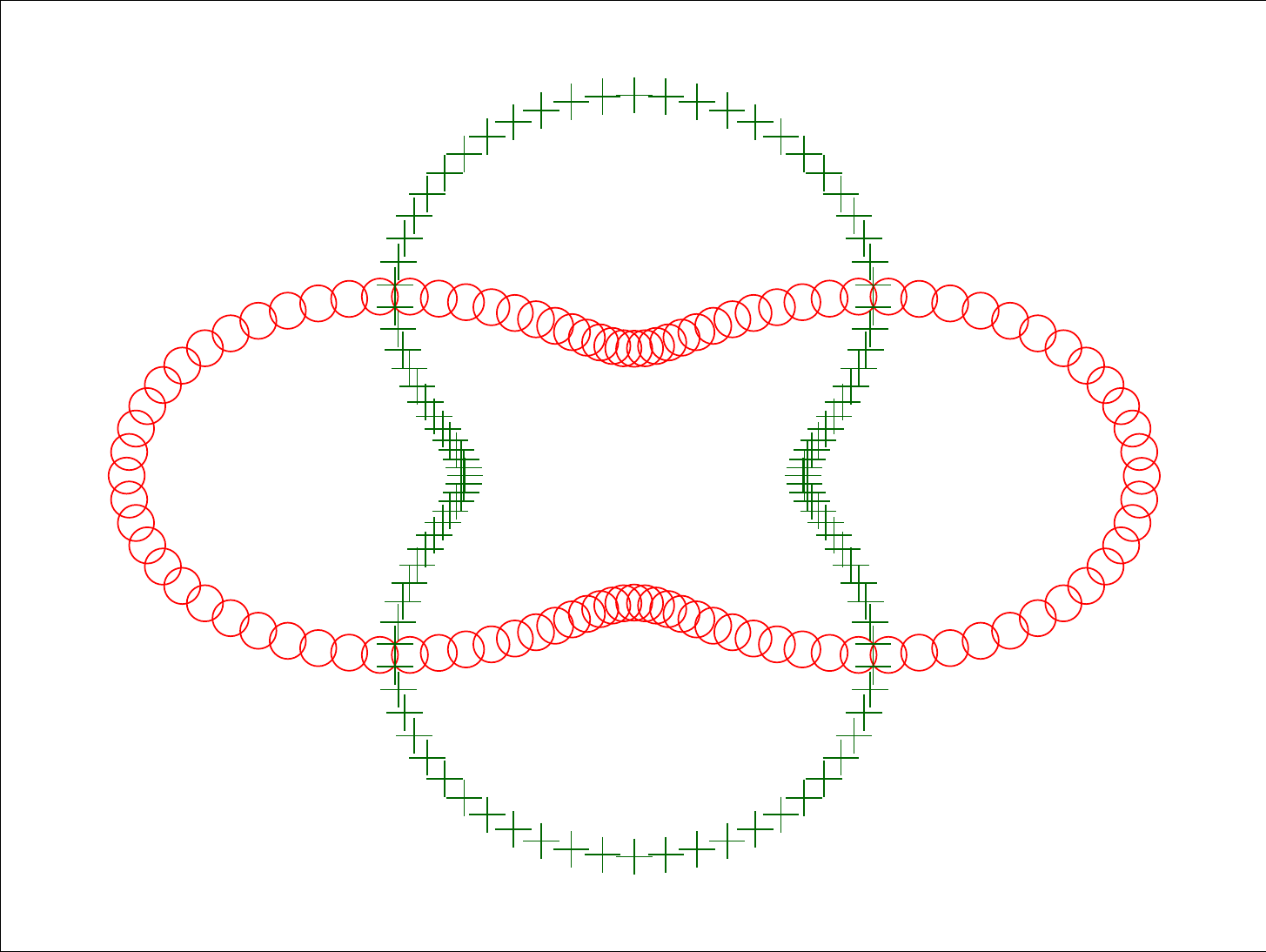}
     \includegraphics[width=.32\linewidth]{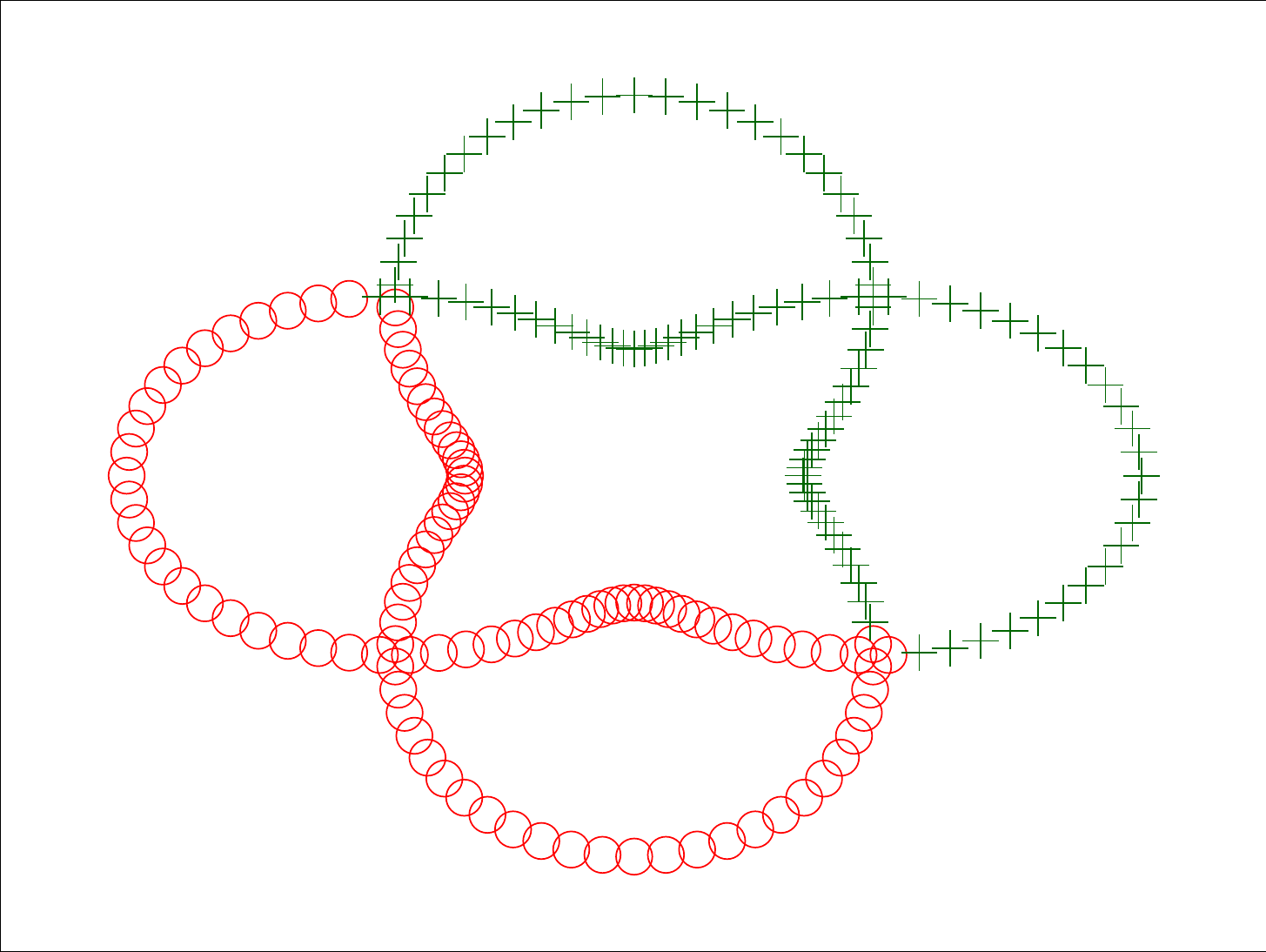}
     \caption{Left: data. Middle: output from HOSC. Right: Output from SC. This example shows that HOSC is able to separate intersecting curvilinear clusters when the incidence angle is perpendicular and there is no jitter ($\tau=0$). In particular, the conditions of Proposition~\ref{prop:intersect} are satisfied. On the contrary, SC fails in this case.}
     \label{fig:inter-sim}
\end{figure}

Though the implications of Proposition~\ref{prop:intersect} are rather limited, we do not know of any other clustering method which provably separates intersecting clusters under a similar generative model.  This is a first small step towards finding such a method.

\section{Software and Numerical Experiments}
\label{sec:numerics}

We include in this section a few experiments where a preliminary implementation of HOSC outperforms SC, to demonstrate that higher-order affinities can bring a significant improvement over pairwise affinities in the context of manifold clustering.

In our implementation of HOSC, we used the heat kernel $\phi(s)=\exp(-s^2)$. Following the discussion in Section~\ref{sec:complexity},
at each point we restrict the computations to its $\ell$ nearest neighbors so that we practically remove the locality parameter $\epsilon$ from the affinity function of \eqref{eq:linear-affinity} and obtain
\begin{equation}
\label{eq:linear-affinity_noeps}
\alpha_{\ud}(\bx_1, \dots, \bx_{m}) = \begin{cases}
\phi\left({\Lambda_{\ud}(\bx_1, \dots, \bx_{m})}/{\eta}\right), & \mathrm{if}\, \bx_2, \ldots, \bx_{m} \in \textrm{$\ell$-NN}(\bx_1) \textrm{ distinct}; \\  0, & \textrm{otherwise},
\end{cases}
\end{equation}
where $\textrm{$\ell$-NN}(\bx_1)$ is the set of the $\ell$ nearest neighbors of $\bx_1$ .
For computational ease, we used
\begin{equation}
\Lambda_{\ud}^{(2)}(\bx_1, \dots, \bx_{m}) = \min_{L \in \cA_{\ud}} \ \sqrt{\frac{1}{m} \ \sum_{j=1}^m \dist(\bx_j, L)^2},
\end{equation}
which can be easily computed using the bottom $m-\ud$ singular values of the $m$ points.  Note that, since
$\Lambda_{\ud}/\sqrt{m}  \leq \Lambda_{\ud}^{(2)} \leq
\Lambda_{\ud},$ the results we obtained apply, with $\eta$ changed
by a $\sqrt{m}$ factor, at most. (In the paper, the standard choice for $\eta$ is a power of $N$, while $m$ is of order at most $\log N$, so this factor is indeed negligible.) In practice, we always search a subinterval of $[0,1]$ for the best working $\eta$ (e.g., $[.001, .1]$), based on the smallest variance of the corresponding clusters
in the eigenspace (the row space of the matrix $\mathbf{V}$), as suggested in~\cite{Ng02}. When the given data contains outliers, the optimal choice of $\eta$ is based on the largest gap between the means of the two sets of degrees (associated to the inliers and outliers), normalized by the maximum degree. The code is available online~\cite{hosc}.

\subsection{Synthetic Data}
We first generate five synthetic data sets in the unit cube $(0,1)^D$ ($D=2$ or $3$), shown in Figure~\ref{fig:artificial_data}. In this experiment, the actual number of clusters (i.e.~$K$) and dimension of the
underlying manifolds (i.e.~$d$) are assumed known to all algorithms. For HOSC, we fix $\ell = 10, m=d+2$, and use the subinterval $[0.001, 0.1]$ as the search interval of $\eta$.
For SC, we considered
two ways of tuning the scale parameter $\epsilon$: directly, by
choosing a value in the interval $[0.001, 0.25]$ (SC-NJW); and by the local scaling method of~\cite{Zelnik-Manor04} (SC-LS), with the number of nearest neighbors $\ell
= 5, \dots, 15$. The final choices of these parameters were also based on the same criterion as used by HOSC.

Figure~\ref{fig:artificial_data} exhibits the clusters found by each algorithm when applied to the five data sets, respectively. Observe that HOSC succeeded in
a number of difficult situations for SC, e.g.,
when the sampling is sparse, or when the separation is small at
some locations.
%
%

\begin{figure}[htbp]
     \centering
     \includegraphics[width=.32\linewidth]{figures/threecircles_data.pdf}
     \includegraphics[width=.32\linewidth]{figures/threecircles_njw.pdf}
     \includegraphics[width=.32\linewidth]{figures/threecircles_lscc.pdf}
     \includegraphics[width=.32\linewidth]{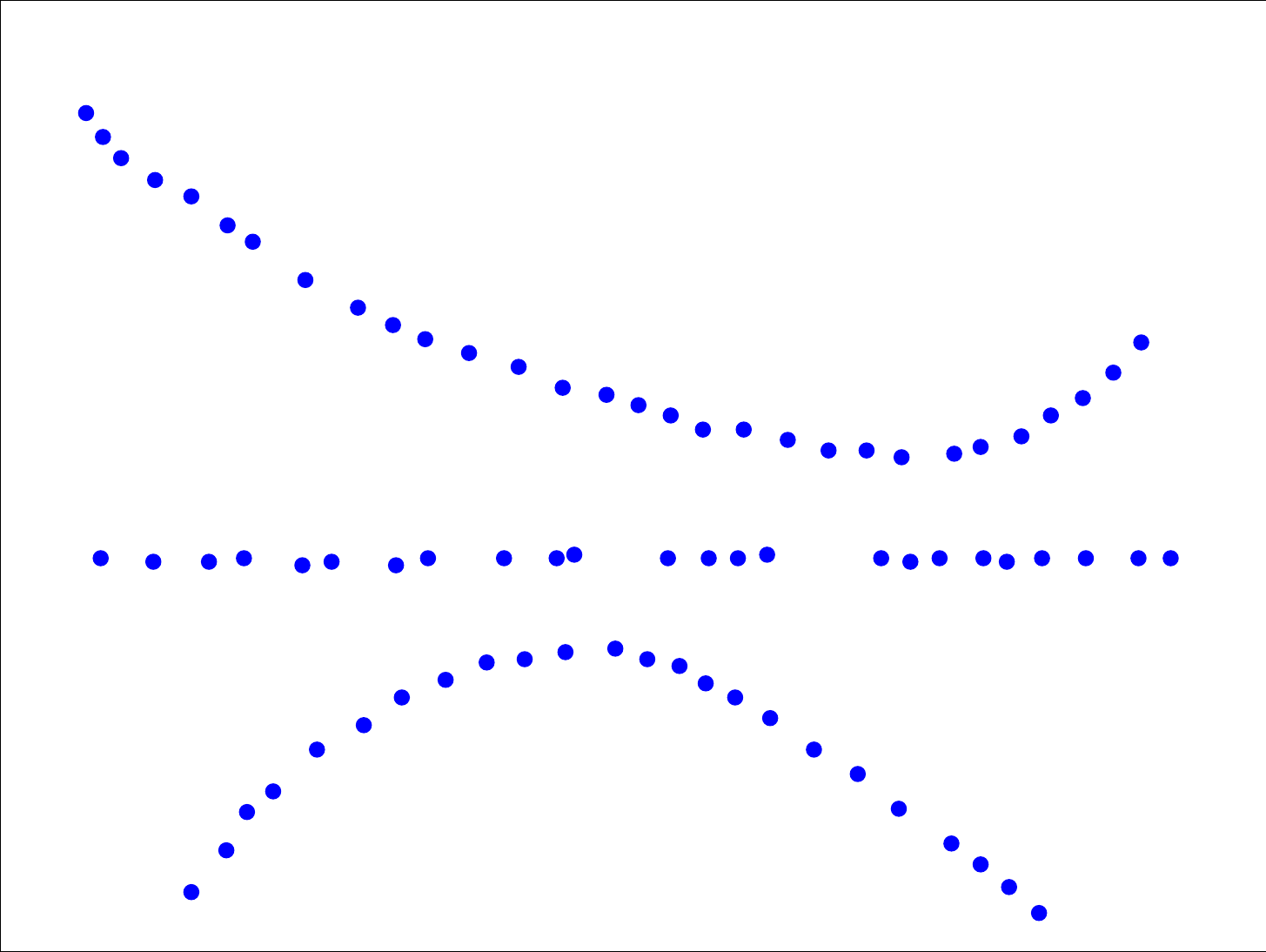}
     \includegraphics[width=.32\linewidth]{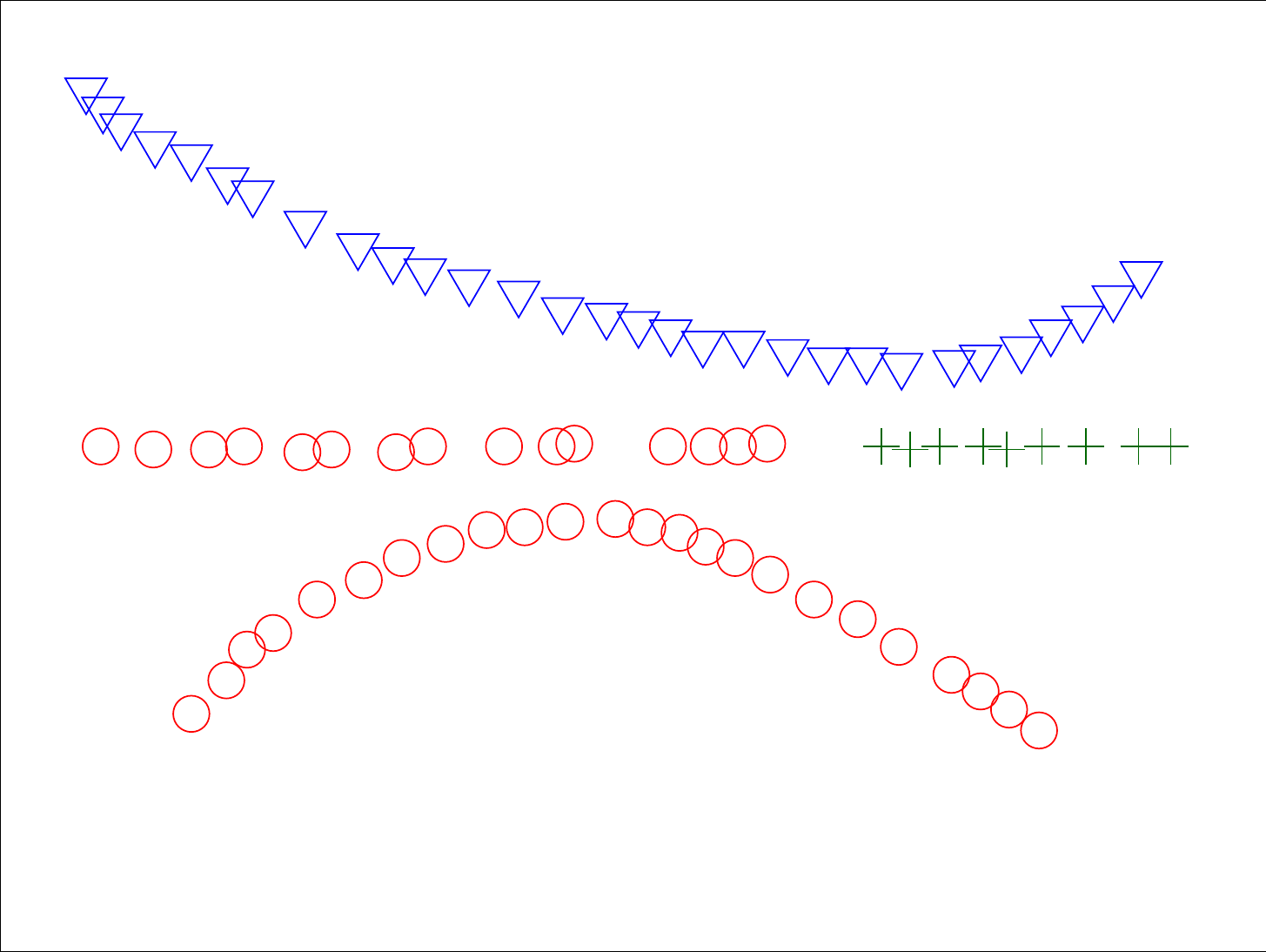}
     \includegraphics[width=.32\linewidth]{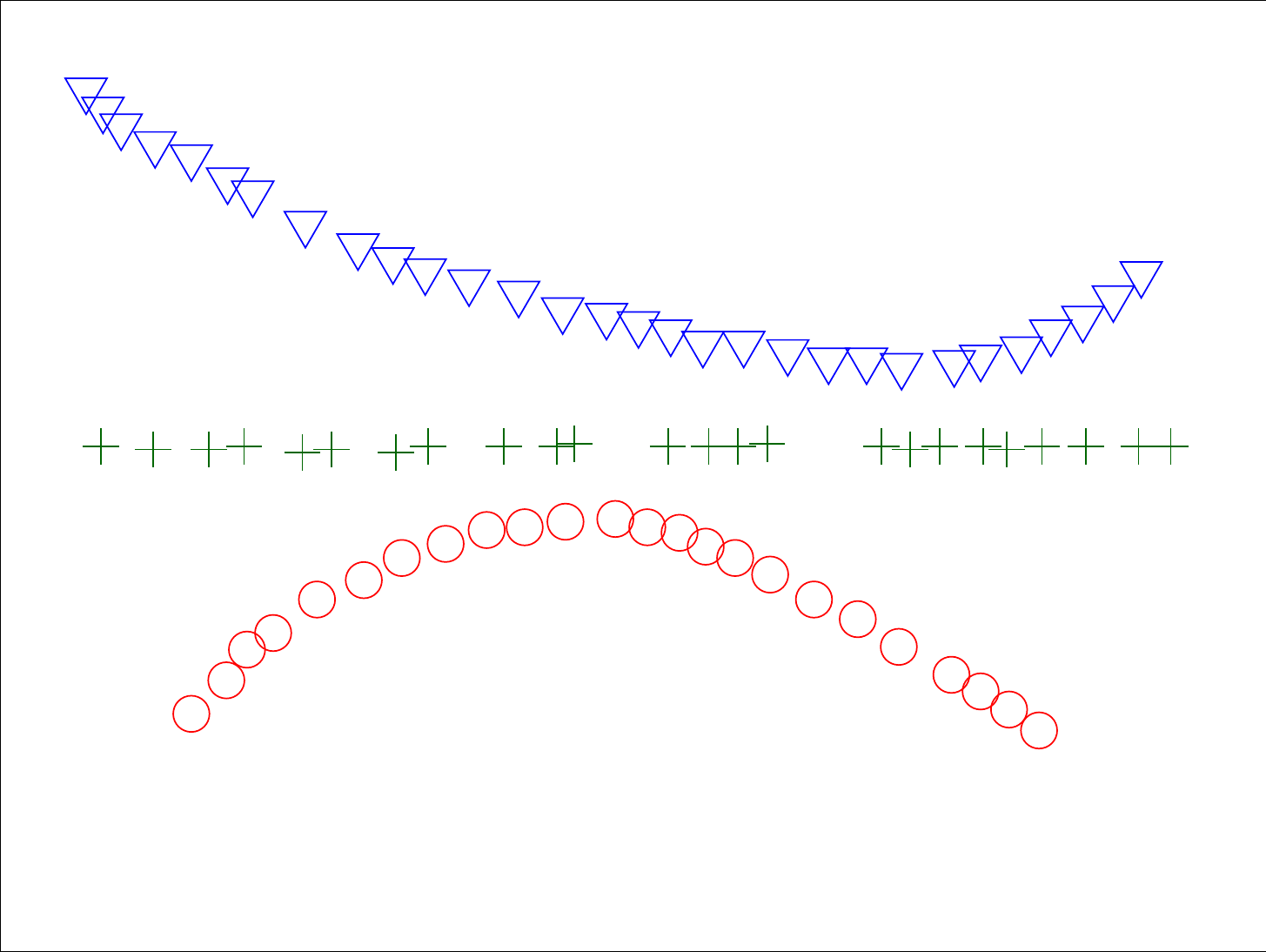}
     \includegraphics[width=.32\linewidth]{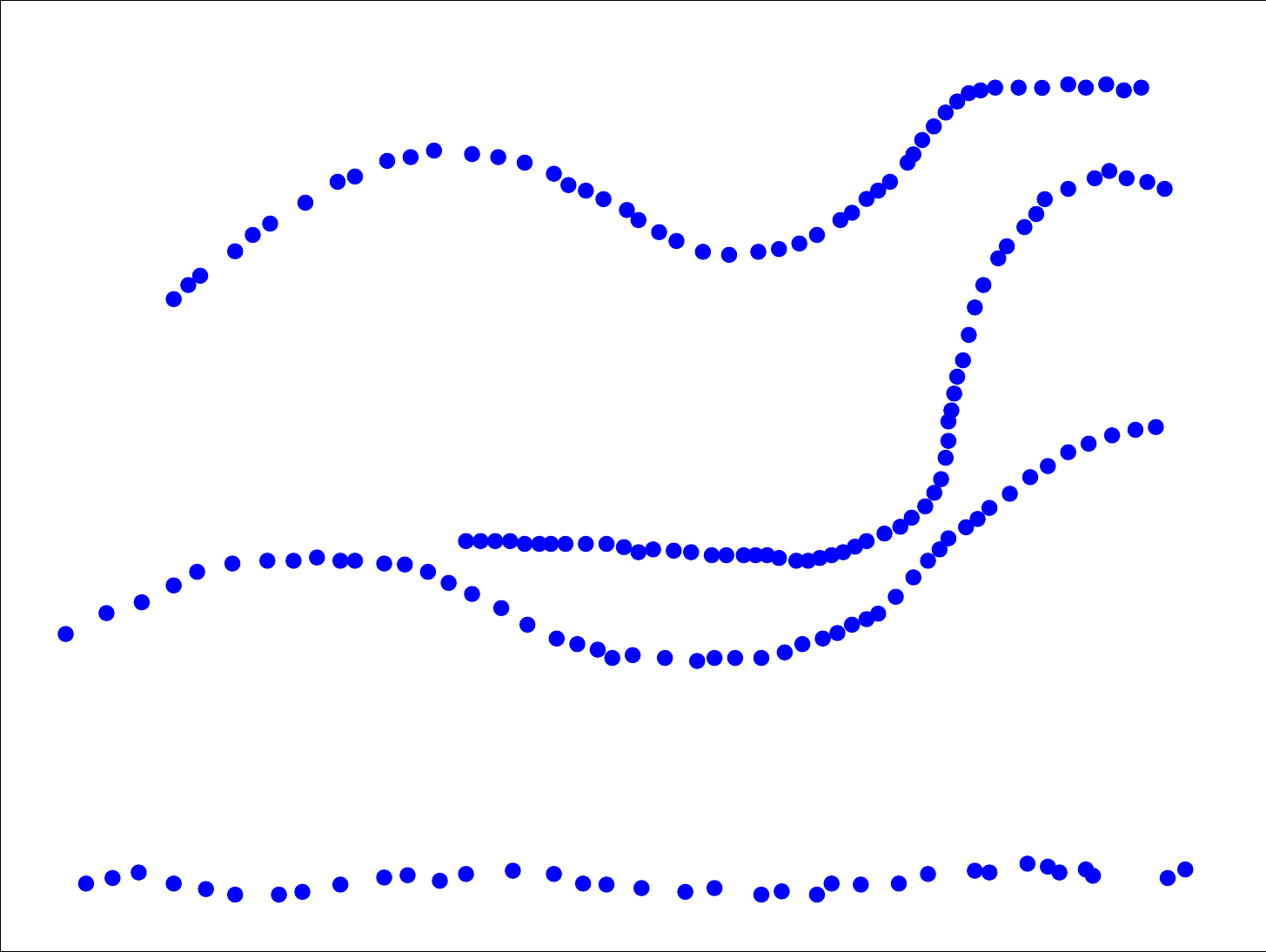}
     \includegraphics[width=.32\textwidth]{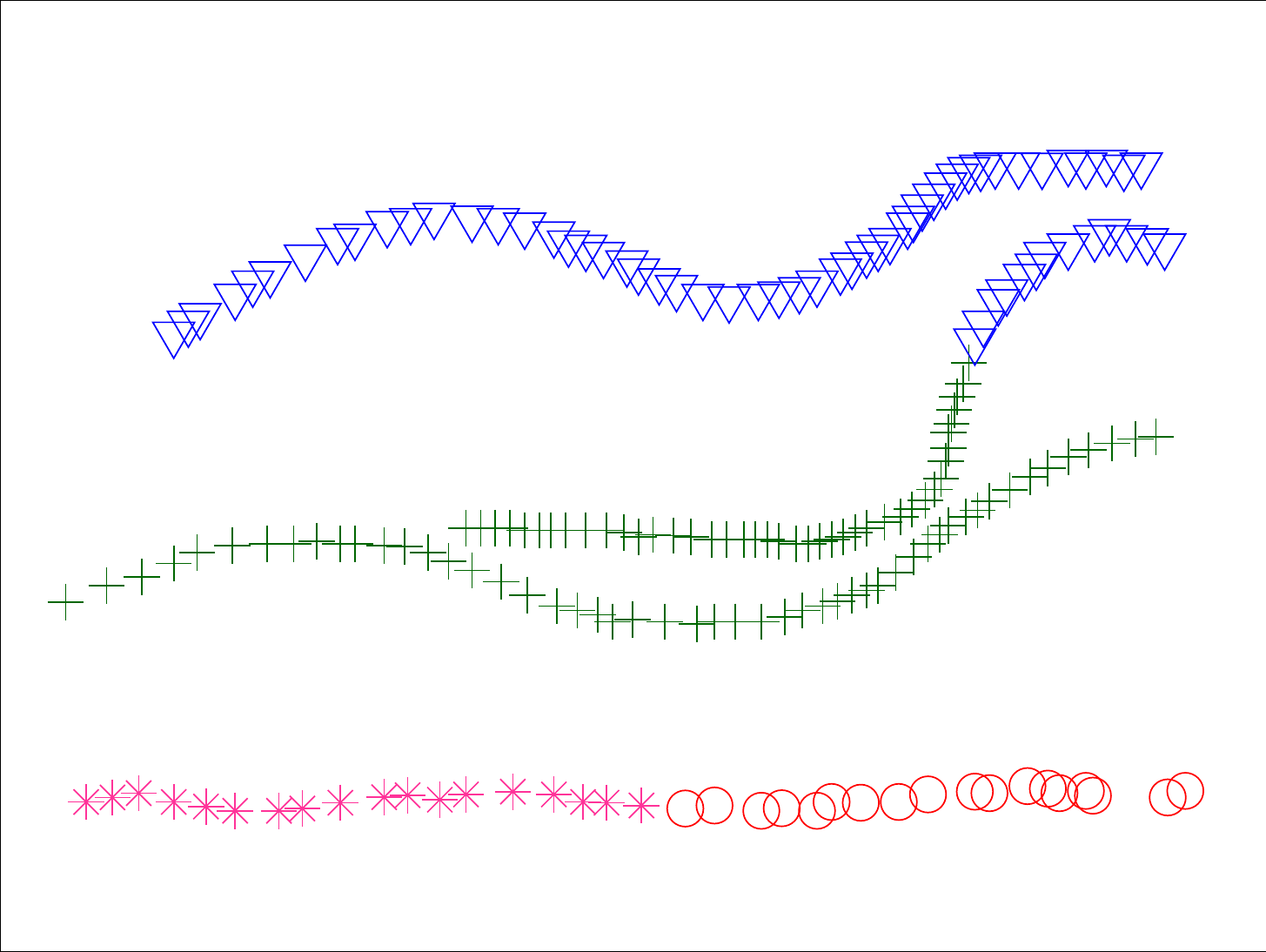}
     \includegraphics[width=.32\textwidth]{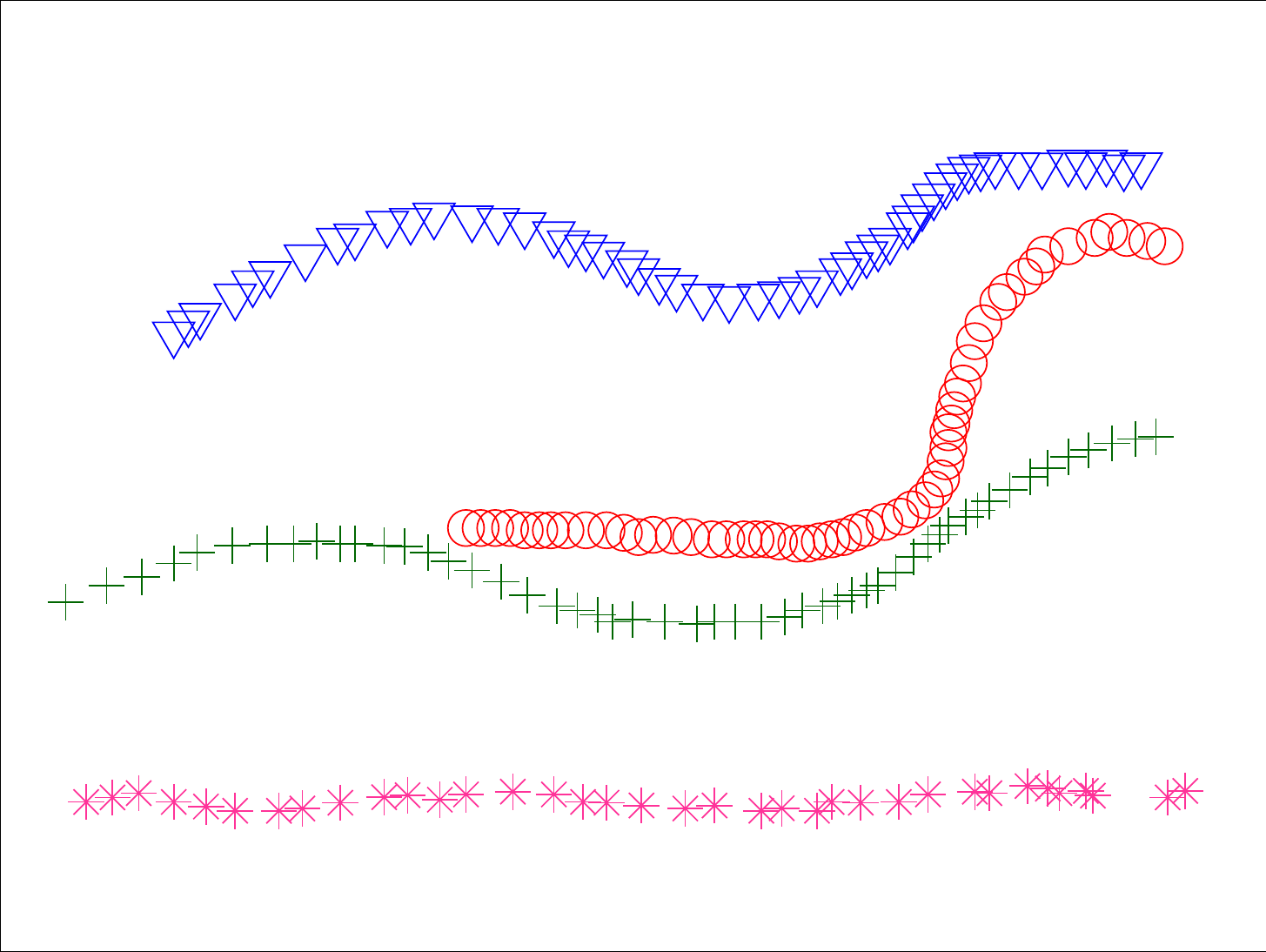}
     \includegraphics[width=.32\linewidth]{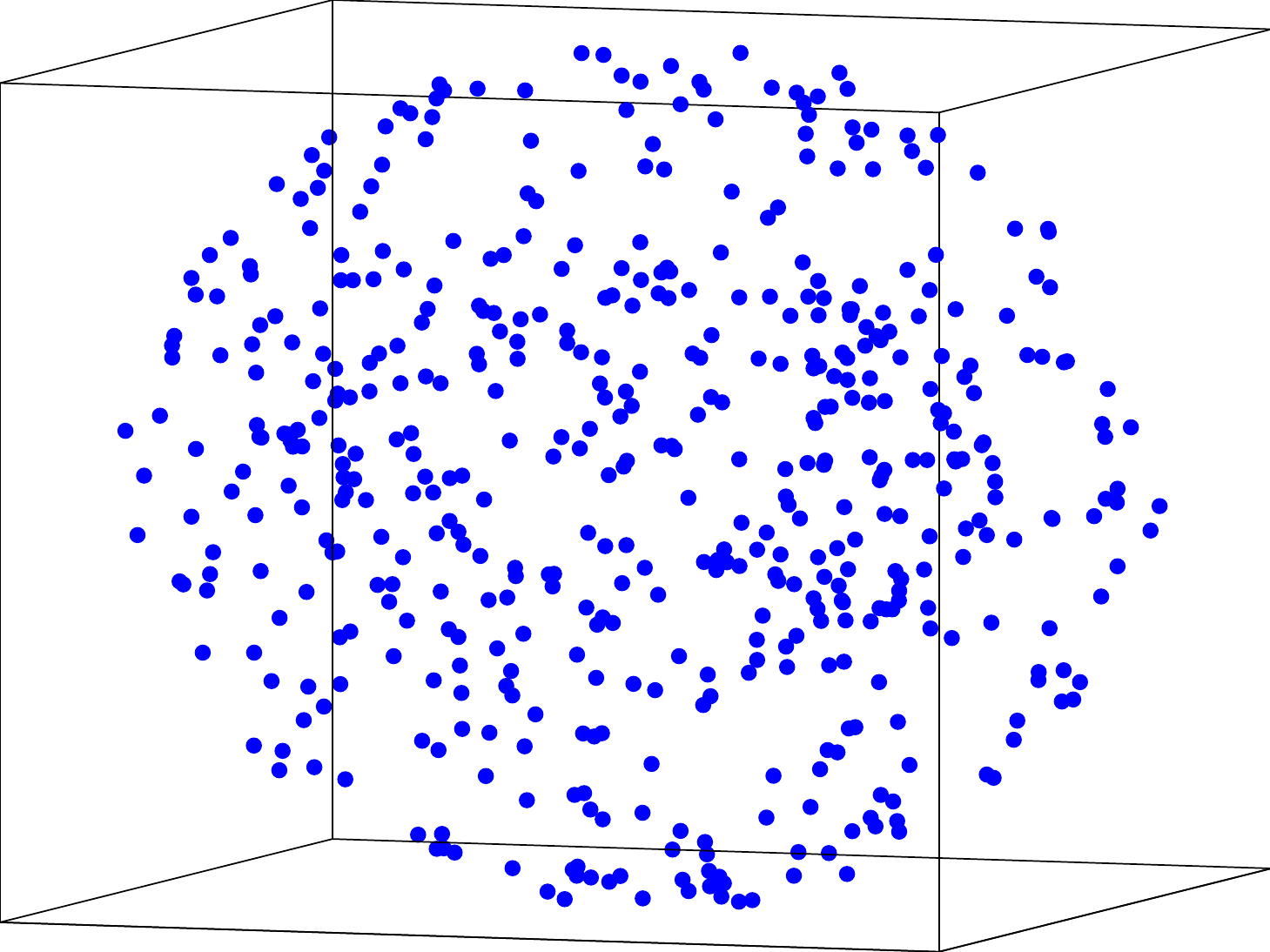}
     \includegraphics[width=.32\textwidth]{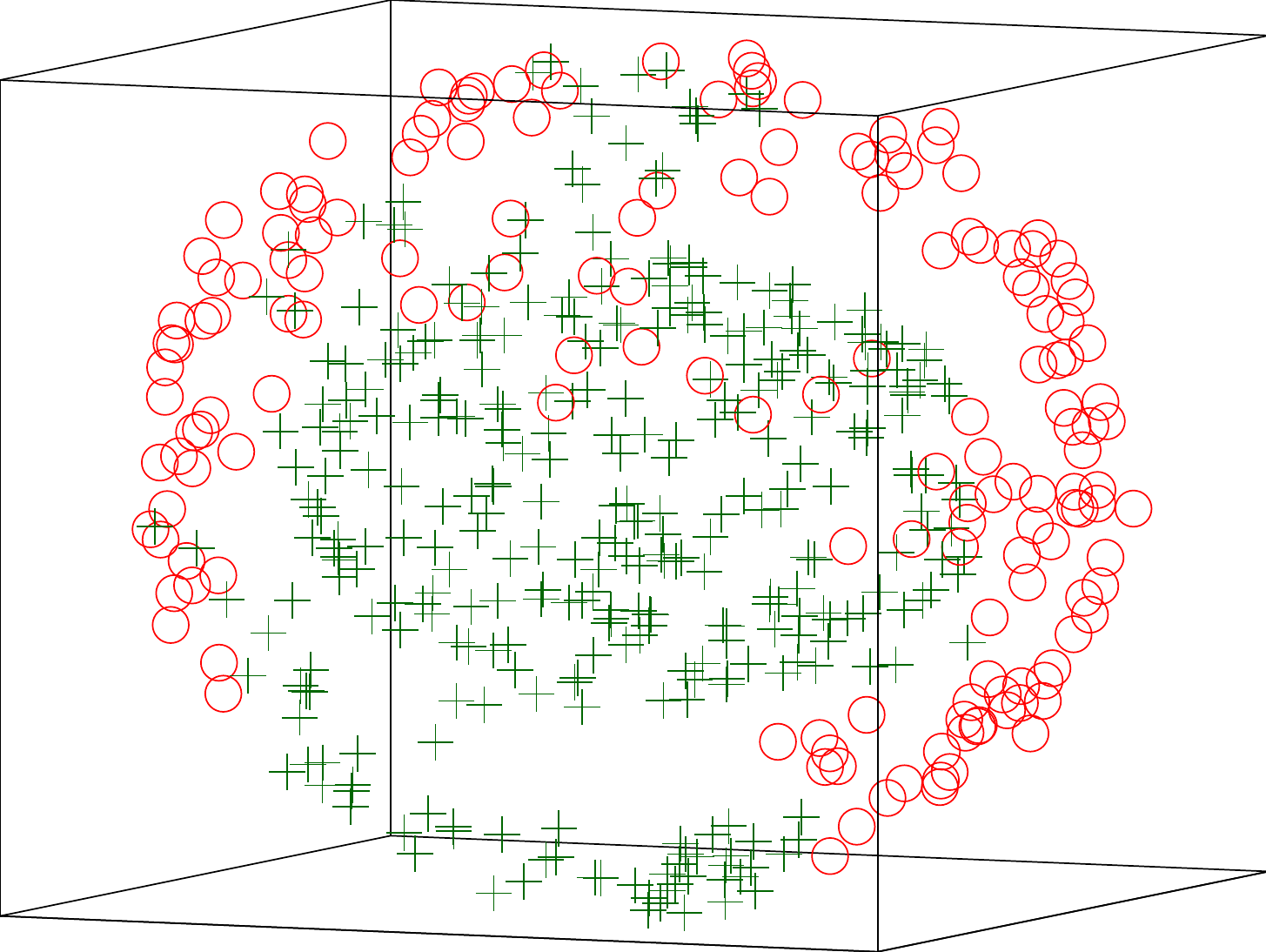}
     \includegraphics[width=.32\textwidth]{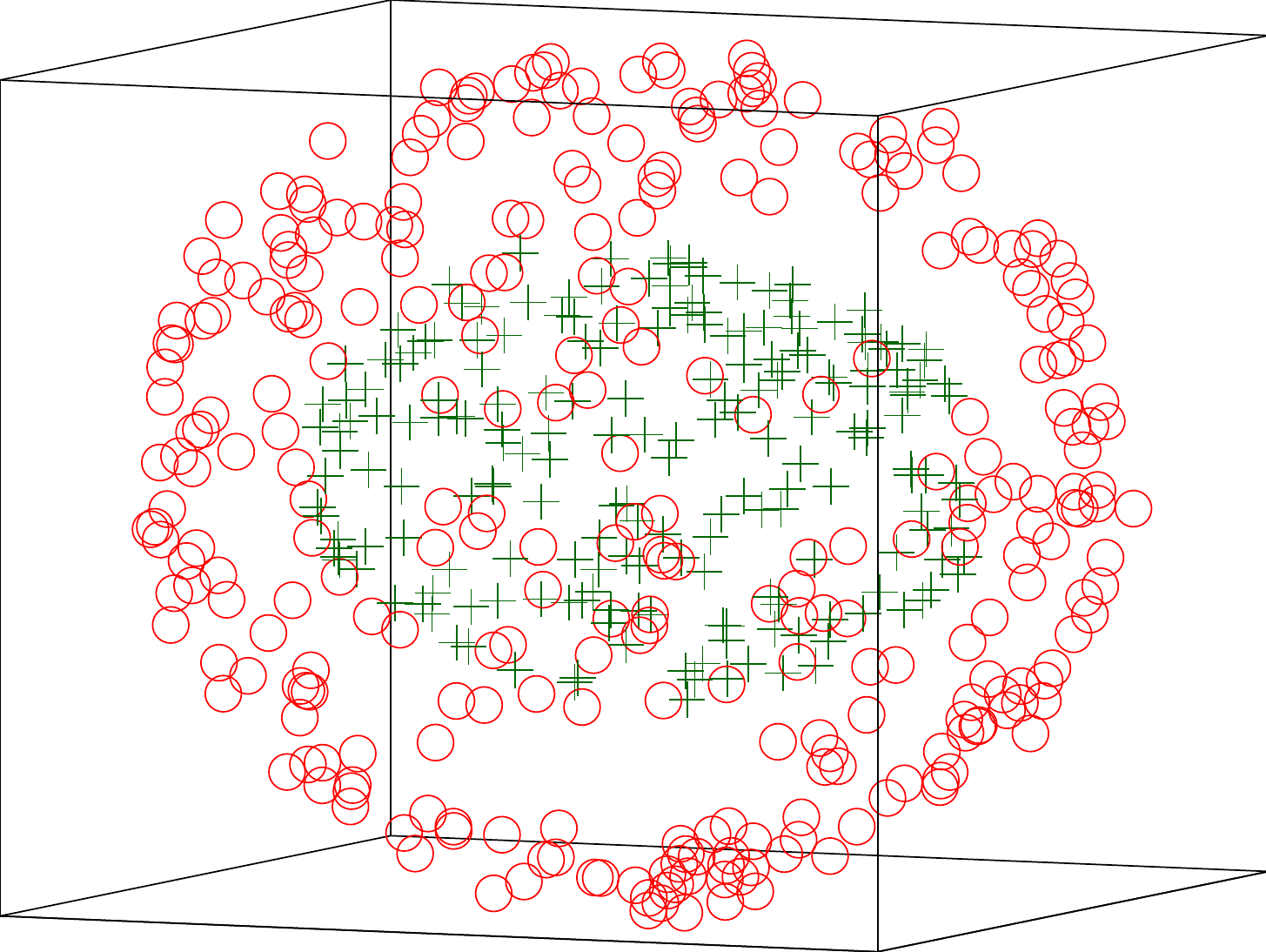}
     \includegraphics[width=.32\linewidth]{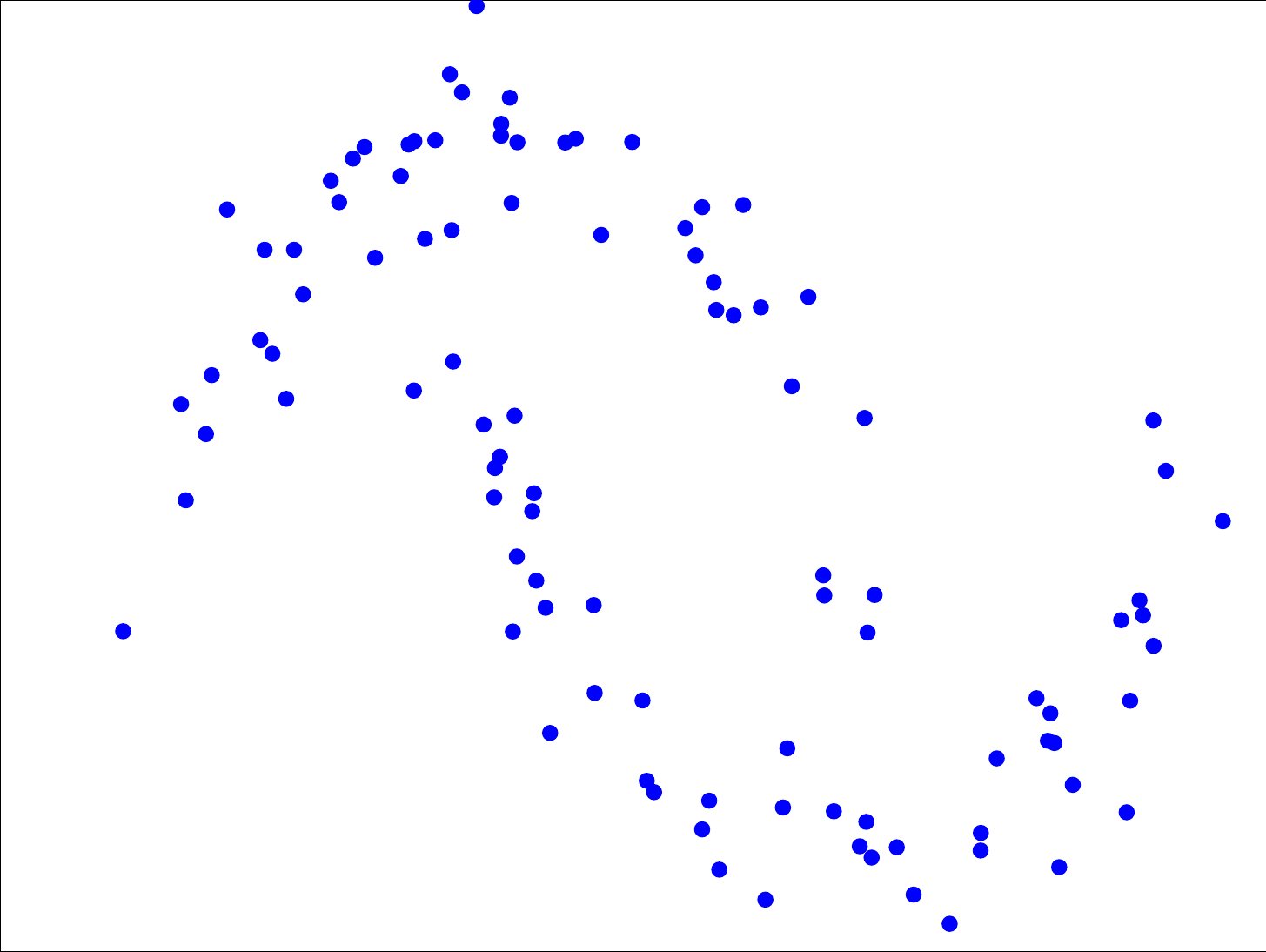}
     \includegraphics[width=.32\linewidth]{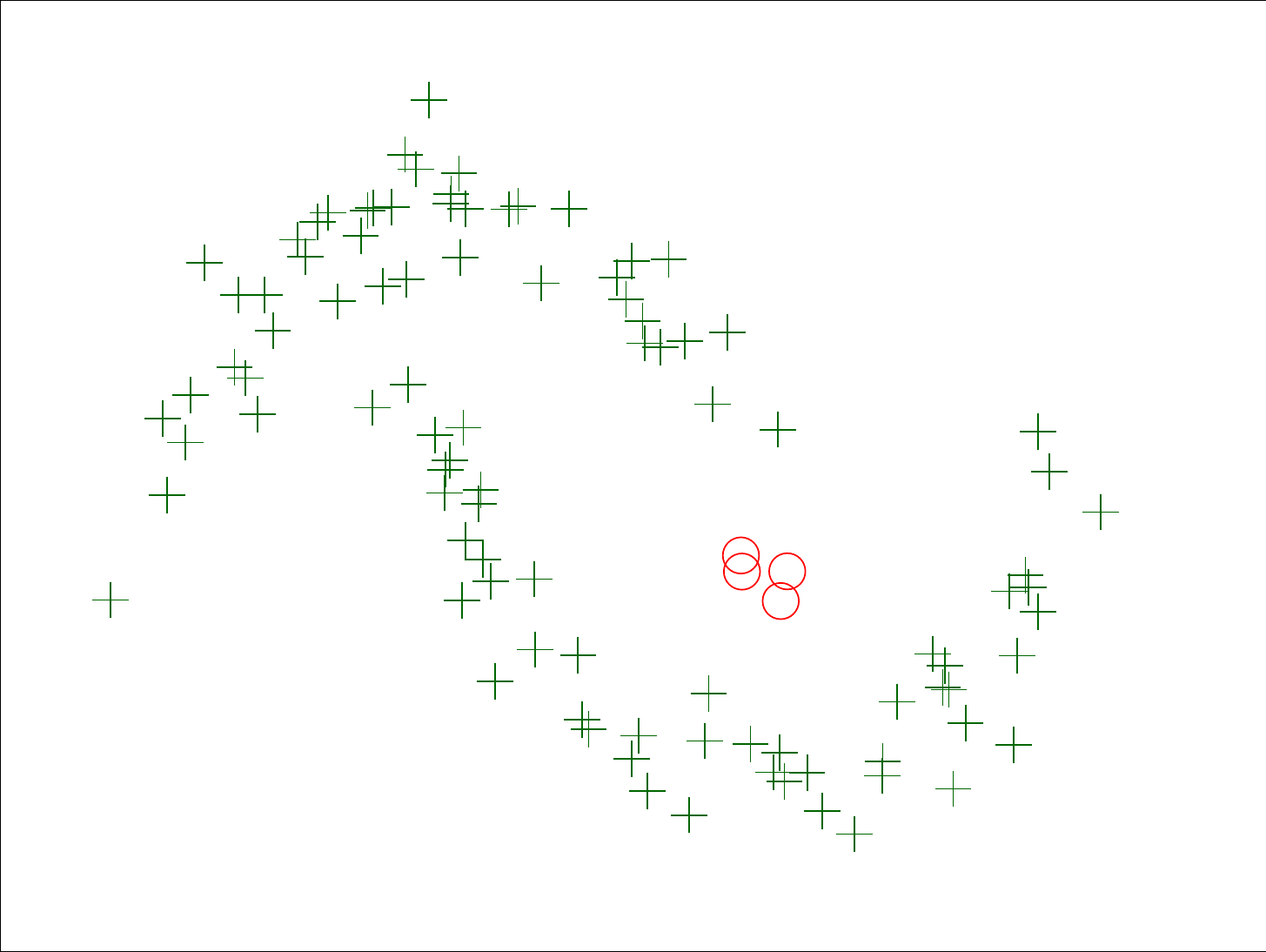}
     \includegraphics[width=.32\linewidth]{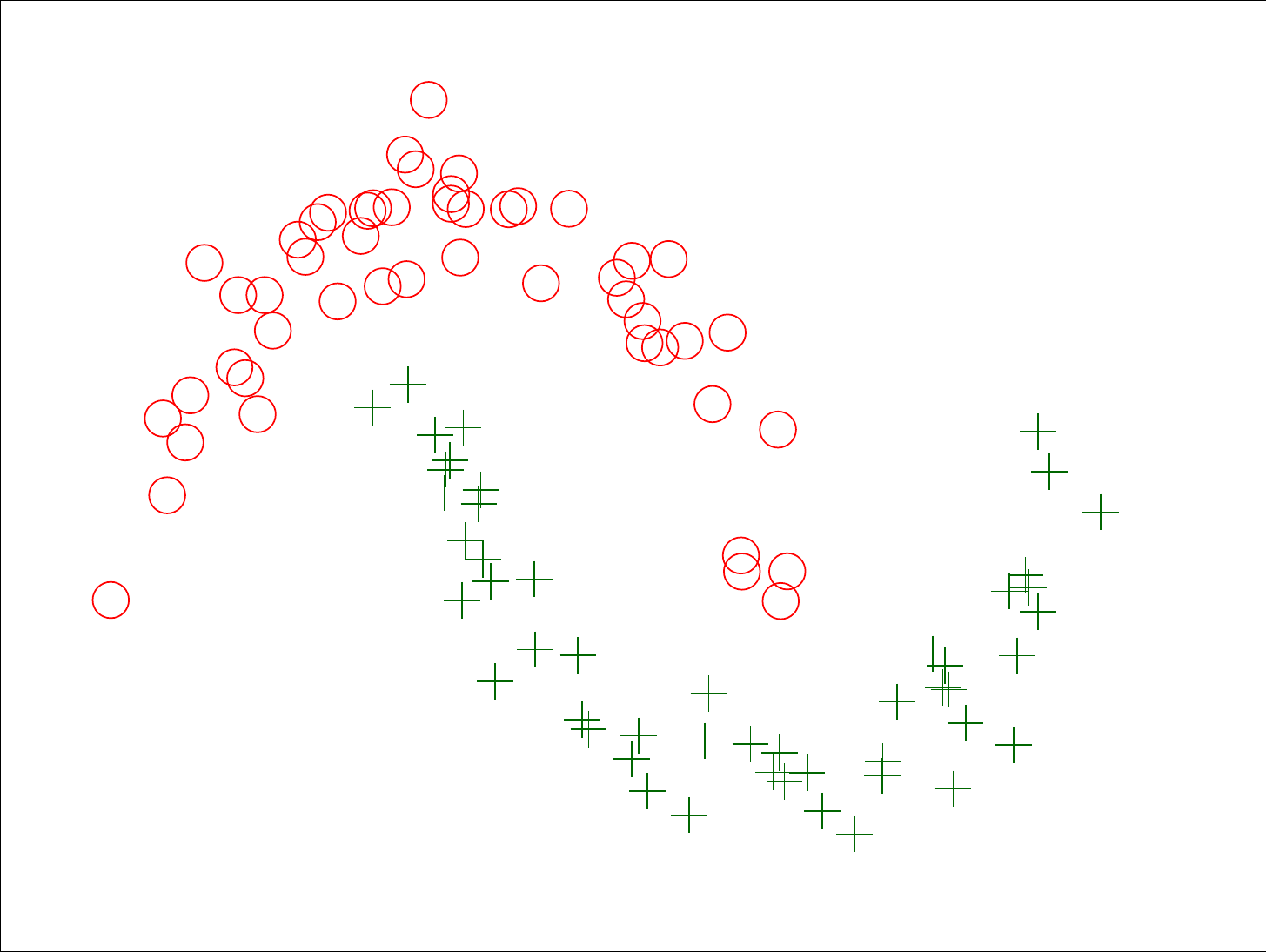}
     \caption{Left column: data.  (The third example shows a sphere containing an ellipsoid inside.)  Middle column: best output from SC with the scale parameter chosen by both searching the interval $[0.001, 0.25]$ and applying local scaling~\cite{Zelnik-Manor04} with at most 15 nearest neighbors.  Right column: output from HOSC. The optimal value of $\eta$ is selected from the interval $[0.001, 0.1]$. We also tried the simple kernel instead of the heat kernel, and obtained same results except in data set 3.}
     \label{fig:artificial_data}
\end{figure}

We also plot the leading eigenvalues of the matrix $\bZ$ obtained by HOSC on each data set; see Figure~\ref{fig:eigs_hosc_syndata}. We see that in data sets 1, 2, 5, the number of eigenvalue 1 coincides with the true number of clusters, while in 3 and 4 there is some discrepancy between the $K$th eigenvalue and the number 1. Though we do not expect the eigengap method to work well in general, Figure~\ref{fig:eigs_hosc_syndata} shows that it can be useful in some cases.

\begin{figure}[htbp]
     \centering
     \includegraphics[width=.32\linewidth]{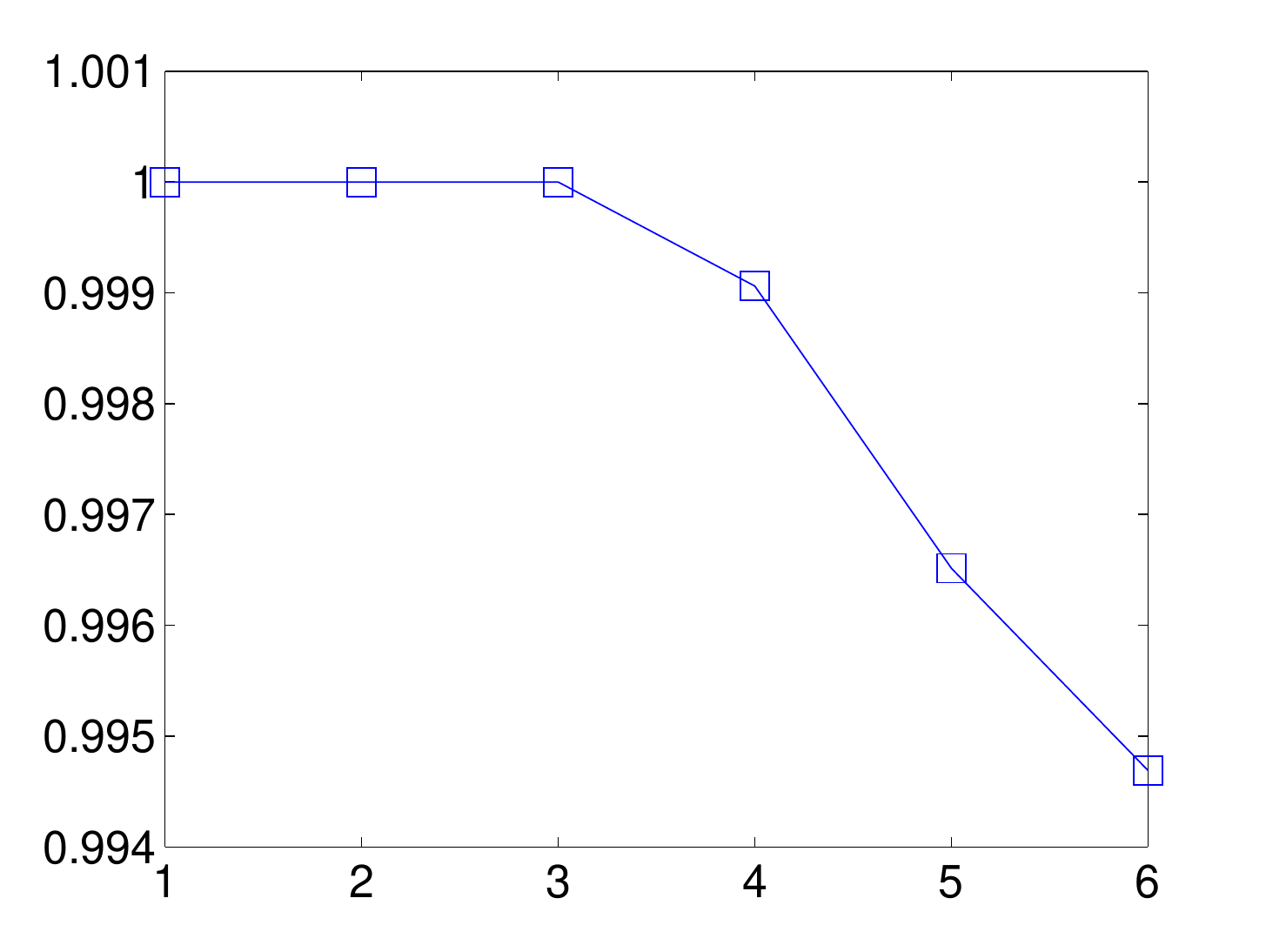}
     \includegraphics[width=.32\linewidth]{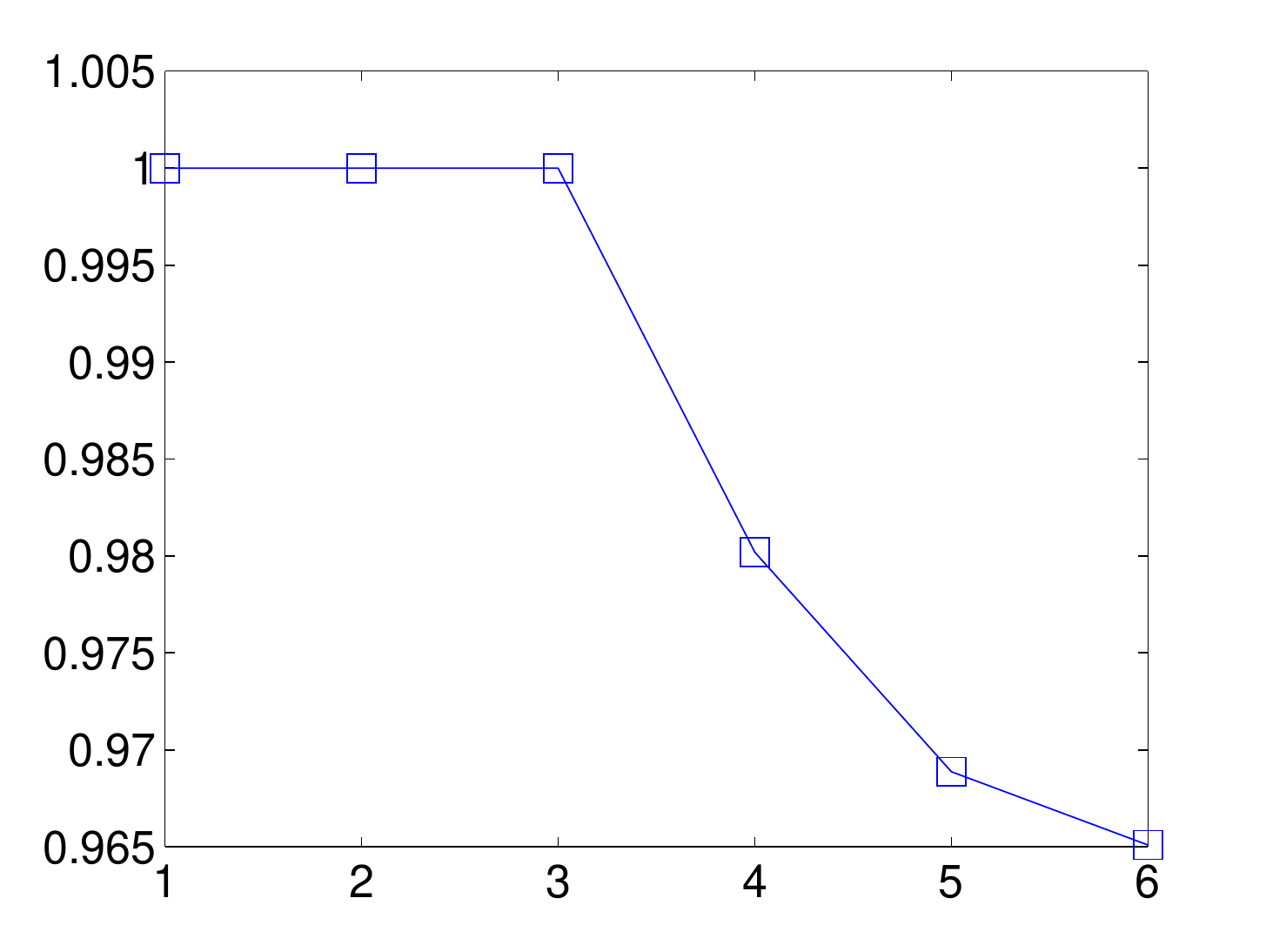}
     \includegraphics[width=.32\linewidth]{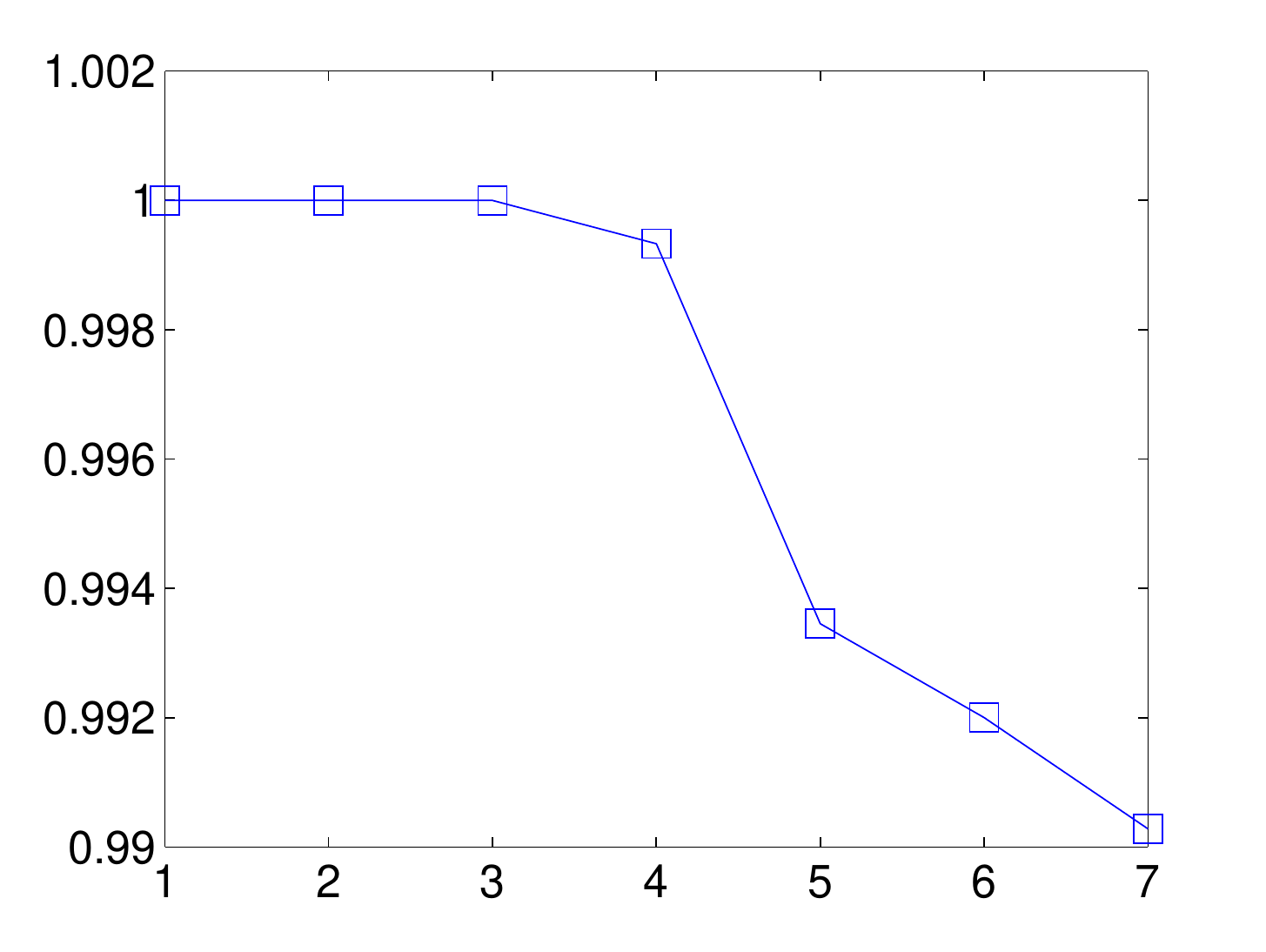}\\
     \includegraphics[width=.32\linewidth]{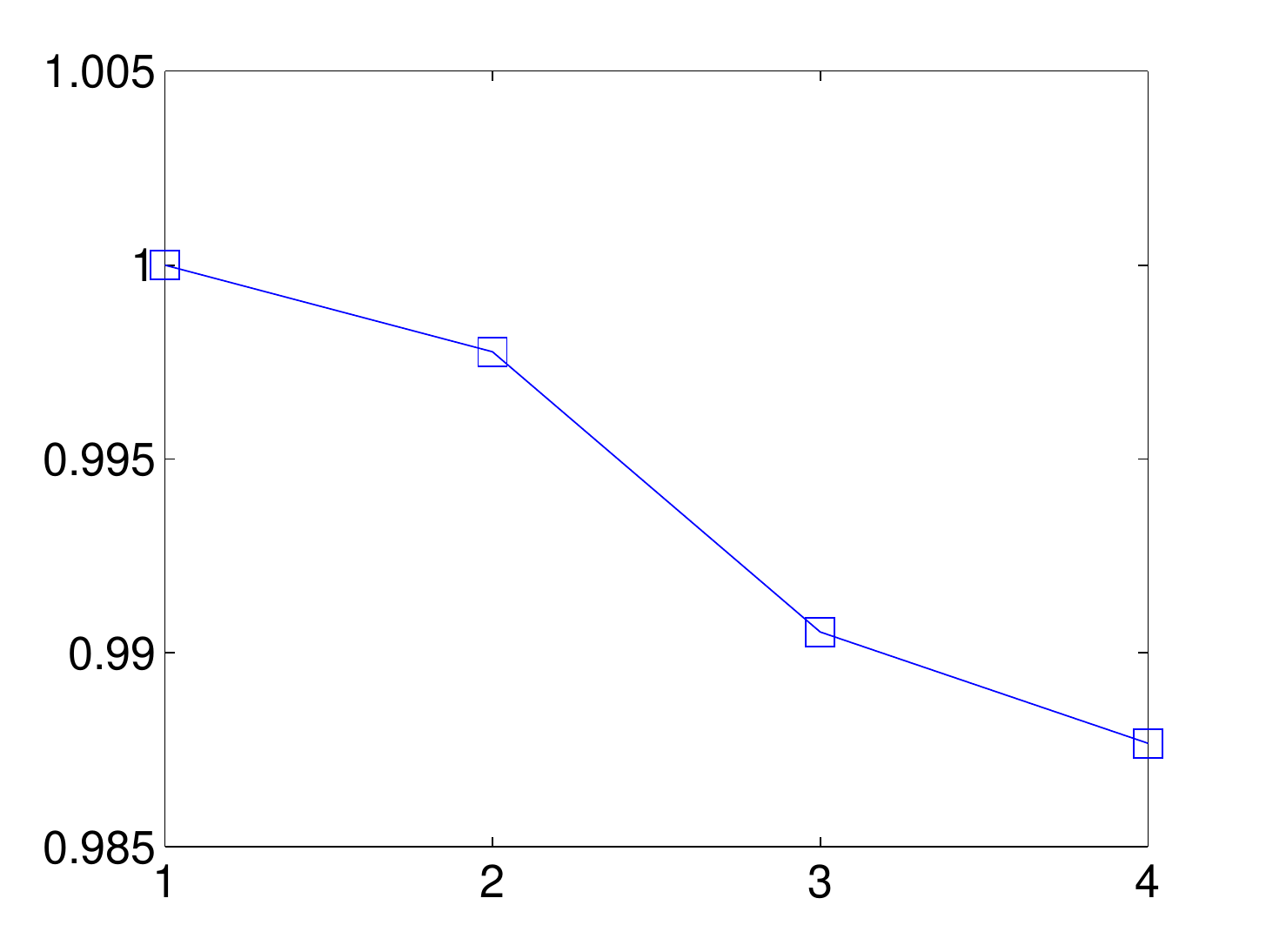}
     \includegraphics[width=.32\linewidth]{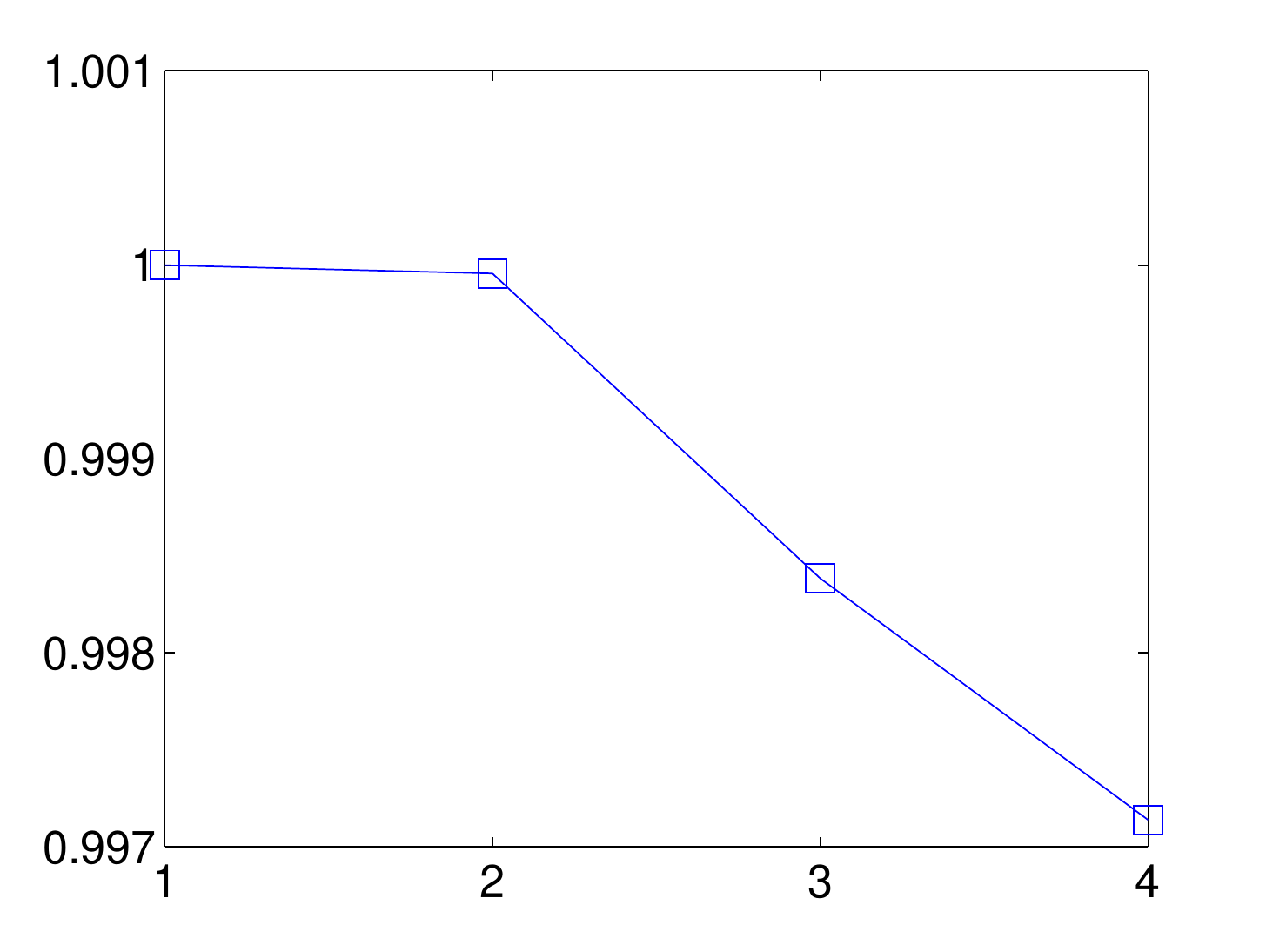}\\
     \caption{Top eigenvalues of the matrix $\bZ$ obtained by HOSC on each of the five data sets in Figure~\ref{fig:artificial_data} (in same order).}
     \label{fig:eigs_hosc_syndata}
\end{figure}

Figure~\ref{fig:artificial_data_outliers} displays some experiments including outliers.  We simply sampled points from the unit square $(0,1)^2$ uniformly at random and added them as outliers to the first three data sets in Figure
\ref{fig:artificial_data}, with percentages 33.3\%, 60\% and 60\%, respectively.
We applied SC and HOSC assuming knowledge of the proportion of outliers, and labeled points with smallest degrees as outliers. Choosing the threshold automatically remains a challenge; in particular, we
did not test the theory.

\begin{figure}[htbp]
\centering
\includegraphics[width=.32\linewidth]{figures/threecircles_data_outliers.pdf}
\includegraphics[width=.32\linewidth]{figures/threecircles_njw_outliers.pdf}
\includegraphics[width=.32\linewidth]{figures/threecircles_lscc_outliers.pdf}
\includegraphics[width=.32\linewidth]{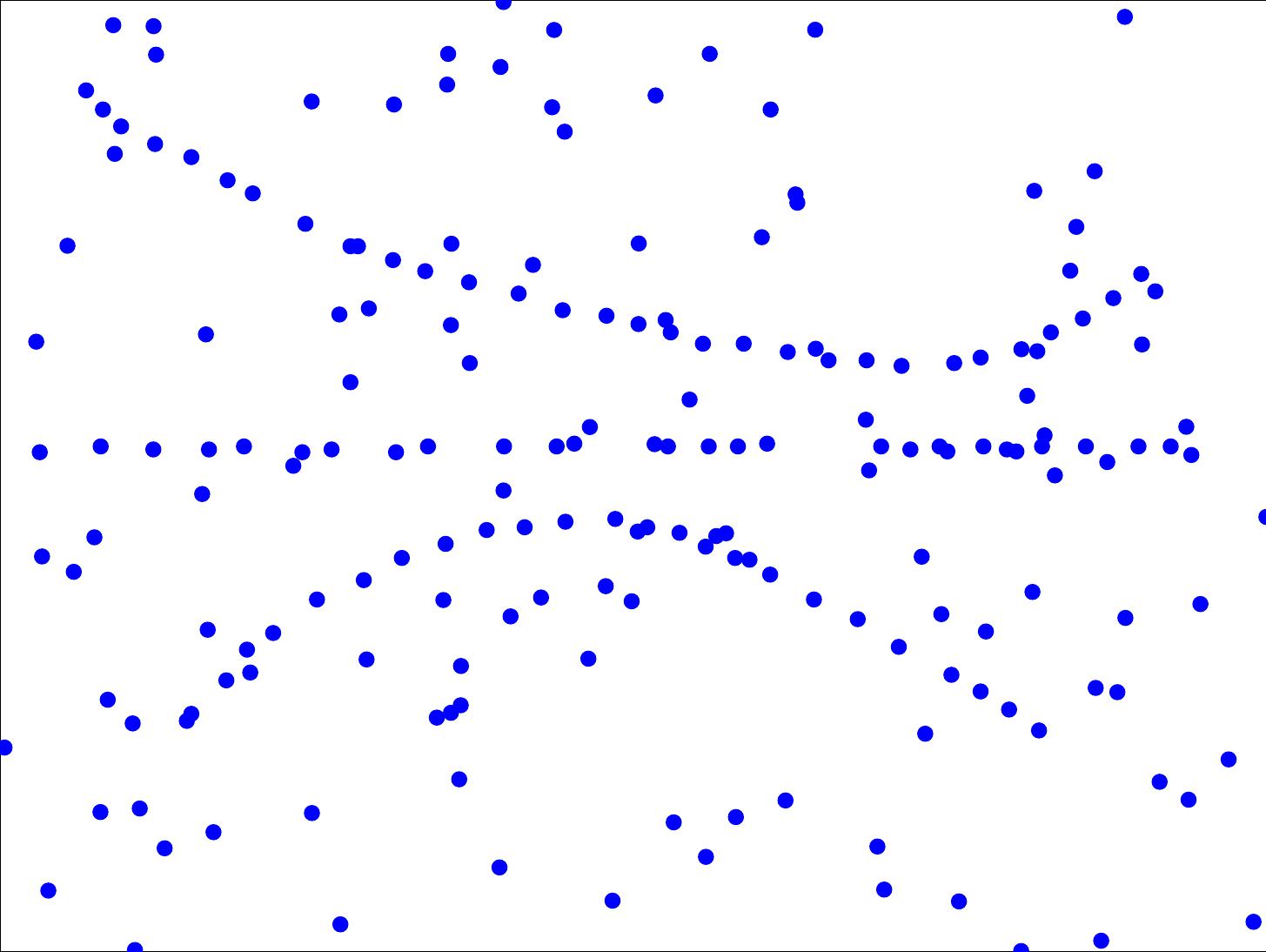}
\includegraphics[width=.32\linewidth]{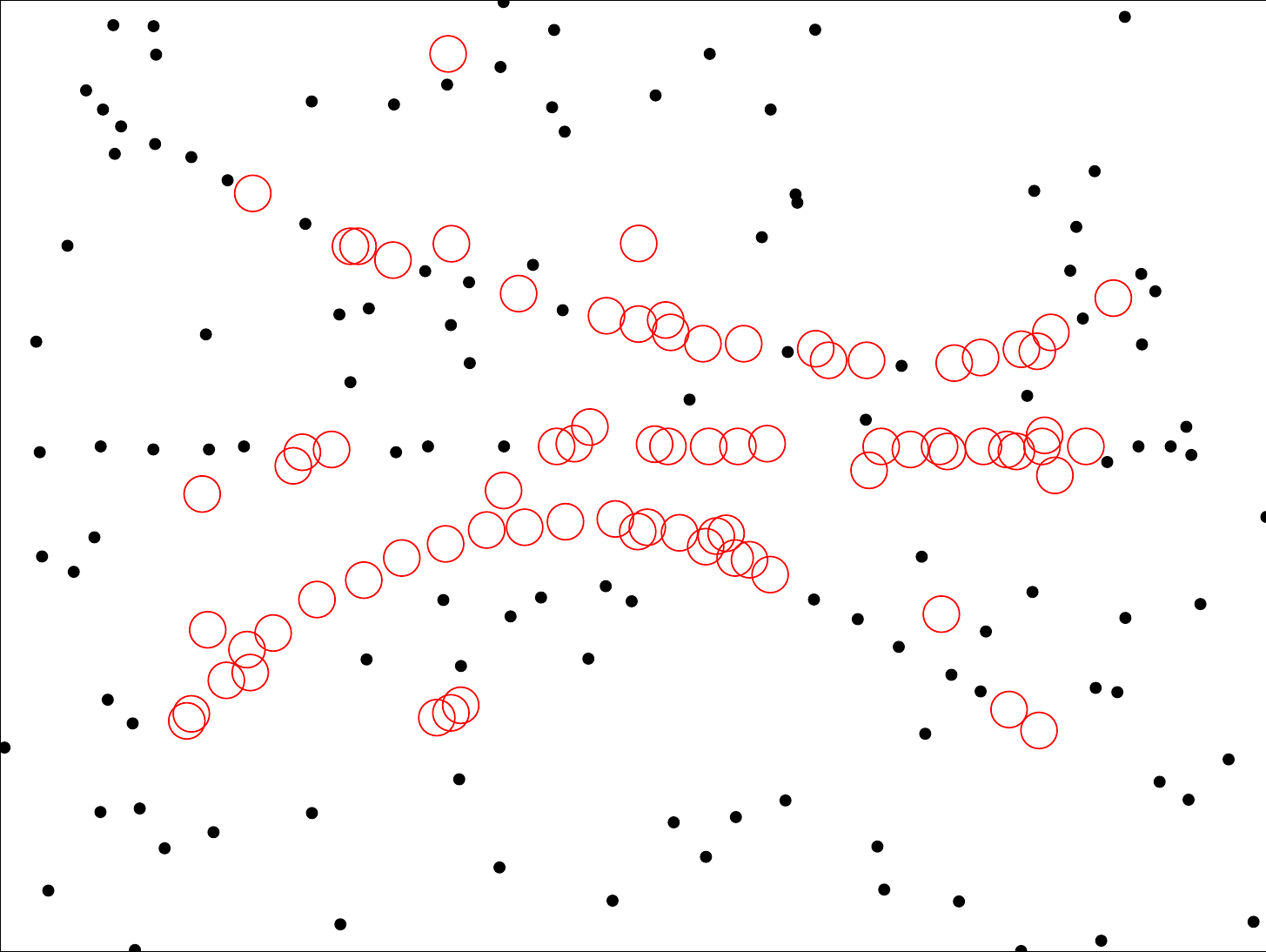}
\includegraphics[width=.32\linewidth]{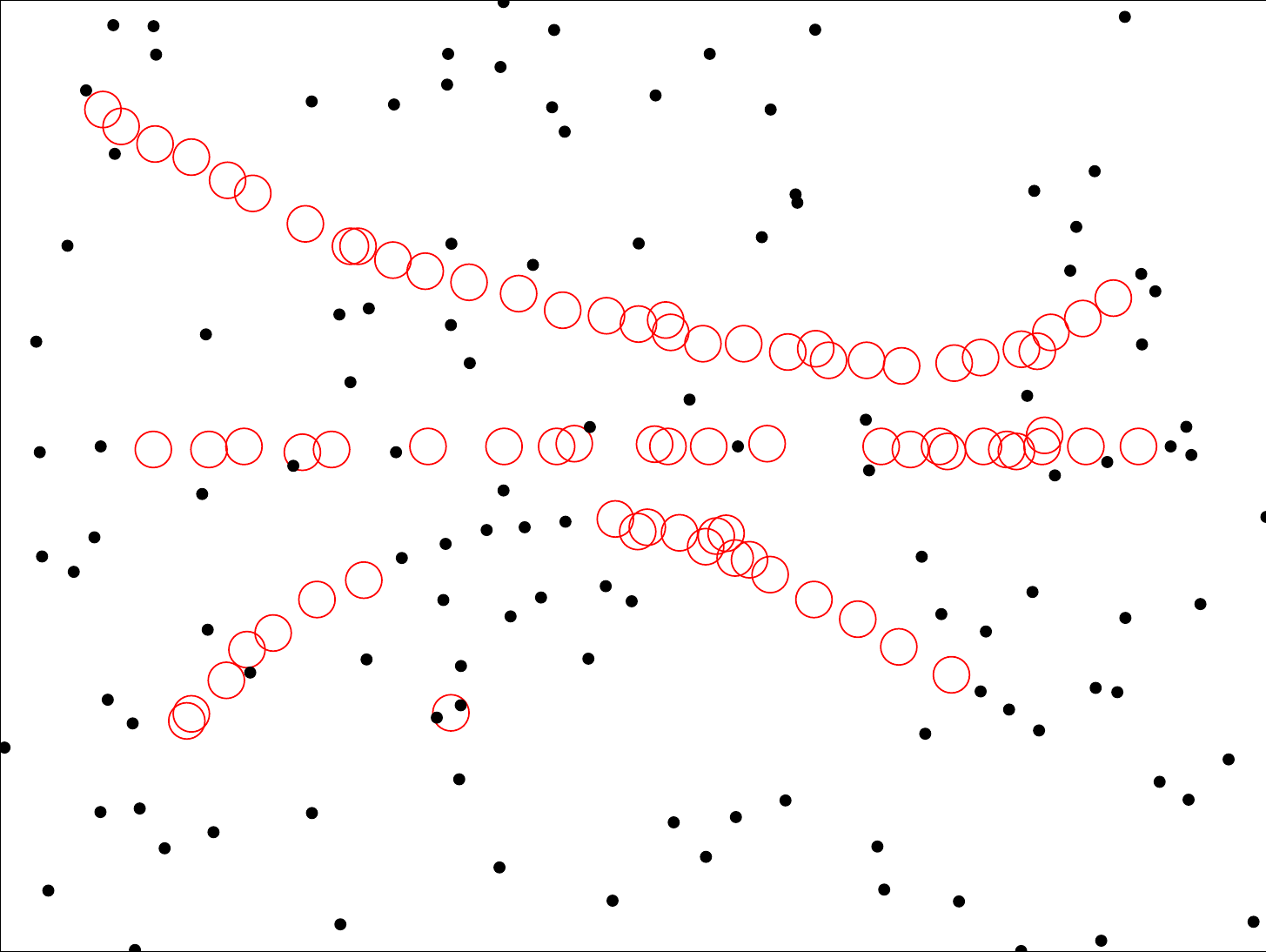}
\includegraphics[width=.32\linewidth]{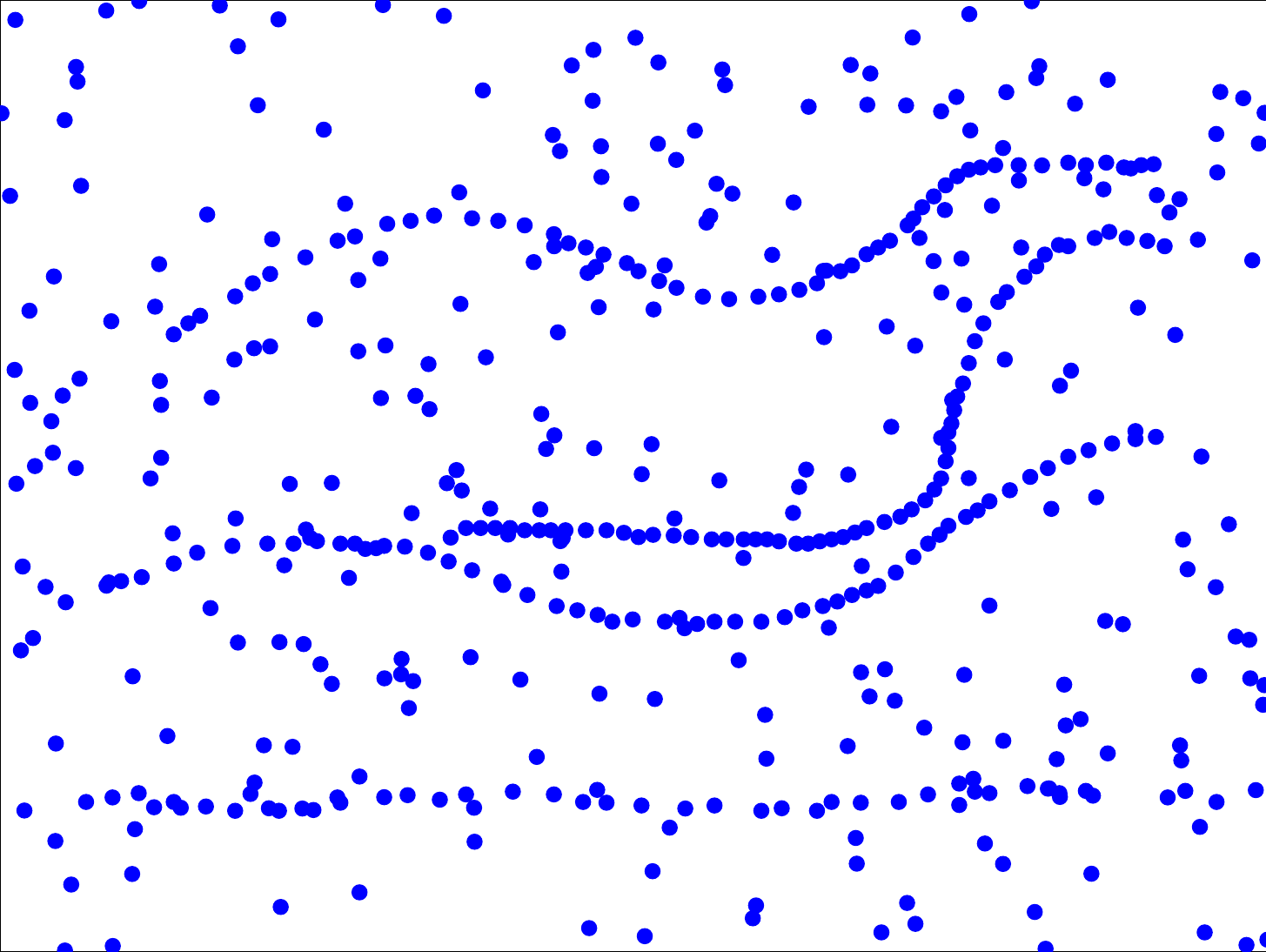}
\includegraphics[width=.32\linewidth]{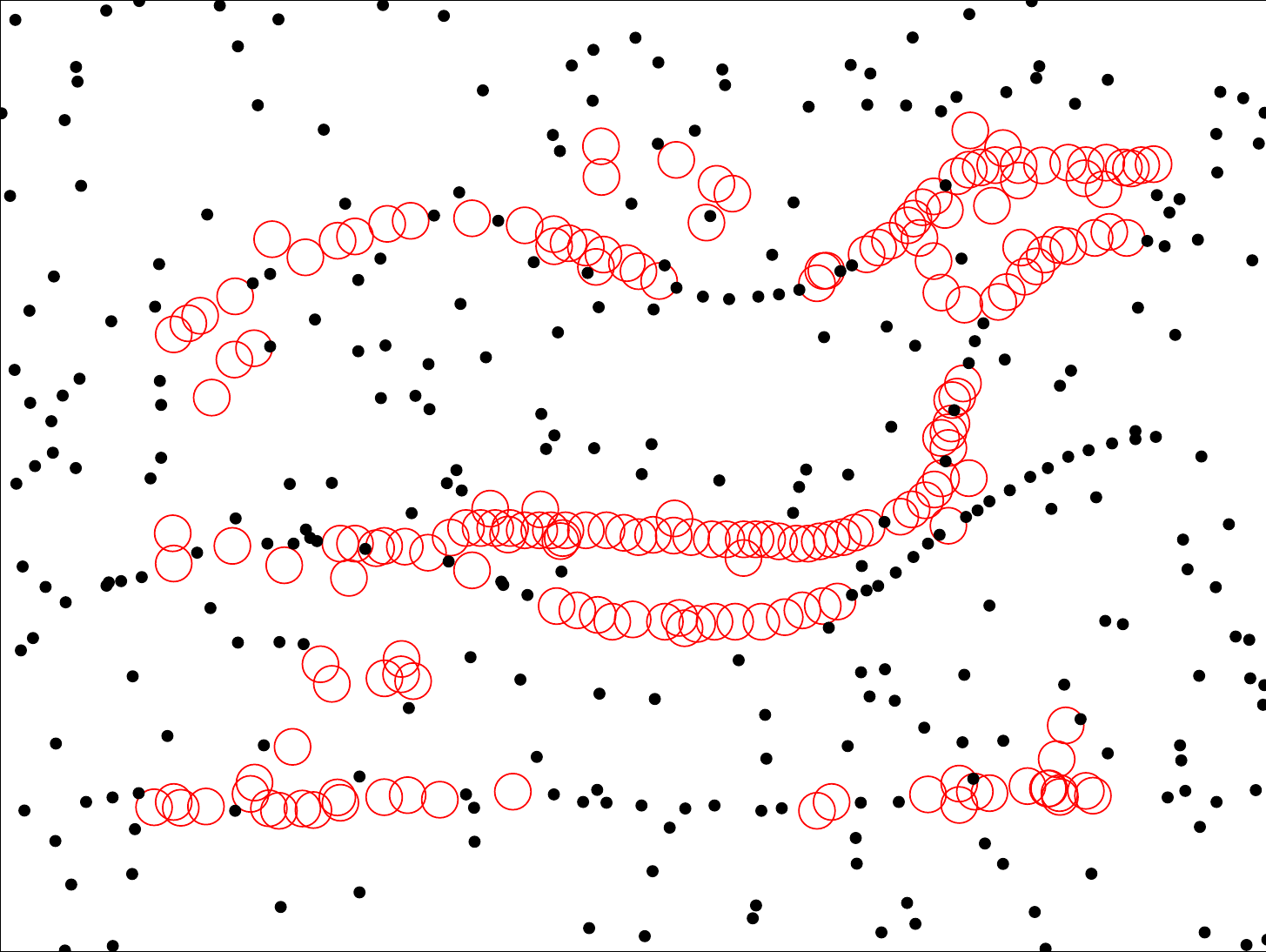}
\includegraphics[width=.32\linewidth]{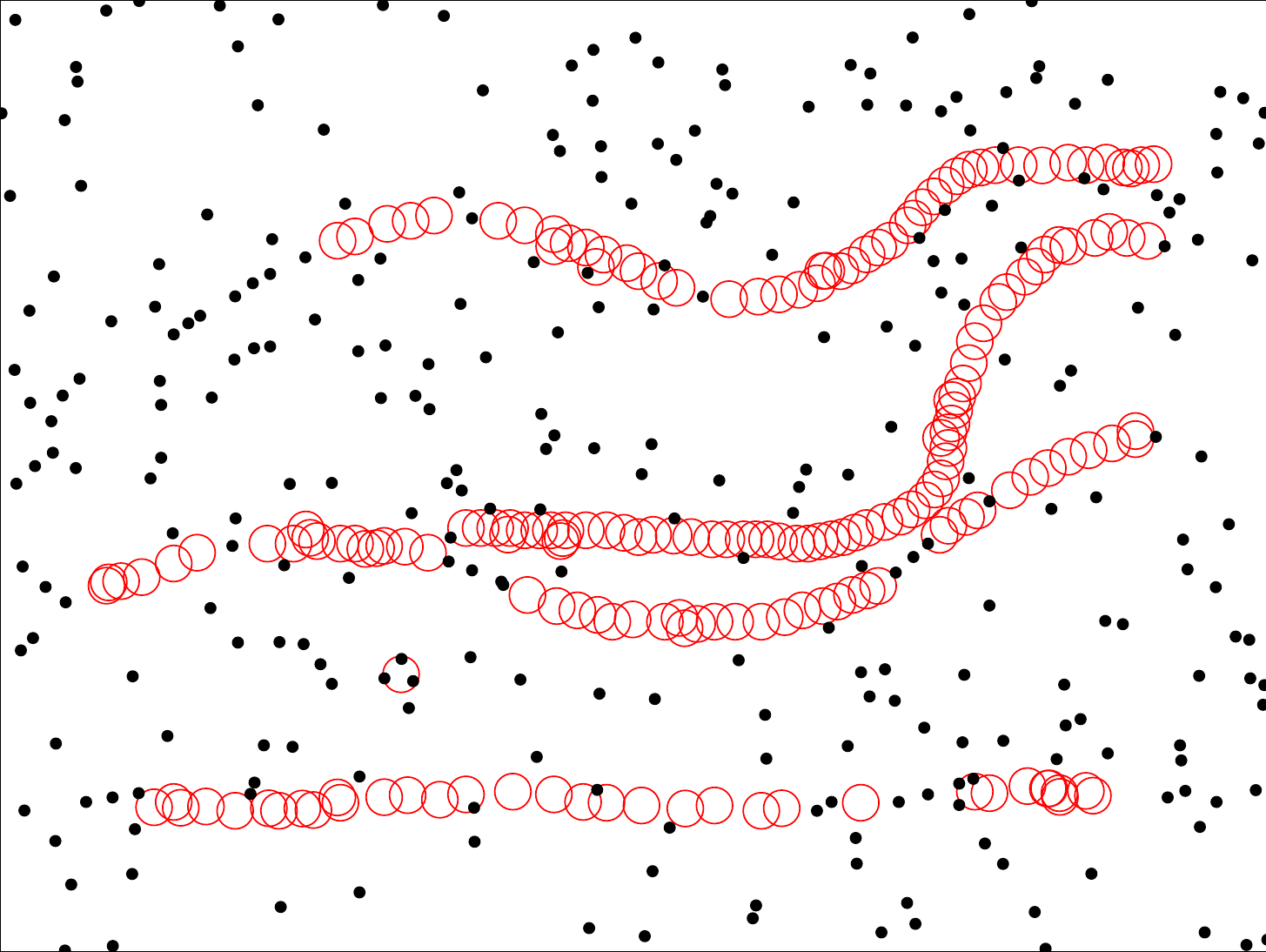}
\caption{Outlier-removal experiments.  Left column: data with outliers.
     The percentages of outliers are 33.3\%, 60\% and 60\%, respectively.
     Middle: outliers (black dots) detected by pairwise spectral clustering (both SC-NJW and SC-LS, but only the better result is shown).
     Right:  outliers (black dots) detected by HOSC.
     The use of the simple kernel (instead of the heat kernel) in HOSC gives very similar results.}
     \label{fig:artificial_data_outliers}
\end{figure}

We observe that HOSC could successfully
remove most of the true outliers, leaving out smooth structures in
the data; in contrast, SC tended to keep
isolated high-density regions, being insensitive to sparse smooth
structures.  A hundred replications of this experiment (i.e., fixing the
clusters and adding randomly generated outliers) show that the True
Positive Rates (i.e., percentages of correctly identified outliers)
for (SC, HOSC) are (58.1\% vs 67.7\%), (75.4\% vs 86.8\%) and (76.8\% vs 88.0\%), respectively.

\subsection{Real Data}
We next compare SC and HOSC using the two-view motion data studied in \cite{AtevKSCC,RAS}. This data set contains 13 motion sequences:
(1) \emph{boxes}, (2) \emph{carsnbus3}, (3) \emph{deliveryvan}, (4) \emph{desk}, (5) \emph{lightbulb}, (6) \emph{manycars}, (7) \emph{man-in-office}, (8) \emph{nrbooks3}, (9) \emph{office}, (10) \emph{parking-lot}, (11) \emph{posters-checkerboard}, (12) \emph{posters-keyboard}, and (13) \emph{toys-on-table}; and each sequence consists of two image frames of a 3-D dynamic scene taken
by a perspective camera (see Figure \ref{fig:twoview_examples} for a few such sequences). Suppose that several feature points have been extracted from the moving objects in the two camera views of the scene. The task is to separate
the trajectories of the feature points according to different motions. This application, which lies in the field of \emph{structure from motion}, is one of the fundamental problems in computer vision.

\begin{figure}[htbp]
\centering
\includegraphics[width=.32\textwidth]{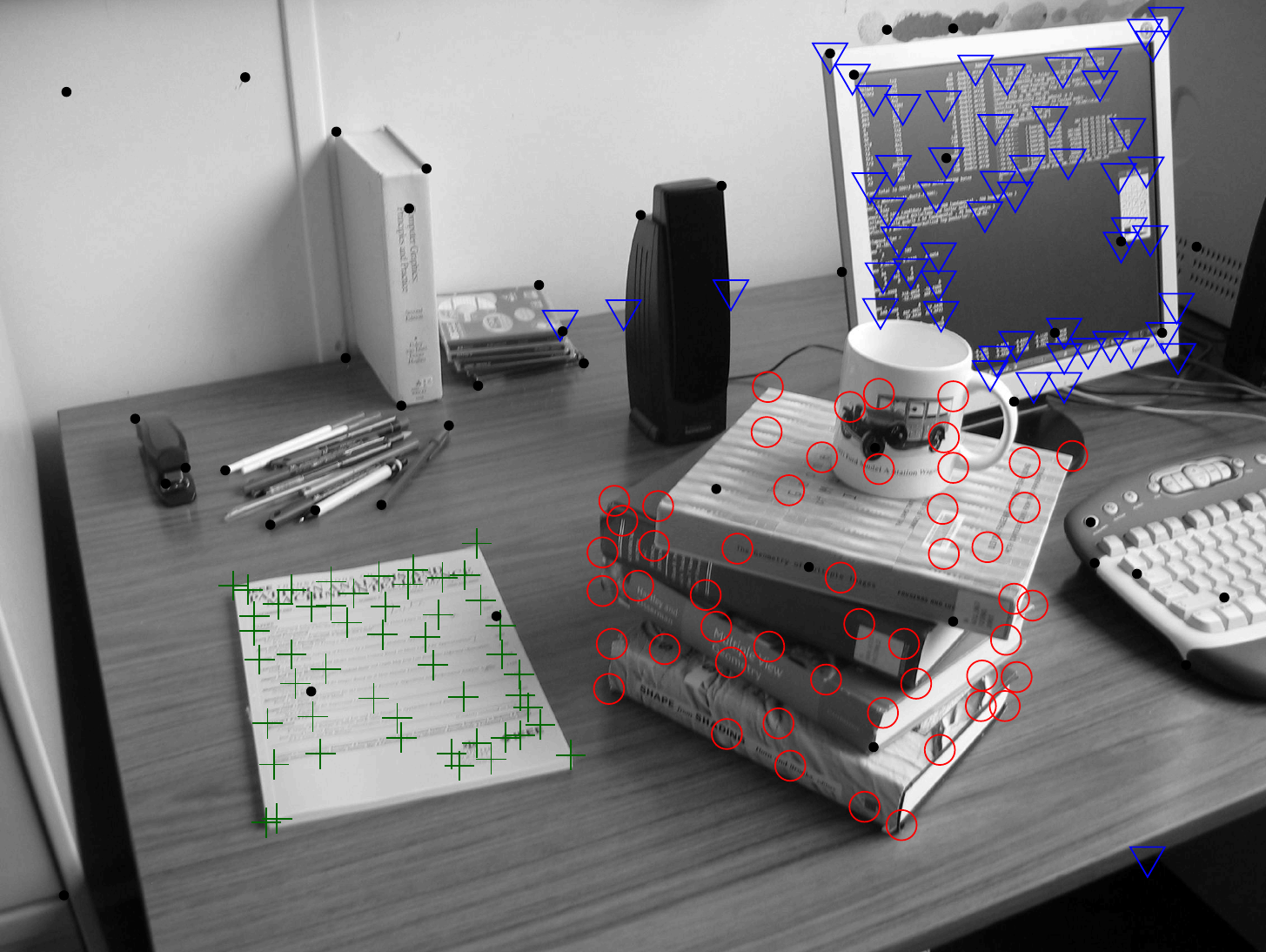}
\includegraphics[width=.32\textwidth]{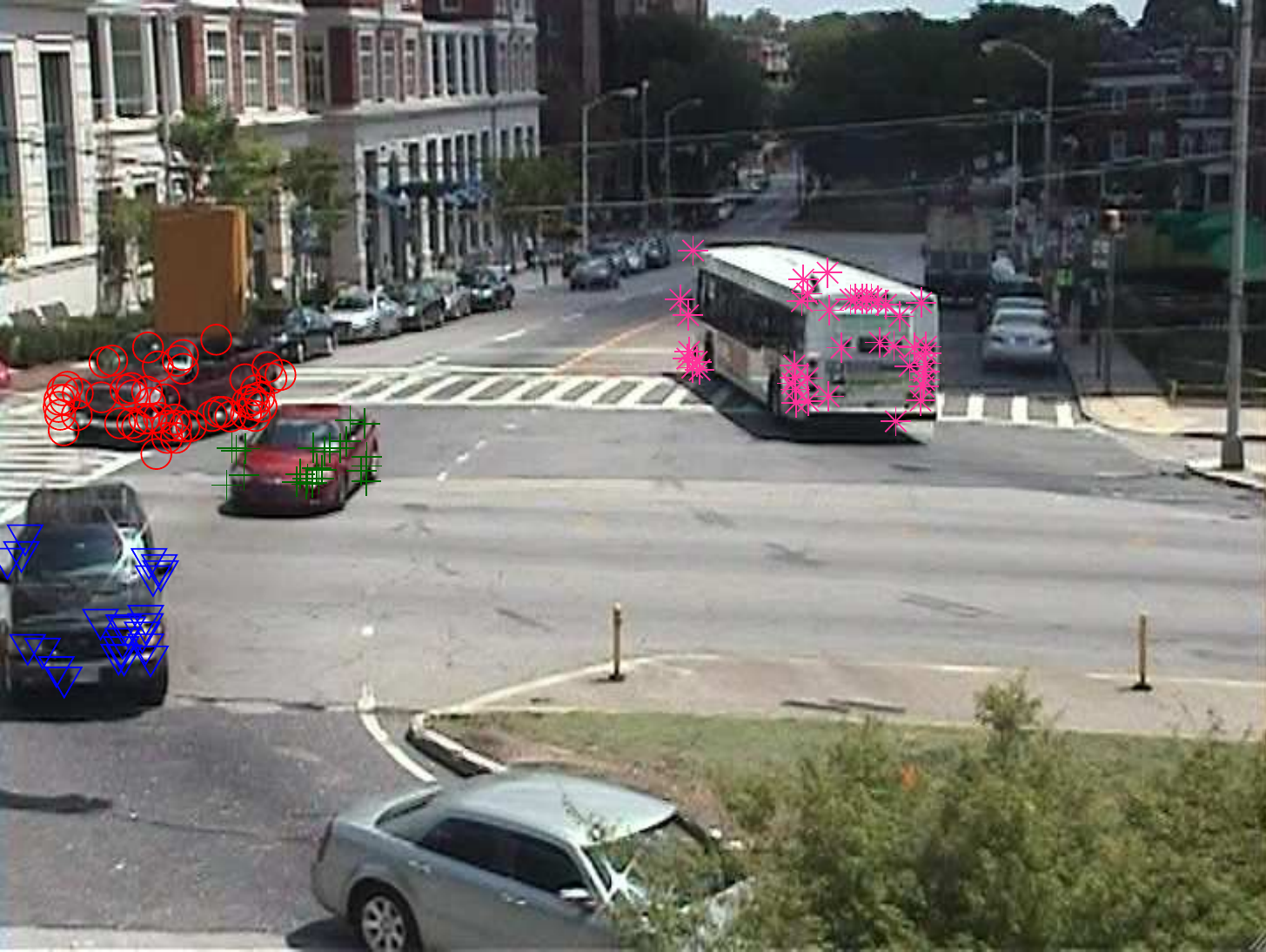}
\includegraphics[width=.32\textwidth]{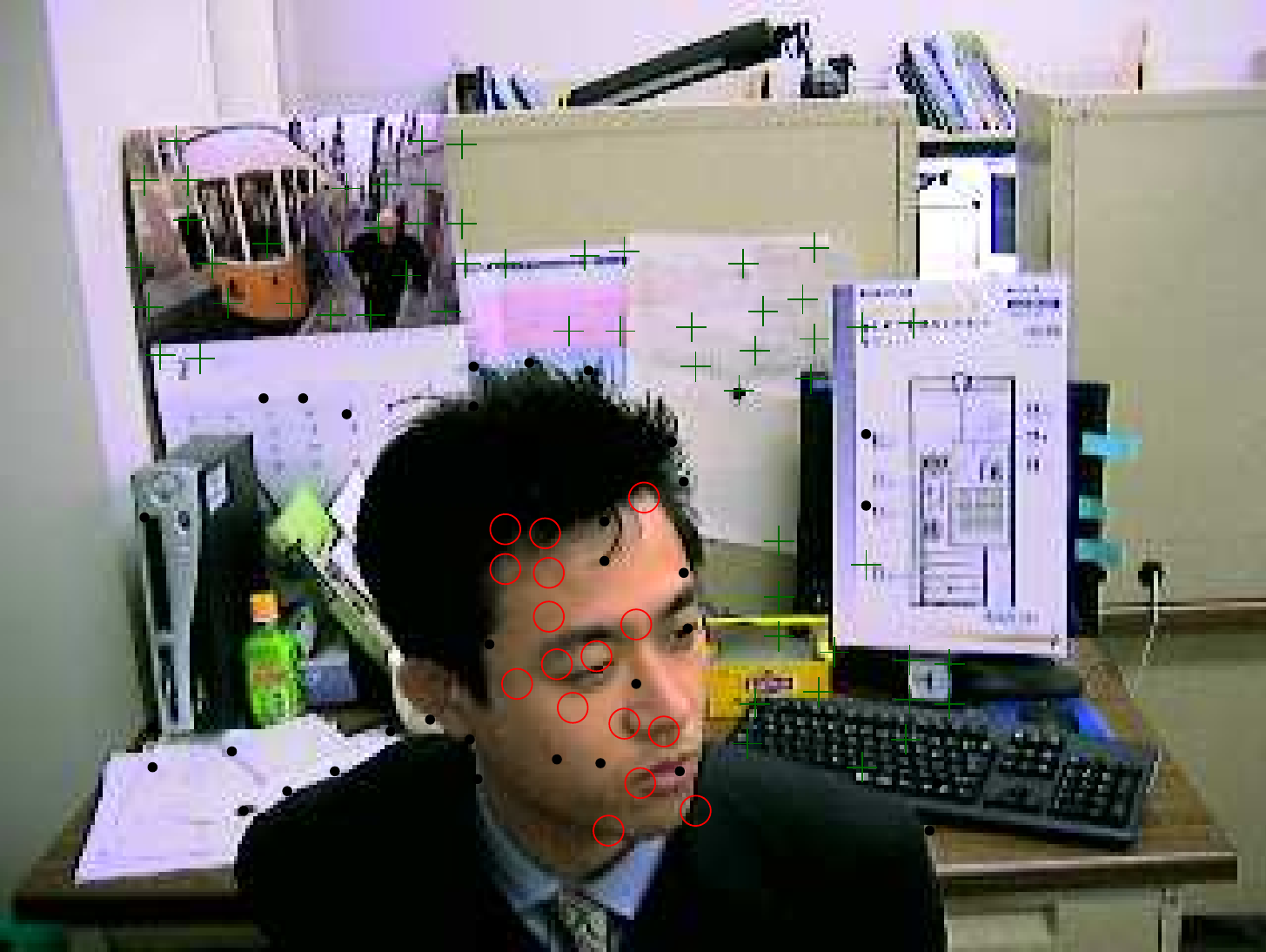}\\
\includegraphics[width=.32\textwidth]{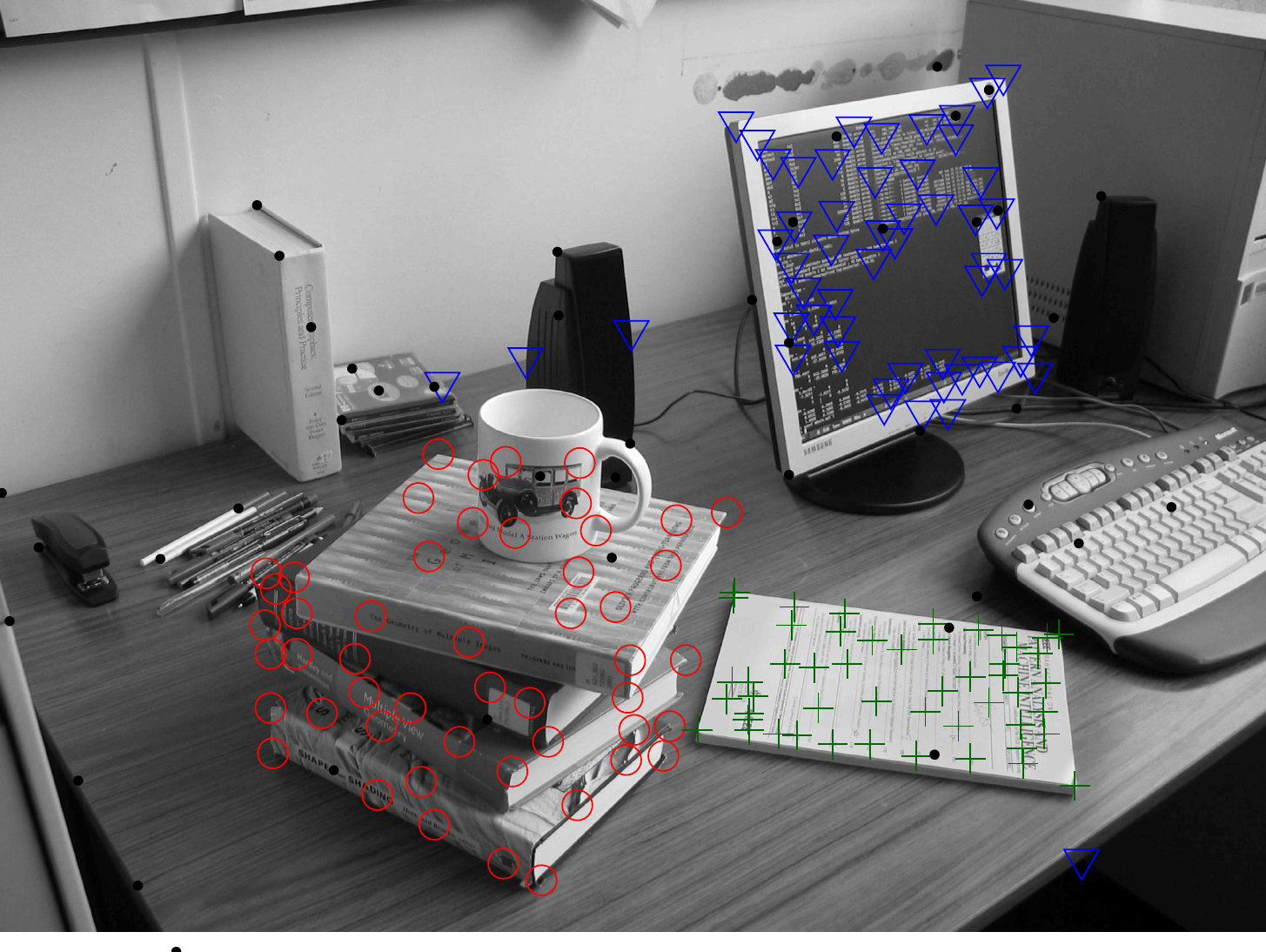}
\includegraphics[width=.32\textwidth]{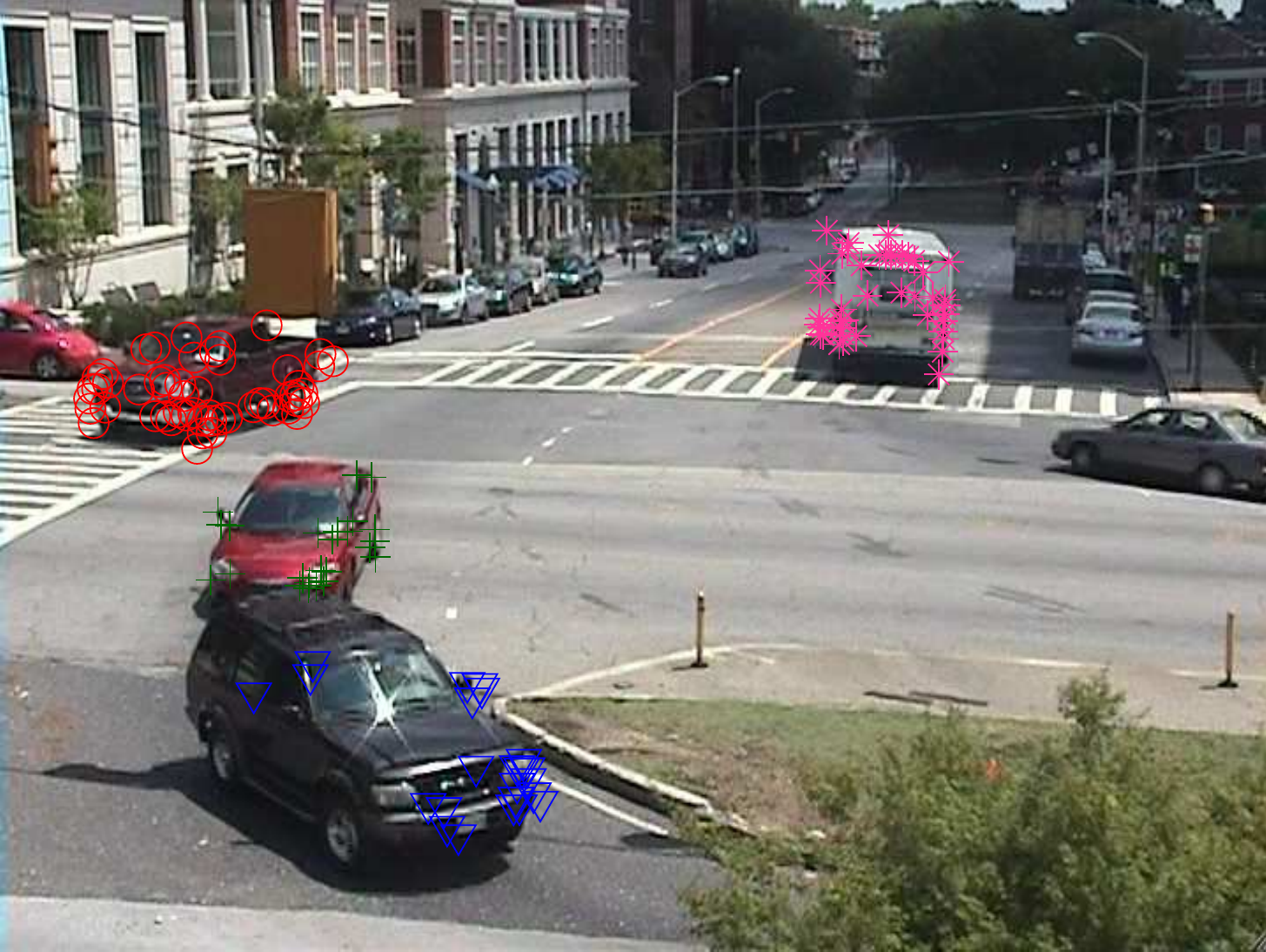}
\includegraphics[width=.32\textwidth]{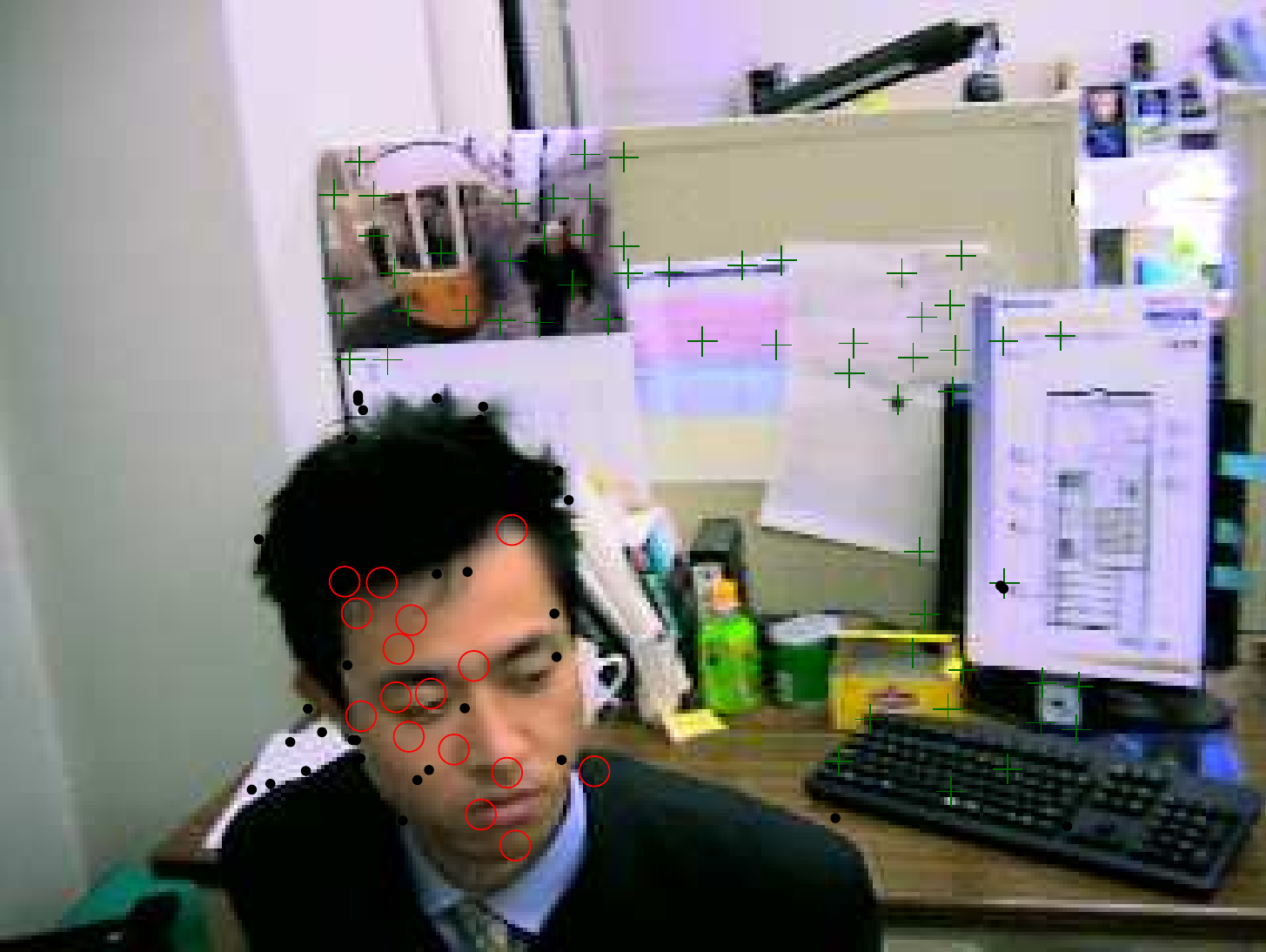} \\
\caption{Three exemplary two-view motion sequences (arranged in columns): (4) \emph{desk}, (6) \emph{manycars} and (7) \emph{man-in-office}. The true clusters are displayed in different colors and markers (the black dots are outliers).}
\label{fig:twoview_examples}
\end{figure}

Given a physical point $\bx \in\mathbb{R}^3$ and its image correspondences in the two views $(x_1,y_1)', (x_2,y_2)' \in \mathbb{R}^2$, one can always form a joint image sample $\mathbf{y}=(x_1,y_1,x_2,y_2,1)'\in \mathbb{R}^5$. It is shown in \cite{RAS} that, under perspective camera projection,
all the joint image samples $\mathbf{y}$ corresponding to different motions live on different manifolds in $\mathbb{R}^5$, some having dimension 2 and others having dimension 4. Exploratory analysis applied to these data suggests that the manifolds in this dataset mostly have dimension 2 (see Figure \ref{fig:twoview_trueclusters}). Therefore, we will apply our algorithm (HOSC) with $d=2$ to these data sets in order to compare with pairwise spectral clustering (SC-NJW, SC-LS).

\begin{figure}[htbp]
\centering
\includegraphics[width=.32\textwidth]{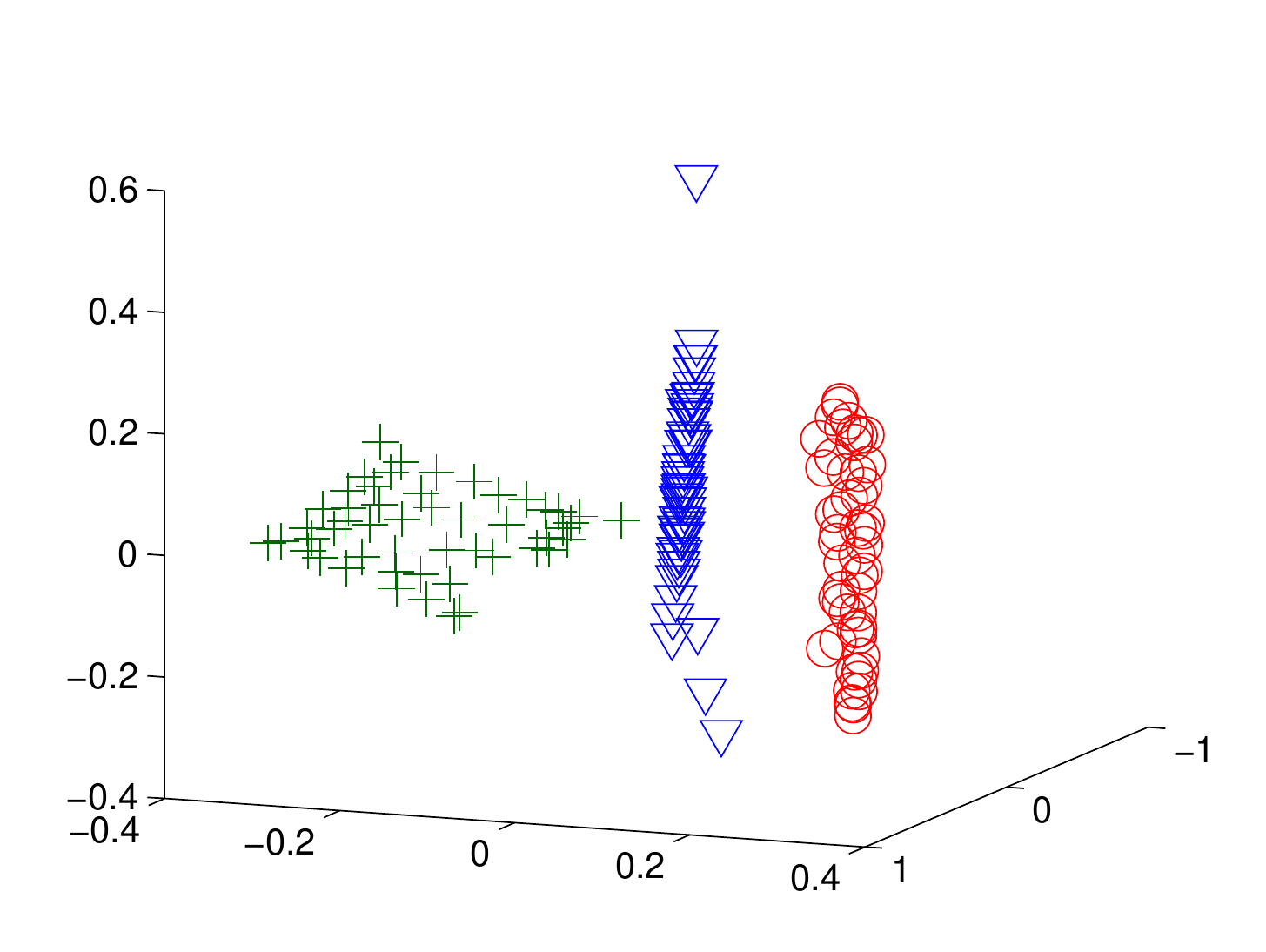}
\includegraphics[width=.32\textwidth]{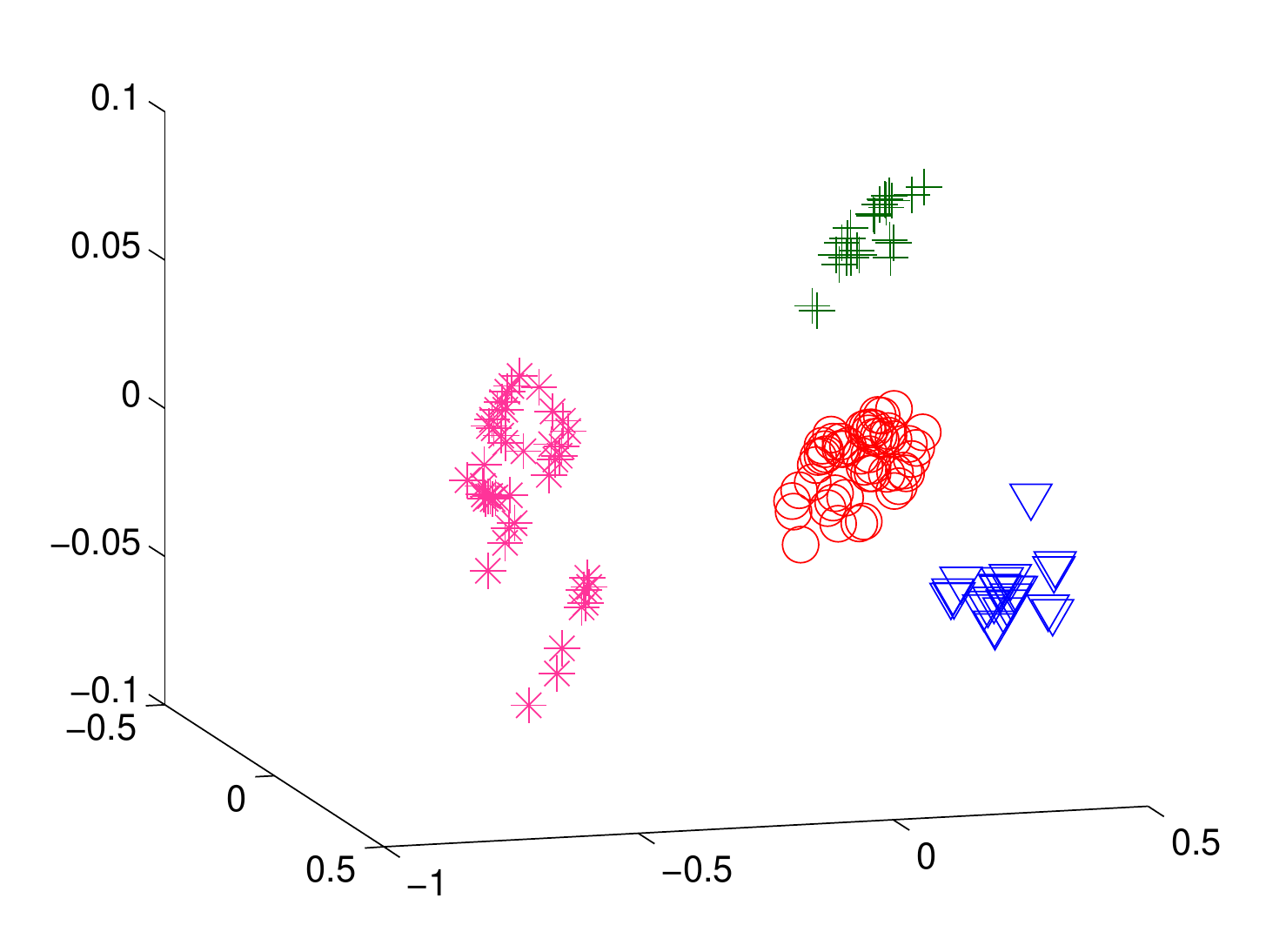}
\includegraphics[width=.32\textwidth]{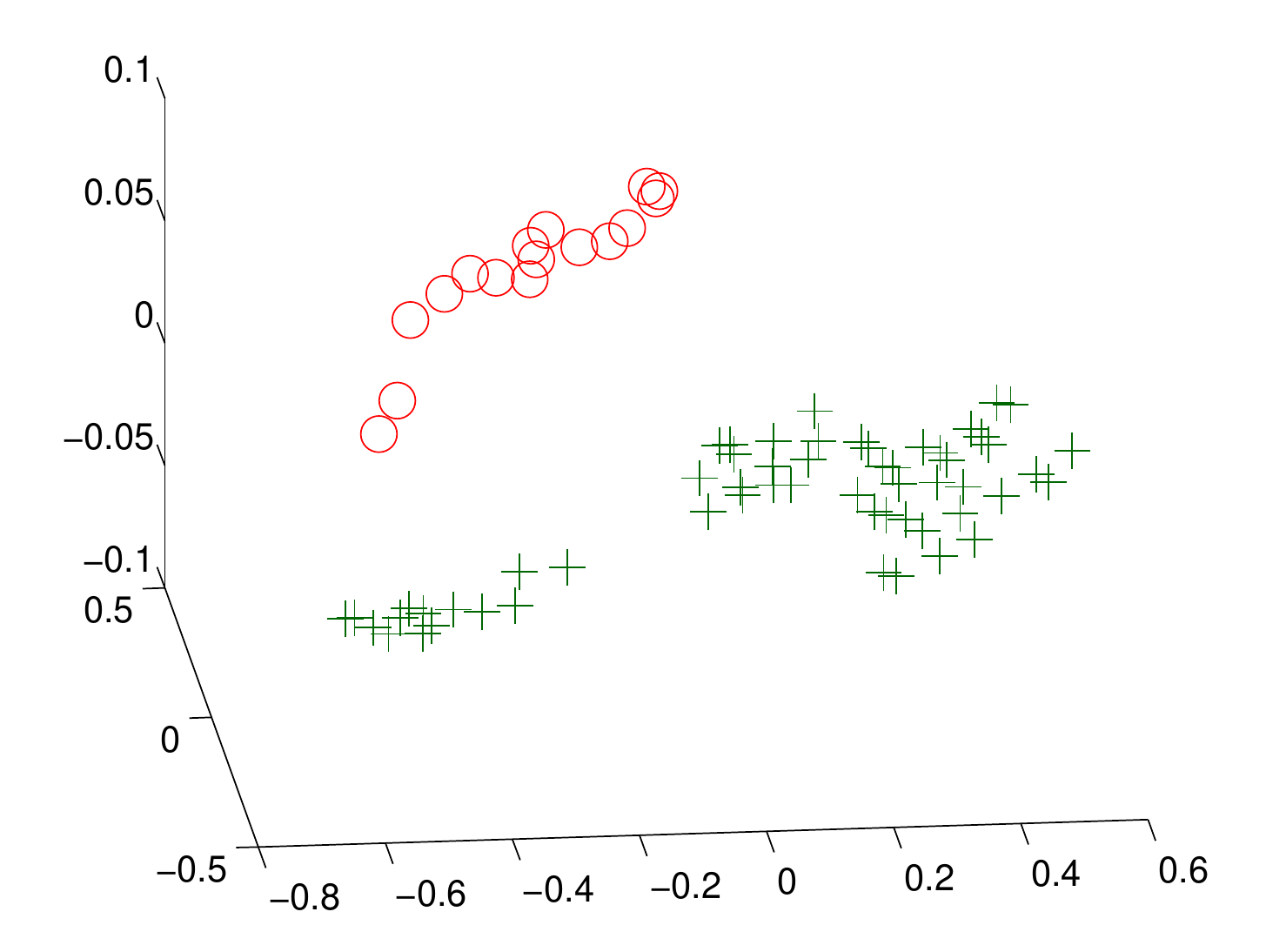}\\
\caption{The true clusters of the three sequences in Figure \ref{fig:twoview_examples} (in same order), shown in top three principal dimensions. (The outliers have been removed from the data and thus are not displayed). These plots clearly indicate that the underlying manifolds are two dimensional.}
\label{fig:twoview_trueclusters}
\end{figure}

We use the following parameter values for the two algorithms. In HOSC, we choose $\ell=20, m=d+2, \eta\in [.0001,.1]$, while in SC we try both searching the interval $[.001, .5]$ (SC-NJW) and local scaling with at most 24 nearest neighbors (SC-LS).

The original data contains some outliers. In fact, 10 sequences out of the 13 are corrupted with outliers, with the largest percentage being about 32\%. We first manually remove the outliers from those sequences and solely focus on the clustering aspects of the two algorithms. Next, we add outliers back and compare them regarding outliers removal. (Note that we need to provide both algorithms with the true percentage of outliers in each sequence.) By doing so we hope to evaluate the clustering and outliers removal aspects of an algorithm separately and thus in the most accurate way.

\begin{table}[htbp]
\caption{\small The misclassification rates and the numbers of true outliers detected by HOSC, SC-NJW and SC-LS. In the clustering experiment, the outliers-free data is used; then the outliers are added back so that each of these algorithms can be applied to detect them. For SC-NJW, the tuning parameter is selected from the interval $[.001, .5]$; for SC-LS, a maximum of $24$ nearest neighbors are used; for HOSC, $20$ nearest neighbors are used and the flatness parameter $\eta$ is selected from the interval $[.0001, .1]$.}
\vspace{.1in}
\begin{tabular}{|l|l|l||r|r|r||l|l|l|}
  \hline
    \multicolumn{3}{|c||}{Data}
    &\multicolumn{3}{c||}{Clustering Errors}
    &\multicolumn{3}{c|}{ \# True Outliers Detected}\\
  \hline
  seq. & \#samples & \#out. & SC-NJW  & SC-LS &  HOSC & SC-NJW & SC-LS & HOSC\\
  \hline\hline
  1    & 115,121       & 2               & 0.85\% & 0.85\% & 0.85\%   & 1     & 1      & 1\\
  2    & 85,45,89      & 28             & 0\%      & 0\%      & 0\%   & 24   & 24    &24\\
  3    & 62,192         & 0               & 30.3\% & 23.6\% & 30.3\%   & N/A  & N/A & N/A\\
  4    & 50,50,55      & 45             & 0.65\% & 2.58\% & 1.29\%      & 35   & 30    &37\\
  5    & 51,121,33    & 0               & 0\%      & 0\%       &  0\%       &N/A  &  N/A & N/A\\
  6    & 54,24,23,43 & 0              & 18.8\% & 0\%       &  0\%       &N/A  &N/A   & N/A \\
  7    & 16,57           & 34             & 19.2\% & 19.2\%  &  0\%      & 17   & 12    &26\\
  8    & 129,168,91  & 32             & 22.9\% & 17.8\%  & 22.9\%        & 12   & 17    &23\\
  9    & 76,109,74    & 48             & 0\%      & 0\%       & 0\%        & 36  & 28    &36\\
  10  & 19,117         & 4               & 0\%      & 47.8\%   & 0\%       & 0    & 0      & 1\\
  11  & 100,99,81    & 99             &  0\%     & 1.79\%  & 0\%       & 42   & 39   &73 \\
  12  & 99,99,99      & 99             & 0.34\% & 0.34\%  & 0\%       & 80  & 43    & 91 \\
  13  & 49,42           & 35             &33.0\%  & 15.4\%  & 2.20\%  & 7    & 6      & 21\\
  \hline
\end{tabular}
\label{tab:comparison_2view}
\end{table}

Table~\ref{tab:comparison_2view} presents the results from the experiments above. Observe that HOSC achieved excellent clustering results in all but two sequences, with zero error on eight sequences, one mistake on sequence (13), and two mistakes on each of (1) and (4). We remark that HOSC also outperformed the algorithms in \cite[Table 1]{AtevKSCC}, in terms of clustering accuracy, but due to the main aim of this paper, we do not include those results in Table~\ref{tab:comparison_2view}.  In contrast, each of SC-NJW and SC-LS failed on at least five sequences (with over 15\% misclassification rates), both containing the two bad sequences for HOSC. As a specific example, we display in Figure \ref{fig:maninoffice_hosc_sc} the clusters obtained by both HOSC and SC on sequence (7), demonstrating again that higher order affinities can significantly improve over pairwise affinities in the case of manifold data.
Regarding outliers removal, HOSC is also consistently better than SC-NJW and SC-LS (if not equally good).

\begin{figure}[htbp]
\centering
\includegraphics[width=.45\textwidth]{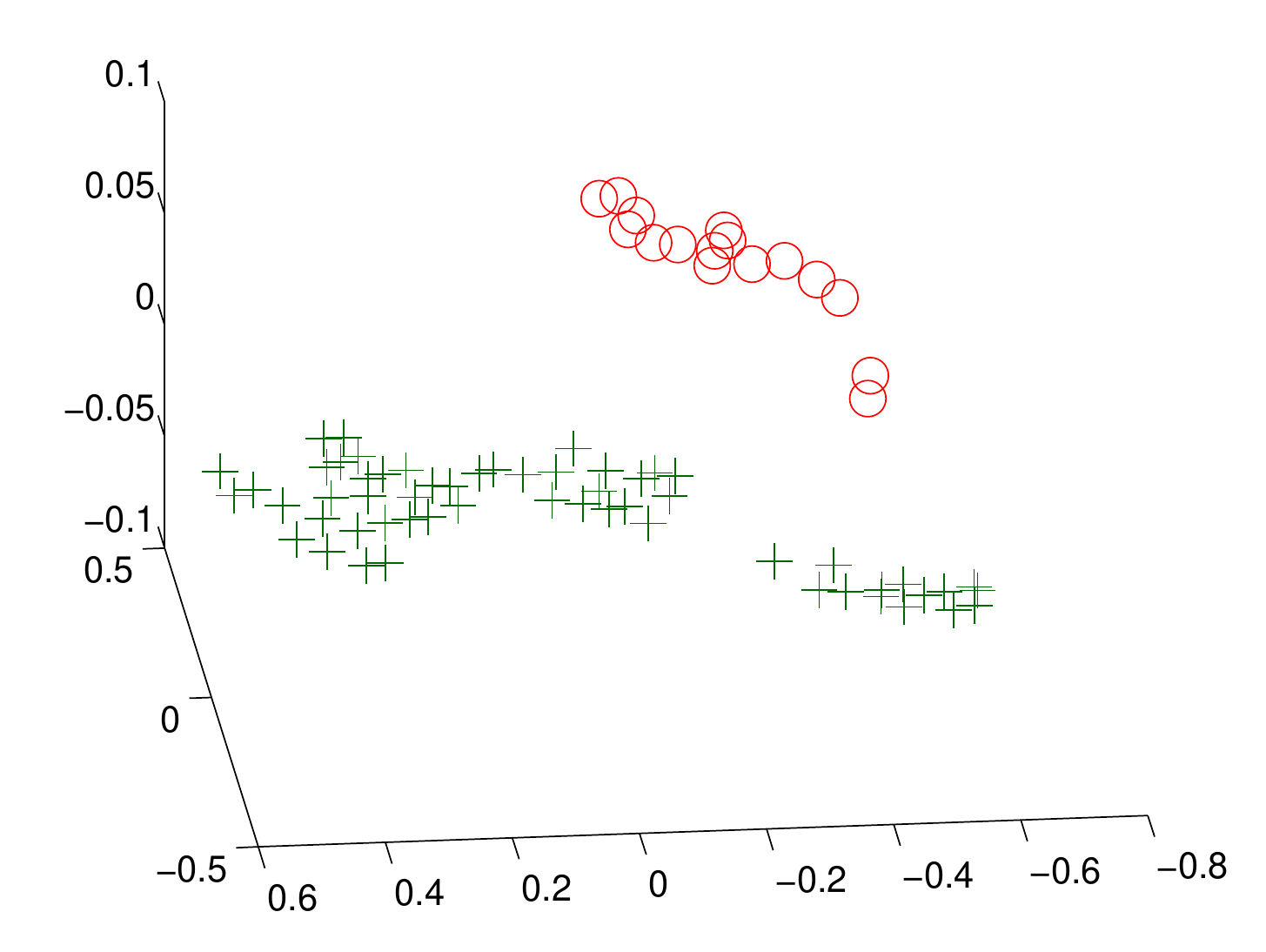}
\includegraphics[width=.45\textwidth]{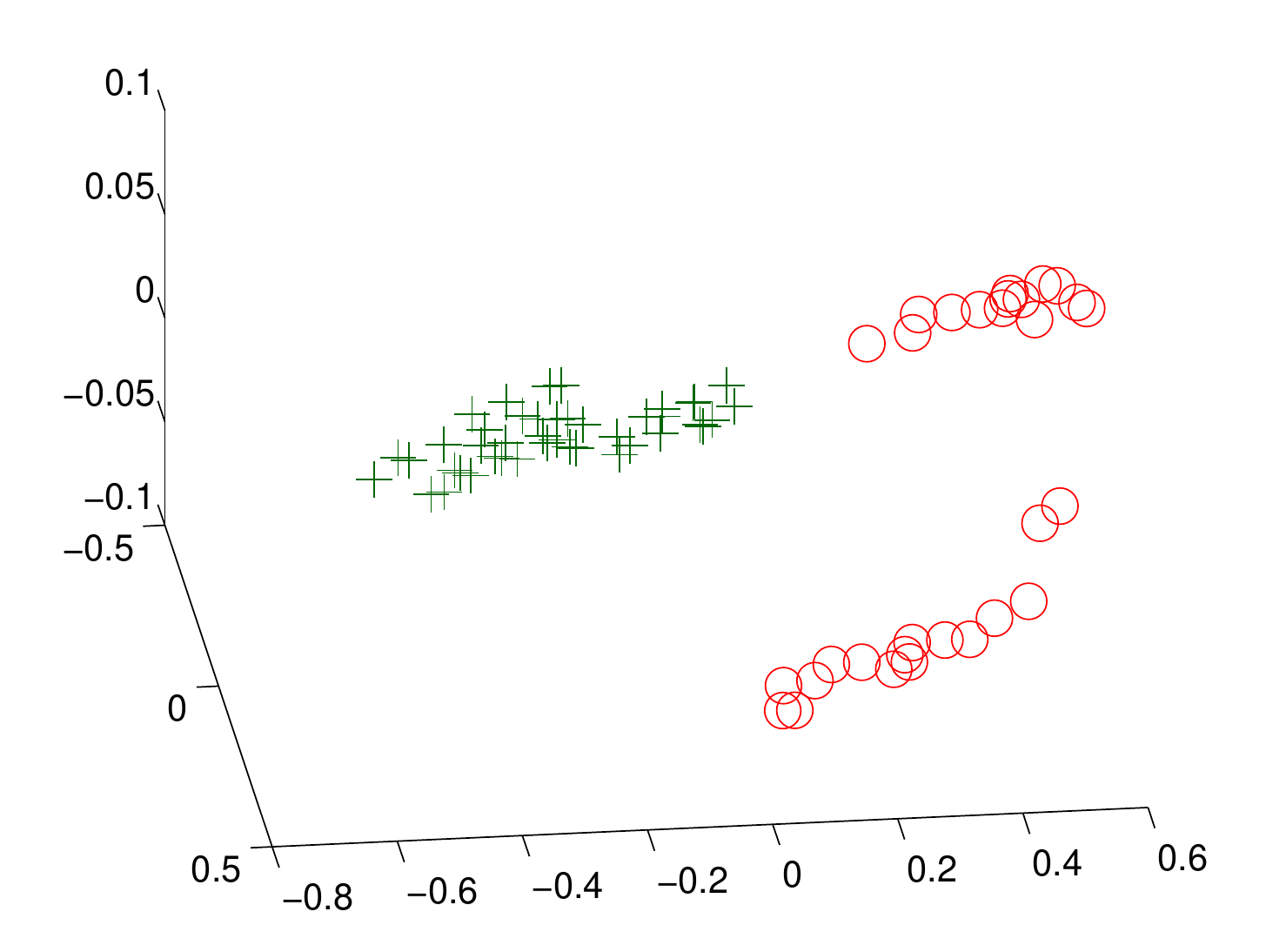}
\caption{Clustering results of both HOSC and SC (left to right) on sequence (7). (The truth is shown in Figure~\ref{fig:twoview_trueclusters}, rightmost plot). In this example, HOSC correctly found the two clusters, using geometric information; in contrast, SC failed because it solely relies on pairwise distances.}
\label{fig:maninoffice_hosc_sc}
\end{figure}

\section{Extensions}
\label{sec:discussion}

\subsection{When the Underlying Surfaces Self-Intersect}
\label{sec:self}

In our generative model described in Section~\ref{sec:setting} we assume that the surfaces are submanifolds, implying that they do not self-intersect.  This is really for convenience as there is essentially no additional difficulty arising from self-intersections.  If we allow the surfaces to self-intersect, then we bound the maximum curvature (from above) and not the reach.
We could, for example, consider surfaces of the form $S = f(B_d(0,1))$, where $f:B_d(0,1) \to (0,1)^D$ is locally bi-Lipschitz and has bounded second derivative.  A similar model is considered in~\cite{MR1332579} in the context of set estimation.  Clearly, proving that each cluster is connected in the neighborhood graph in this case is the same.  The only issue is in situations where a surface comes within distance $\eps$ from another surface at a location where the latter intersects itself.  The geometry involved in such a situation is indeed complex.  If we postulate that no such situation arises, then our results generalize immediately to this setting.

\subsection{When the Underlying Surfaces Have Boundaries}
\label{sec:boundary}


When the surfaces have boundaries, points near the boundary of a surface may be substantially connected with points on a nearby surface.  See Figure~\ref{fig:right-angle} for an illustration.  This is symptomatic of the fact that the algorithm is not able to resolve intersections in general, as discussed in Section~\ref{sec:intersect}, with the notable exception of clusters of dimension $d=1$, as illustrated in the `two moons' example of Figure~\ref{fig:artificial_data}.

\begin{figure}[htbp]
\centering
\includegraphics[width=.45\linewidth]{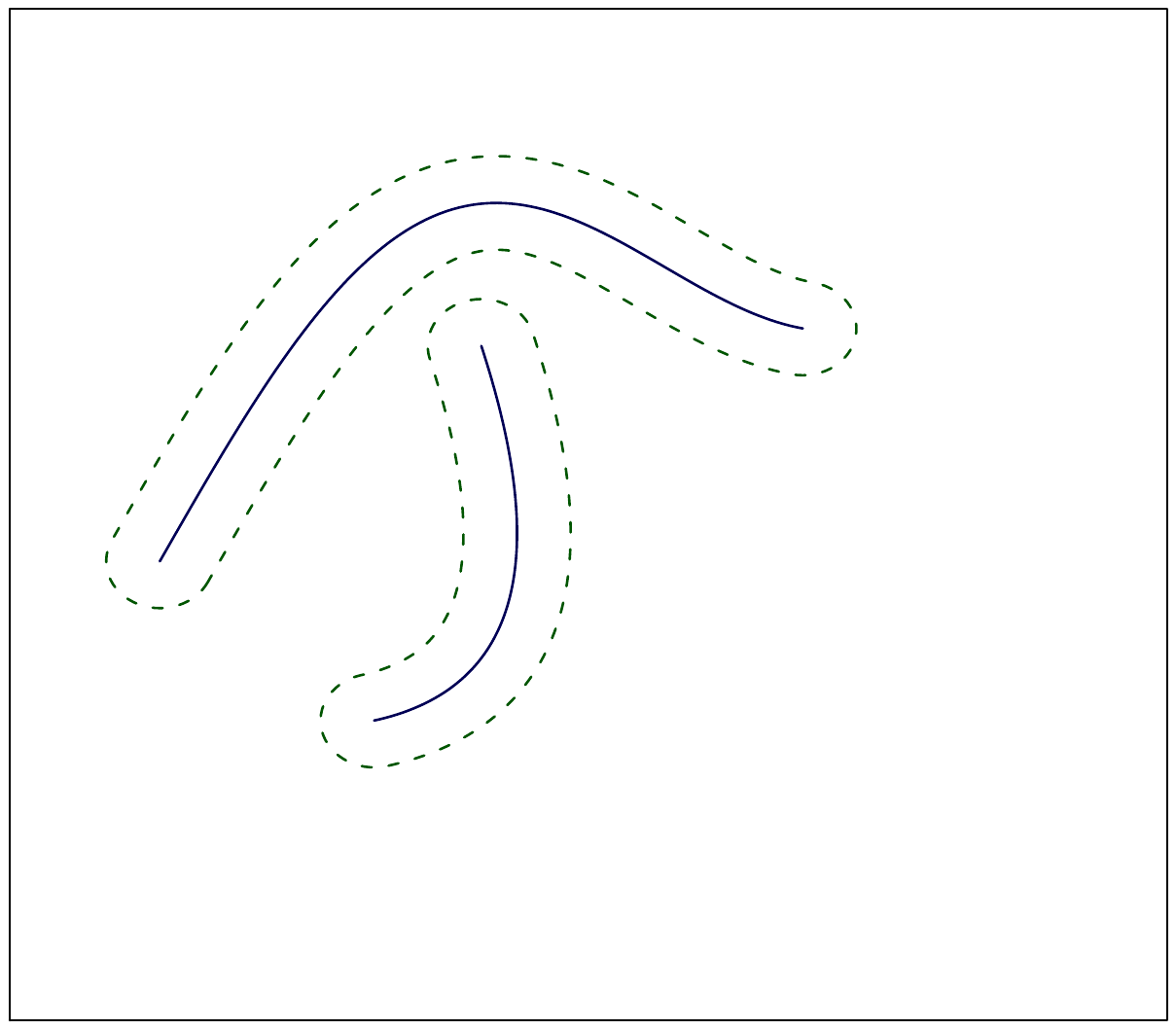} \
\includegraphics[width=.45\linewidth]{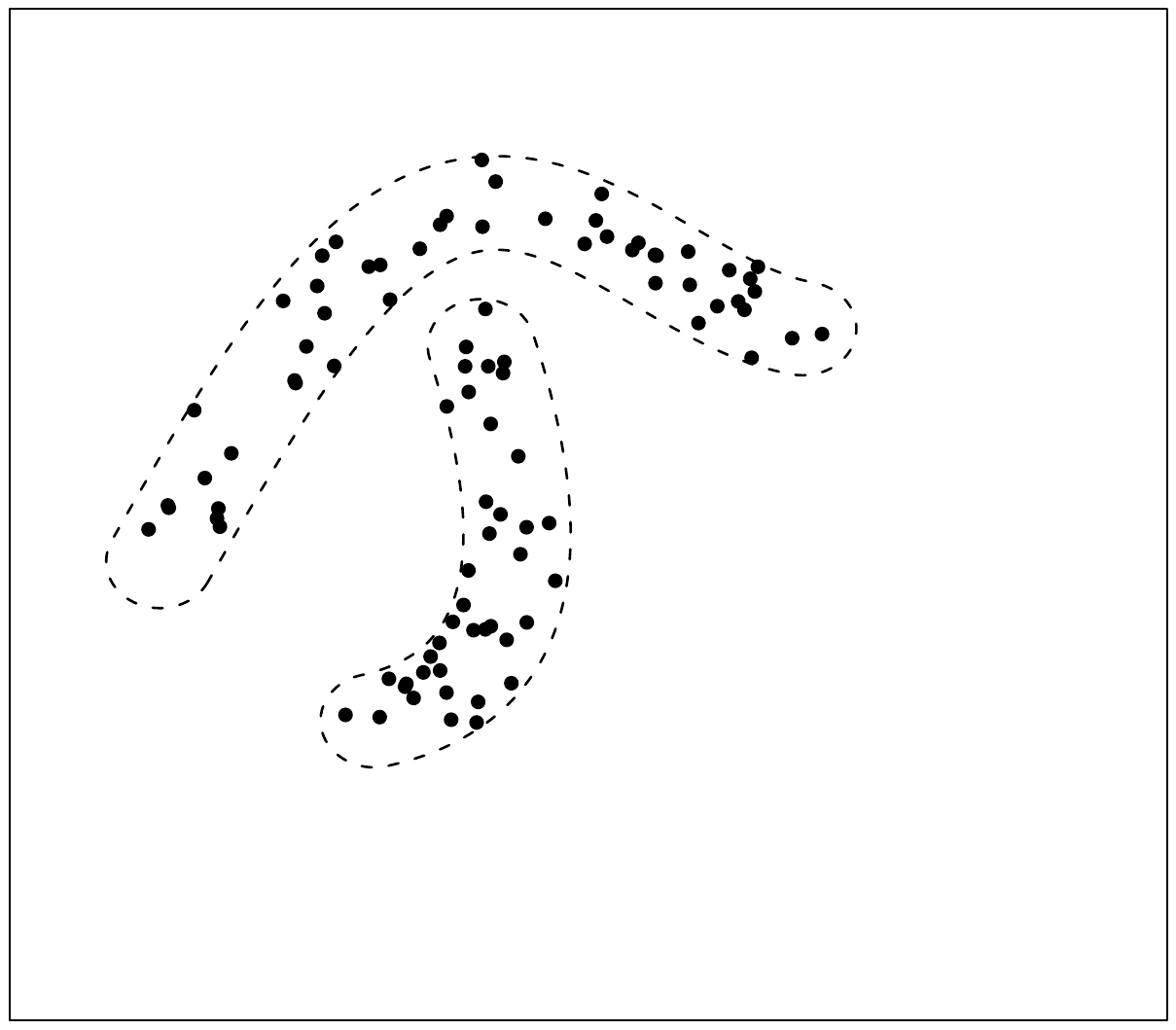} \
\vspace{-.3in}
\caption{An example of a surface with a boundary coming close to another surface.  This is a potentially  problematic situation for HOSC as the points near the boundary of one surface and close to the other surface may be strongly connected to points from both clusters.  Numerically, we show in Figure~\ref{fig:artificial_data} such an example where HOSC is successful.}
\label{fig:right-angle}
\end{figure}

If we require a stronger separation between the boundary of a surface and the other surfaces, specifically,
\begin{equation} \label{eq:delta-partial}
\dist(\partial S_k, S_\ell) \geq \delta_\ddag, \quad \forall k \neq \ell,
\end{equation}
with $\delta_\ddag - 2 \tau > \eps$, no point near the boundary of a cluster is close to a point from a different cluster.  (A corresponding requirement in the context of outliers would be that outliers be separated from the boundary of a cluster by at least $\delta_{0, \ddag}$, with $\delta_{0, \ddag} - \tau > \eps$.)

\subsection{When the Data is of Mixed Dimensions}
\label{sec:mixed}

In a number of situations, the surfaces may be of different intrinsic dimensions.  An important instance of that is the study of the distribution of galaxies in space, where the galaxies are seen to cluster along  filaments ($d=1$) and walls ($d=2$)~\cite{MarSaa}.  We propose a top-down approach, implementing HOSC for each dimension $d$ starting at $D-1$ and ending at $1$ (or between any known upper and lower bounds for $d$).

At each step, the algorithm is run on each cluster obtained from the previous step, including the set of points identified as outliers.  Indeed, when the dimension parameter of the algorithm is set larger than the dimension of the underlying surfaces, HOSC may not be able to properly separate clusters.  For example, two parallel segments satisfying the separation requirement of Theorem~\ref{th:linear} still belong to a same plane and HOSC with dimension parameter $d=2$ would not be able to separate the two line segments.  Another reason for processing the outlier bin is the greater disparity in the degrees of the data points in the neighborhood graph often observed with clusters of different dimensions.  At each step, the number of clusters is determined automatically according to the procedure described in Section~\ref{sec:param}, for such information is usually not available.  The parameters $\eps $ and $\eta$ are chosen according to \eqref{eq:eps-eta-O2}.
Partial results suggest that, under some additional sampling conditions, this top-down procedure is accurate under weaker separation requirements than required by pairwise methods, which handle the case of mixed dimensions seamlessly~\cite{pairwise}.
The key is that an actual cluster $\cX_k$, as defined in Section~\ref{sec:setting}, is never cut into pieces.
Indeed, properties \ref{a1} and \ref{a4} in the proof of Theorem~\ref{th:linear}, which guarantee the connectivity and regularity (in terms of comparable degrees) of the subgraph represented by $\cX_k$, are easily seen to also be valid when the dimension parameter of the algorithm is set larger than $d$.
(This observation might explain the success of the SCC algorithm of~\cite{spectral_applied} in some mixed settings when using an upper bound on the intrinsic dimensions.)

\subsection{Clustering Based on Local Polynomial Approximations}
\label{sec:poly}

For $1 \leq d \leq D-1$ and an integer $r \geq 3$, let $\cS_d^{r}(\kappa)$ be the subclass of $\cS_d^2(\kappa)$ of $d$-dimensional submanifolds $S$ such that, for every $\bx \in S$ with tangent $T_\bx$, the orthogonal projection $S \cap B(\bx, 1/\kappa) \to T_\bx$ is a $C^{r}$-diffeomorphism with all partial derivatives of order up to $r$ bounded in supnorm by $\kappa$.
For example, $\cS_d^r(\kappa)$ includes a subclass of surfaces of the form $S = f(B_d(0,1))$, where $f:B_d(0,1) \to (0,1)^D$ is locally bi-Lipschitz and has its first $r$ derivatives bounded.  (We could also consider surfaces of intermediate, i.e., H\"older smoothness, a popular model in function and set estimation~\cite{MR0358168,MR1332579}.)

Given that surfaces in $\cS_d^{r}$ are well-approximated locally by polynomial surfaces, it is natural to choose an affinity based on the residual of the best $d$-dimensional polynomial approximation of degree at most $r-1$ to a set of points $\bx_1, \dots, \bx_{m}$.  This may be implemented via the ``kernel trick" with a polynomial kernel, as done in~\cite{AtevKSCC} for the special case of algebraic surfaces.
The main difference with the case of $C^2$ surfaces that we consider in the rest of the paper is the degree of approximation to a surface $S \in \cS_d^r$ by its osculating algebraic surface of order $r-1$; within a ball of radius $\eps$, it is of order $O(\eps^r)$.

Partial results suggest that, under similar conditions, the kernel version of HOSC with $r$ known may be able to operate under a separation of the form (\ref{eq:sep}), with the exponent $2/d$ replaced by $r/d$ and, in the presence of outliers, within a logarithmic factor of the best known sampling rate ratio achieved by any detection method~\cite{AriasCastro2009,CTD}:
\begin{equation}
\min_k N_k \geq N^{d/(r D - (r-1) d)} \vee N \tau^{D-d}.
\end{equation}
Regarding the estimation of $\tau$, defining the correlation dimension using the underlying affinity defined here allows to estimate $\tau$ accurately down to (essentially) $(\log(N)/N)^{r/d}$, if the surfaces are all in $\cS_d^{r}(\kappa)$.  The arguments are parallel and we omit the details.

Thus, using the underlying affinity defined here may allow for higher accuracy, if the surfaces are smooth enough.  However, this comes with a larger computational burden and at the expense of introducing a new parameter $\ur$, which would need to be estimated if unknown, and we do not know a good way to do that.

\subsection{Other Extensions}
\label{sec:extensions}

The setting we considered in this paper, introduced in Section~\ref{sec:setting}, was deliberately more constrained than needed for clarity of exposition.  We list a few generalizations below, all straightforward extensions of our work.

\begin{itemize}
\item {\it Sampling.}  Instead of the uniform distribution, we could use any other distribution with a density bounded away from $0$ and $\infty$, or with fast decaying tails such as the normal distribution.

\item {\it Kernel.}  The rate of decay of the kernel $\phi$ dictates the range of the affinity (\ref{eq:linear-affinity}).  Let $\omega_\sN$ be a non-decreasing sequence such that $N^{3m} \phi(\omega_\sN) \to 0.$
For a compactly supported kernel, $\omega_\sN = \sup\{s: \phi(s) > 0\}$, while for the heat kernel, we can take $\omega_\sN = 2 \sqrt{m \log N}$.  As we will take $m \to \infty$, $\phi$ is essentially supported in $[0, \omega_\sN]$ so that points that are further than $\omega_\sN \eps$ apart have basically zero affinity.  Specifically, we use the following bounds:
$$
\phi(1) {\bf 1}_{\{|s| < 1\}} \leq \phi(s) \leq {\bf 1}_{\{|s| < \omega_\sN\}} + \phi(\omega_\sN).
$$
The results are identical, except that statements of the form $\delta - 2 \tau > Z$ are replaced with $\delta - 2 \tau > \omega_\sN Z$.

\item {\it Measure of flatness.}  As pointed out in the introduction, any reasonable measure of linear approximation could be used instead.  Our choice was driven by convenience and simplicity.

\end{itemize}

\appendix
\section{Preliminaries}

We gather here some preliminary results.  Recall that, for $a, b \in \bbR$, $a \vee b := \max(a, b)$; $a \wedge b := \min(a, b)$; $a_+ = a \vee 0$.  For $(a_\sN), (b_\sN) \in \bbR^{\mathbb{N}}$, $a_\sN \prec b_\sN$ means $a_\sN = O(b_\sN)$; $a_\sN \asymp b_\sN$ means both $a_\sN = O(b_\sN)$ and $b_\sN = O(a_\sN)$; $a_\sN \ll b_\sN$ means $a_\sN = o(b_\sN)$.  For $L \in \cA_d$, $P_L$ denotes the orthogonal projection onto $L$.  The canonical basis of $\bbR^D$ is denoted $\be_1, \dots, \be_D$.

\subsection{Large Deviations Bounds}

The following result is a simple consequence of Hoeffding's or Bernstein's inequalities.
\begin{lemma}[\cite{MR2133757}, Lem.~5.3.7]  \label{lem:hoeff}
Let $(X_i)_{i \geq 1}$ be independent random variables in $[0,1]$.

If $4 a \leq \sum_i \expect{X_i}$,
$$\pr{\sum_i X_i \leq a} \leq \exp(- a).$$

If $a \geq 8 \sum_i \expect{X_i}$,
$$\pr{\sum_i X_i \geq a} \leq \exp(- a).$$
\end{lemma}

\subsection{Some Geometrical Results}

We start by quantifying how well a surface $S \in \cS^2(\kappa)$ is locally approximated by its tangent.  Recall that, for an affine subspace $T$,
$P_T$ denotes the orthogonal projection onto $T$. For any $\bs \in S$, let $T_\bs$ denote the tangent of $S$ at $\bs$. %
\begin{lemma} \label{lem:S-approx}
For any $S \in \cS_d^2(\kappa)$ and $\bs \in S$, the orthogonal projection onto $T_\bs$ is injective on $B(\bs, 1/(4\kappa)) \cap S$ and
$P_{T_\bs}^{-1}$ has Lipschitz constant bounded by $\sqrt{2}$ on its image, which contains $B(\bs, 1/(8 \kappa)) \cap T_\bs$.  Moreover,
$$B(\bs, \eps) \cap S \subset B(T_\bs, \kappa \eps^2), \ \forall \eps,$$
and
$$B(\bs, \eps) \cap T_\bs \subset B(S, 2 \kappa \eps^2), \ \forall \eps < 1/(8 \kappa).$$
\end{lemma}
\begin{proof}
This sort of result is standard in differential geometry.  We follow the exposition in~\cite{1349695}. We note that the manifold parameter $\tau$
in~\cite{1349695}, i.e., the inverse of the condition number, coincides with the manifold's reach. We thus fix here an $S \in \cS_d^2(\kappa)$ and
denote $\tau:={\rm reach}(S)$. Since $1/\kappa$ is a lower bound on the reach for manifolds in $\cS_d^2(\kappa)$, we have the inequality $\tau \geq 1/\kappa$.

Fix also a point $\bs \in S$. Applying~\cite[Lem.~5.4]{1349695}, we obtain that $P_{T_\bs}$ is one-to-one on $B(\bs, \eps) \cap S$ for any $\eps <
\tau/2$, in particular, $\eps < 1/(2\kappa)$. We obtain an estimate on the image of $P_{T_\bs}$ as follows. We note that~\cite[proof of
Lem.~5.3]{1349695} implies that
\begin{equation}
\label{eq:bound_image1}
P_{T_\bs}(B(\bs, \eps) \cap S) \supseteq B(\bs, \eps \cos \arcsin(\eps/(2 \tau))) \cap T_\bs.
\end{equation}
Furthermore,
\begin{equation}
\label{eq:bound_image2}
\text{if } \eps \leq 1/(4\kappa), \ \cos \arcsin(\eps/(2 \tau)) \geq \cos \arcsin(\kappa \eps/2) \geq \sqrt{63/64} > 1/2.
\end{equation}
Combining~\eqref{eq:bound_image1} and~\eqref{eq:bound_image2}, we conclude that
\begin{equation}
\label{eq:bound_image3}
P_{T_\bs}(B(\bs, \eps) \cap S) \supseteq B(\bs, \eps/2) \cap T_\bs, \quad \forall \eps \leq 1/(4\kappa).
\end{equation}
In particular, for $\eps=1/(4\kappa)$, we obtain that the range of  $P_{T_\bs}$ (when applied to $B(\bs, 1/(4\kappa)) \cap S$) contains the ball
$B(\bs, 1/(8 \kappa)) \cap T_\bs$.

Next, for any $\bs' \in B(\bs, 1/(4 \kappa))\cap S$, the derivative of the linear operator $P_{T_\bs}$ in the direction $\bu$, a unit vector in
$T_{\bs'}$, is 
\begin{equation}
\label{eq:deriv_direction_projection}
\nabla_{\bu}(P_{T_\bs}) = (P_{T_\bs}) \cdot \bu =\cos \theta_1(T_\bs, \Sp\{\bu\}) \geq \cos \theta_1(T_\bs, T_{\bs'}),
\end{equation}
where $\theta_1$ denotes the largest principal angle between the corresponding subspaces. In order to further bound from below the RHS
of~\eqref{eq:deriv_direction_projection}, we couple~\cite[Props.~6.2, 6.3]{1349695} and use $\tau \geq 1/\kappa$ to obtain that
\begin{equation}
\label{eq:est_for_deriv_direction}
\cos \theta_1(T_\bs, T_{\bs'}) \geq \sqrt{1 - 2 \kappa \|\bs - \bs'\|}.
\end{equation}
Combining~\eqref{eq:deriv_direction_projection} and~\eqref{eq:est_for_deriv_direction} we conclude that $P_{T_\bs}^{-1}$ has Lipschitz constant
bounded by $\sqrt{2}$ in $B(\bs, 1/(4 \kappa)) \cap T_\bs$.

For the inclusions, we use the fact that
\begin{equation}
\label{eq:to_get_inclusion}
\|P_{T_\bs}(\bx) - \bx\| \leq (\kappa/2)\cdot \|\bs - \bx\|^2, \quad \forall \bx, \bs \in S,
\end{equation}
which appears in~\cite[Th.~4.18(2)]{MR0110078}. This immediately implies the first inclusion---which actually holds for any $\eps > 0$ and with $\kappa$ replaced by $\kappa/2$.  The second
inclusion follows by combining~\eqref{eq:bound_image3} with~\eqref{eq:to_get_inclusion}.
\end{proof}

Next, we estimate the volume of the intersection of the neighborhood of a surface and a ball centered at a point within that neighborhood.
\begin{lemma}[\cite{pairwise}, Lem.~1] \label{lem:vol1}
For $S$ satisfying (\ref{eq:S-vol}), $\bx \in B(S,\tau)$ and $\eps, \tau > 0$,
$$\vol_D(B(S, \tau) \cap B(\bx, \eps)) \asymp \eps^d (\eps \wedge \tau)^{D-d}, \quad \vol_D(B(S, \tau)) \asymp \tau^{D-d}.$$
\end{lemma}

The following result is on the approximation of a set of points in the neighborhood of a $d$-dimensional affine subspace by a $d$-dimensional affine subspace generated by a subset of $d+1$ points.
\begin{lemma} \label{lem:height}
There is a constant $C > 0$ depending only on $d$ such that, if $\bz_1, \dots, \bz_m \in B(L, \eta)$, with $L \in \cA_d$ and $m \geq d+2$, then there exists $H \in \cA_d$ generated by $d+1$ points among $\bz_1, \dots, \bz_m,$ such that $\bz_1, \dots, \bz_m \in B(H, C \eta)$.
\end{lemma}
\begin{proof}
\def\span{{\rm span}}
\def\aspan{{\rm aspan}}
\def\bg{{\bf g}}
For points $\ba_1, \dots, \ba_k$, let $\aspan\{\ba_1, \dots, \ba_k\}$ denote the affine subspace of minimum dimension passing through $\ba_1, \dots, \ba_k$.
Let $(i_1, i_2) \in \argmax_{i,j} \|\bz_i - \bz_j\|$ and, for $d \geq k \geq 3$,
$$i_k \in \argmax_{i \neq i_1, \dots, i_{k-1}} \dist(\bz_i, \aspan\{\bz_{i_1}, \dots, \bz_{i_{k-1}}\}).$$
Let $A_k = \aspan\{\bz_{i_1}, \dots, \bz_{i_{k+1}}\}$, for $d \geq k \geq 1$.
Define $\lambda_1 = \|\bz_{i_2} - \bz_{i_1}\|$ and, for $d \geq k \geq 2$, $\lambda_k = \dist(\bz_{i_{k+1}}, \span\{\bz_{i_1}, \dots, \bz_{i_{k}}\})$.  Also, let $\bv_1 = (\bz_{i_2} - \bz_{i_1})/\lambda_1$ and, for $k \geq 2$, $\bv_k = (\bz_{i_{k+1}} - P_{A_{k-1}} \bz_{i_{k+1}})/\lambda_k.$
Without loss of generality, assume that $\bz_{i_1}$ is the origin, which allows us to identify a point $\bz$ with the vector $\bz - \bz_{i_1}$.  Take $\bz \in \{\bz_1, \dots, \bz_m\}$ and express it as $\bz = a_1 \bv_1 + \cdots + a_d \bv_d + \bw$, with $\bw \perp A_d$.  We show that $\|\bw\| \leq C \eta$ for a constant $C$ depending only on $d$, which implies that $\bz \in B(A_d, C \eta)$.
Let $C_1 > 0$, to be made sufficiently large later.
By construction, $A_1 \subset \cdots \subset A_d$ and $\lambda_1 \geq \cdots \geq \lambda_d$ with $\|P_{A_{k-1}^\perp} \bz\| \leq \lambda_k$ for all $k = 1, \dots, d$.  Consequently, if $\lambda_d \leq C_1 \eta$, then $\|\bw\| \leq \|P_{A_{d-1}^\perp} \bz\| \leq \lambda_d \leq C_1 \eta$ and we are done.  Therefore, assume that $\lambda_d > C_1 \eta$.
Define $\bq_k = P_L \bv_k$.  We have
$$
\|\bq_k -\bv_k\| = \|P_{L^\perp} \bv_k\| = \frac1{\lambda_k} \|P_{L^\perp} P_{A_{k-1}^\perp} \bz_{i_{k+1}}\| \leq \frac1{\lambda_k} \|P_{L^\perp} \bz_{i_{k+1}}\| \leq \frac\eta{\lambda_k} \leq \frac1{C_1}.
$$
Hence, for $C_1$ large enough, $\bq_1, \dots, \bq_d$ are linearly independent, and therefore span $L$.  Suppose this is the case and define matrices $\bV$ with columns $\bv_1, \dots, \bv_d$ and $\bQ$ with columns $\bq_1, \dots, \bq_d$.  Then, by continuity, for $C_1$ large enough we have
$$\|P_L - P_{A_d}\| = \|\bQ (\bQ^T \bQ)^{-1} \bQ^T - \bV \bV^T\| \leq 1/2,$$
where $\|\cdot\|$ here denotes the (Euclidean) operator norm.  When $C_1$ is that large, we have
$$
\|P_L \bw\| = \|(P_L - P_{A_d}) \bw\| \leq \frac12 \|\bw\| \leq \frac12 \left(\|P_{L} \bw\| + \|P_{L^\perp} \bw\|\right),
$$
so that $\|P_{L} \bw\| \leq \|P_{L^\perp} \bw\|$.
Now, using the triangle inequality,
\beq \label{eq:w-bound}
\|P_{L^\perp} \bw\| \leq \|P_{L^\perp} \bz\| + |a_1| \|P_{L^\perp} \bv_1\| + \cdots + |a_d| \|P_{L^\perp} \bv_d\|.
\eeq
Because $\bz \in B(L, \eta)$, we have $\|P_{L^\perp} \bz\| \leq \eta$.  For the other terms, we have
$\|P_{L^\perp} \bv_k\| \leq \eta/\lambda_k$ as before, and, using the fact that, by construction, the $\bv_1, \dots, \bv_d$ are orthonormal with $A_k = \span\{\bv_1, \dots, \bv_k\}$ and $\|P_{A_{k-1}^\perp} \bz\| \leq \lambda_k$, together with the Cauchy-Schwartz inequality, we also have
$$
|a_k| = |\bv_k^T \bz| = \left|\bv_k^T P_{A_{k-1}^\perp} \bz\right| \leq \lambda_k.
$$
Hence, the RHS in \eqref{eq:w-bound} is bounded by $(d+1) \eta$, implying
$$\|\bw\| \leq \|P_{L} \bw\| + \|P_{L^\perp} \bw\| \leq 2 \|P_{L^\perp} \bw\| \leq 2 (d+1) \eta.$$
We then let $C = \max(C_1, 2d+2)$.
\end{proof}

Below we provide an upper bound on the volume of the three-way intersection of the neighborhood of a surface, a ball centered at a point on the surface and the neighborhood of an affine $d$-dimensional subspace passing through that point, in terms of the angle between this subspace and the tangent to the surface at that same point.
The principal angles between linear subspaces $L, L' \in \cA_d$, denoted by
$$\frac\pi{2} \geq \theta_1(L,L') \geq \cdots \geq \theta_d(L, L') \geq 0,$$
are recursively defined as follows:
$$
\cos \theta_r(L,L') = \min_{\bu \in L} \min_{\bu' \in L'} \bu^T \bu' = \bu_r^T \bu_r',
$$
subject to
\begin{eqnarray*}
\|\bu\| = \|\bu'\| = 1; \\
\bu^T \bu_s = 0, \, \forall s = 1, \dots, r-1; \\
{\bu'}^T \bu'_s = 0, \, \forall s = 1, \dots, r-1.
\end{eqnarray*}
Note that the orthogonality constraints are void when $r = 1$.  (Some authors use the reverse ordering, e.g,~\cite{MR1417720}.)
\begin{lemma} \label{lem:vol-angle}
Consider a surface $S \in \cS_d^2(\kappa)$.  Suppose $\eps \geq \eta \vee \tau$, $\eta \geq \eps^2$ and $\tau > 0$.  Let $\Psi$ be the uniform distribution on $B(S, \tau)$.  For $\bs \in S$, let $T_\bs$ be the tangent space to $S$ at $\bs$.  Then for $L \in \cA_{d}$ containing $\bs$, 
$$\Psi(B(\bs, \eps) \cap B(L, \eta)) \prec \eps^{d} (1 \wedge (\eta/\tau))^{D-d} \prod_{j=1}^d \left(1 \wedge \frac{\eta \vee \tau}{\eps \, \theta_j(L, T_\bs)}\right).$$
\end{lemma}
\begin{proof}
Fix $\bs \in S$ and $L\in \cA_{d}$ containing $\bs$, and let $T := T_\bs$ and $\theta_j := \theta_j(L, T)$ for short.
By definition,
$$\Psi(B(\bs, \eps) \cap B(L, \eta)) = \frac{\vol_D(B(S, \tau) \cap B(\bs, \eps) \cap B(L, \eta))}{\vol_D(B(S, \tau))}.$$

By Lemma~\ref{lem:vol1}, it suffices to show that
$$\vol_D(B(S, \tau) \cap B(L, \eta) \cap B(\bs, \eps)) \prec \eps^{d} (\eta \wedge \tau)^{D-d} \prod_{j=1}^d \left(1 \wedge \frac{\eta \vee \tau}{\eps \, \theta_j}\right).$$
We divide the proof into two cases; though the proof is similar for both, the first case is simpler and allows us to introduce the main ideas with ease before generalizing to the second case.

{\it Case $\eps^2 \leq \tau$.}
We use Lemma~\ref{lem:S-approx} and the fact that $\tau \geq \eps^2$, to get
\begin{equation} \label{eq:SinT-1}
B(S, \tau) \cap B(\bs, \eps) \subset B(T, (1 + \kappa) \tau) \cap B(\bs, \eps).
\end{equation}
Ignoring the constant factor $1 + \kappa$, we bound
$$\vol_D(B(T, \tau) \cap B(L, \eta) \cap B(\bs, \eps)).$$
We may assume without loss of generality that $\bs$ is the origin and
$$T = {\rm span}\{\be_1, \dots, \be_d\}, \text { and }$$
$$L = {\rm span}\{(\cos \theta_1) \be_1 + (\sin \theta_1) \be_{d+1}, \dots, (\cos \theta_d) \be_d + (\sin \theta_d) \be_{2d}\}.$$
Then
\begin{eqnarray*}
B(T, \tau) & =& \{(z_1, \dots, z_D): \sum_{j > d} z_j^2 \leq \tau^2\}; \\
B(L, \eta) & = & \{(z_1, \dots, z_D): \sum_{j \leq d} (z_{j} \sin \theta_j - z_{d+j} \cos \theta_j)^2 + \sum_{j > 2d} z_j^2 \leq \eta^2\}; \\
B(\bs, \eps) & = & \{(z_1, \dots, z_D): \sum_{j} z_j^2 \leq \eps^2\}.
\end{eqnarray*}
Take $j \leq d$; since $|z_{d+j}| \leq \tau$, we have
$$|z_j \sin \theta_j - z_{d+j} \cos \theta_j| \leq \eta \quad \Rightarrow \quad |z_j| \leq 2 (\eta \vee \tau) / \sin \theta_j \leq \pi (\eta \vee \tau) / \theta_j.$$
Therefore,
$$B(T, \tau) \cap B(L, \eta) \cap B(\bs, \eps) \subset \prod_{j=1}^d \left[- \eps \wedge \frac{\pi (\eta \vee \tau)}{\theta_j}, \eps \wedge \frac{\pi (\eta \vee \tau)}{\theta_j}\right] \times B_{D-d}(0,\eta \wedge \tau).$$
From that we obtain the desired bound.

{\it Case $\tau \leq \eps^2$.}
The arguments here are a little different and we simply bound $\vol_D(B(S,\tau) \cap B(\bs, \eps)).$  Assume that $\eps < 1/(8\kappa)$.
Because $P_T$ is contractile, we have
$$P_T(S \cap B(\bs, \eps)) \subset T \cap B(\bs, \eps) = B_d(0, \eps),$$
so that, by Lemma~\ref{lem:S-approx},
$$S \cap B(\bs, \eps) \subset P_T^{-1} (B_d(0, \eps)),$$
where $P_T^{-1} : B_d(0, \eps) \to S \cap B(s, \eps)$.
Hence,
$$B(S,\tau) \cap B(\bs, \eps) \subset \{(\ba, \bb): \ba \in B_d(0, \eps), \|\bb - P_T^{-1}(\ba)\| \leq \tau\}.$$
And by direct integration, the set on the RHS has $D$-volume of order $\eps^d \tau^{D-d}$ since $P_T^{-1}$ is Lipschitz on $B_{d}(0,\eps)$ by Lemma~\ref{lem:S-approx}.
\end{proof}

A companion of the previous result, the following lemma provides a lower bound on the angle between the  affine subspace and the tangent.
\begin{lemma} \label{lem:L-angle}
Let $\eps, \eta > 0$, and
take $S \in \cS_d^2(\kappa)$.
Suppose $L \in \cA_d$ is such that $B(L, \eta)$ contains $\bs \in S$ and $\by \in B(\bs, \eps)$.  Let $T_\bs$ the tangent to $S$ at $\bs$.
Then
$$\theta_1(L, T_\bs) \geq \frac{\dist(\by,S) - 2 \kappa \eps^2 - \eta}{\eps + \eta}.$$
\end{lemma}
\begin{proof}
Let $T$ denote $T_\bs$ for short, and let $L'$ be the line passing through $\bs$ and $P_L(\by)$.
Since $L' \subset L$, we have $\theta_1 (L, T) \geq \theta_1 (L', T)$, and using the triangle inequality and the fact that $\theta \geq \sin \theta,$ for $\theta \geq 0$, this is bounded below by
$$
\frac{\dist(P_L(\by),T)}{\dist(P_L(\by), \bs)} \geq \frac{\dist(\by,T) - \eta}{\dist(\bs, \by) + \eta}.
$$
The denominator does not exceed $\eps + \eta$.
For the numerator,
$$
\dist(\by,T) = \|P_T(\by) - \by\| \geq \dist(\by,S) - \dist(P_T(\by), S).
$$
Since $\| \by - \bs \| \leq \eps$, we have $P_T(\by) \in T \cap B(\bs, \eps)$, so that $\dist(P_T(\by), S) \leq 2 \kappa \eps^2$ by Lemma~\ref{lem:S-approx}.
Consequently, the numerator is bounded from below by $\dist(\by,S) - \kappa \eps^2 - \eta$.
\end{proof}

Next is another result estimating some volume intersections.  It is similar to Lemma~\ref{lem:vol-angle}, though the conditions are different.
\begin{lemma} \label{lem:vol2}
Consider a surface $S \in \cS_d^2(\kappa)$.  Let $\Psi$ be the uniform distribution on $B(S, \tau)$.  Then for $\eps \geq \eta$ and $\tau > 0$,
$$
\sup_{\by, L} \Psi(B(\by, \eps) \cap B(L, \eta)) \prec \eps^d (1 \wedge (\eta/\tau))^{D-d},
$$
where the supremum is over $\by \in \bbR^D$ and $L \in \cA_d$, and the implicit constants depend only on $\kappa, d$.
Also, for $\eps \geq 10 \eta$, $\eta \geq 10 \kappa \eps^2$ and $\tau > 0$, and any $\bx \in B(S, \tau)$,
$$
\sup_{L} \Psi(B(\bx, \eps) \cap B(L, \eta)) \succ \eps^d (1 \wedge (\eta/\tau))^{D-d}.
$$
\end{lemma}
\begin{proof}
The proof is similar to that of Lemma~\ref{lem:vol-angle}.  We divide the proof into two parts.

{\it Upper bound.}
Let $\bx \in B(S, \tau) \cap B(\by, \eps) \cap B(L, \eta)$.
When $\eta \geq \tau$, we use
$$B(S, \tau) \cap B(\by, \eps) \cap B(L, \eta) \subset B(S, \tau) \cap B(\bx, 2 \eps),$$
while, when $\eta \leq \tau$, we use
$$B(S, \tau) \cap B(\by, \eps) \cap B(L, \eta) \subset B(L, \eta) \cap B(\bx, 2 \eps).$$
In both cases, we conclude with Lemma~\ref{lem:vol1}.

{\it Lower bound.}
Let $\bs$ be the point on $S$ closest to $\bx$, with tangent subspace $T$.
When $\eta \geq 2\tau + 4 \kappa \eps^2$, take as $L$ the translate of $T$ passing through $\bx$ and use Lemma~\ref{lem:S-approx} to get
$$B(S, \tau) \cap B(\bx, \eps) \subset B(T, \tau + \kappa (\tau + \eps)^2) \subset B(L, \eta),$$
and therefore
$$B(S, \tau) \cap B(\bx, \eps) \cap B(L, \eta) \supset B(S, \tau) \cap B(\bx, \eps).$$
We then use Lemma~\ref{lem:vol1}.
Now, suppose $\eta \leq 2\tau + 4 \kappa \eps^2$ and notice that, since $\eta \geq 10 \kappa \eps^2$, we have $\tau \geq 3 \kappa \eps^2$.  First, assume that $\eps \geq 10 \tau$. 
We use Lemma~\ref{lem:S-approx} to get
$$B(S, \tau) \cap B(\bx, \eps) \supset B(T, \tau - 2 \kappa \eps^2) \cap B(\bs, \eps) \cap B(\bx, \eps),$$
and therefore,
\beq \label{eq:BBBB}
B(S, \tau) \cap B(\bx, \eps) \cap B(L, \eta) \supset B(T, \tau - 2 \kappa \eps^2) \cap B(L, \eta) \cap B(\bs, \eps) \cap B(\bx, \eps).
\eeq
Without loss of generality, assume that $\bx$ is the origin, $L = {\rm span}\{\be_1, \dots, \be_d\}$.  Since the volume is least when $\|\bx - \bs\| = \tau$, assume that $\bs = \tau \be_{d+1}$ (seen as a point in space).  Define $\nu = (\eta+2\kappa \eps^2)/2$ and note that $\nu \leq \eta \wedge (2\tau)$ by the conditions on $\eta$ and $\tau$.  Then
\begin{eqnarray*}
B(T, \tau - 2 \kappa \eps^2) \cap B(L, \eta) & \supset & \{(z_1, \dots, z_D): \sum_{j > d + 1} z_j^2 + (z_{d+1} - \nu)^2 \leq (\eta/3)^2\}; \\
B(\bs, \eps) & = & \{(z_1, \dots, z_D): \sum_{j \neq d+1} z_j^2 + (z_{d+1} - \tau)^2 \leq \eps^2\}; \\
B(\bx, \eps) & = & \{(z_1, \dots, z_D): \sum_{j} z_j^2 \leq \eps^2\}.
\end{eqnarray*}
By the conditions imposed on $\eps, \eta, \tau$, the RHS in \eqref{eq:BBBB} contains
$$B_d(0, \eps/10) \times [\eta/4, 3\eta/4] \times B_{D-d-1}(0, \eta/10).$$
Therefore the result.
Finally assume that $\tau \geq \eps/10$ and take $L$ passing through $\bx$ and $\bz = (1 - \lambda) \bx + \lambda \bs$, where $\lambda = \eps/(2\tau)$.  We have $\|\bz - \bx\| \leq \eps/2$ and $\|\bz - \bs\| \leq \tau - \eps/2$, so that $B(\bz,\eps/2) \subset B(S, \tau) \cap B(\bx, \eps)$ by the triangle inequality.  Hence,
$$
B(S, \tau) \cap B(L, \eta) \cap B(\bx, \eps) \supset B(L, \eta) \cap B(\bz, \eps/2).
$$
We then conclude with Lemma~\ref{lem:vol1}.
\end{proof}

\begin{lemma} \label{lem:vol2-out}
Let $\Psi$ be the uniform distribution on a measurable subset $A \subset \bbR^D$ of positive $D$-volume.  Then for $\eps \geq \eta$,
\begin{equation} \nonumber 
\sup_{\by, L} \Psi(B(\by, \eps) \cap B(L, \eta)) \prec \eps^d \eta^{D-d},
\end{equation}
where the supremum is over $\by \in \bbR^D$ and $L \in \cA_d$, and the implicit constant depends only on $d$ and $\vol_D(A)$.
\end{lemma}

\begin{proof}
The proof is parallel to (and simpler than) that of Lemma~\ref{lem:vol2}.  We omit details.
\end{proof}

\subsection{A Perturbation Bound}

In the proof of Theorem~\ref{th:linear}, we follow the strategy outlined in~\cite{Ng02} based on verifying the following conditions (where \ref{a4} has been simplified).  Let $I_k = \{i: \bx_i \in \cX_k\}$ and let $\brW_k$ denote the matrix with coefficients indexed by $i, j \in I_k$ and defined as
$$\rW_{ij} = \sum_{i_1, \dots, i_{m-2} \in I_k} \alpha_{\ud}(\bx_{i}, \bx_{j}, \bx_{i_1},\ldots,\bx_{i_{m-2}}),
\quad \rD_i = \sum_{j \in I_k} \rW_{ij}.$$
Let $\rW_{ij} = 0$ if $i \in I_k, j \in I_\ell,$ with $k \neq \ell$.
Those are the coefficients of $\bW$ and $\bD$ under infinite separation, i.e.,assuming $\delta = \infty$.  (In fact $\delta > \eps + 2 \tau$ is enough since we use the simple kernel.)

\renewcommand{\theenumi}{(A\arabic{enumi})}
\renewcommand{\labelenumi}{\theenumi}
\begin{enumerate}
\item \label{a1}  For all $k$, the second largest eigenvalue of $\brW_k$ is bounded above by $1 - \gamma$.
\item \label{a2}  For all $k, \ell$, with $k \neq \ell$,
$$\sum_{i \in I_k} \sum_{j \in I_\ell} \frac{W_{ij}^2}{\rD_{i} \rD_{j}} \leq \nu_1.$$
\item \label{a3}  For all $k$ and all $i \in I_k$,
$$\frac{1}{\rD_i} \sum_{j \notin I_k} W_{ij} \leq \nu_2 \left(\sum_{s, t \in I_k} \frac{W_{st}^2}{\rD_{s} \rD_{t}}\right)^{-1/2}.$$
\item \label{a4}  For all $k$ and all $i, j \in I_k$, $\rD_i \leq Q \rD_j$.
\end{enumerate}

The following result is a slightly modified version of~\cite[Th.~2]{Ng02}, stated and proved in~\cite[Th.~7]{pairwise}.  See also~\cite[Th.~4.5]{spectral_theory}.  Recall the matrix $\bV$ defined in Algorithm~\ref{algo:NJW}.

\begin{theorem} \label{th:NJW-orig}
Let $\bv_1, \dots, \bv_\sN$ denote the row vectors of $\bV$.
Under \ref{a1}-\ref{a4}, there is an orthonormal set $\{\br_1, \dots, \br_\sK\} \subset \bbR^K$ such that,
$$\frac{1}{N} \sum_{k = 1}^K \sum_{i \in I_k} \|\bv_i - \br_k\|^2 \leq 4 Q \gamma^{-2} (K^2 \nu_1 + K \nu_2^2).$$
\end{theorem}

\section{Main Proofs}

For a set $A$, its cardinality is denoted by $\# A$.
Throughout the paper, $C$ denotes a generic constant that does not depend on the sample size $N$ and satisfies $C \geq 1$.

\subsection{Proof of Theorem~\ref{th:linear}}
\label{sec:proof-linear}
Given Theorem~\ref{th:NJW-orig}, we turn to proving that the four conditions \ref{a1}-\ref{a4} hold with probability tending to one with $\nu_1 = \nu_2^2 = (\rho_\sN/\zeta)^{-m/2}$, $\gamma > C^{-m} N^{-2}$ and $Q \leq C^m$ for some constant $C > 0$.  Since $m \log(\rho_\sN/\zeta) \gg \log N$, this implies
$$
\max_{i=1,\dots,N} \min_{k=1,\dots,K} \|\bv_i - \br_k\| \to 0.
$$
Therefore, since the $\br_k$'s are themselves orthonormal, the $K$-means algorithm with near-orthogonal initialization outputs the perfect clustering.

We restrict ourselves to the case where $\tau \leq (\rho_\sN^2 \log(N)/N)^{1/d}$, for otherwise $\eta \geq \eps$ and HOSC is essentially SC, studied in~\cite{pairwise}.  With that bound on $\tau$, (\ref{eq:eps}) reduces to $\eps \geq (\rho_\sN^2 \log(N)/N)^{1/d}$.  By the same token, we assume that $\eta \leq \eps$, so that $\eps \geq \eta \geq \tau + \rho_\sN \eps^2$.

To verify conditions \ref{a2}, \ref{a3} and \ref{a4} we need to estimate the degree of each vertex under infinite separation and the edge weights under finite separation.  We start with the case of infinite separation.
\begin{proposition} \label{prp:rD}
With probability at least $1 - N^{-\rho_\sN^2/(K \zeta)}$,
\begin{equation} \label{eq:rW}
{\bf 1}_{\{\| \bx_i - \bx_j \| \leq \eps/2\}} N_k \eps^d \prec\rW_{ij}^{1/(m-2)} \prec {\bf 1}_{\{\| \bx_i - \bx_j \| \leq \eps\}} N_k \eps^d;
\end{equation}
and also,
\begin{equation} \label{eq:rD}
\rD_i^{1/(m-1)} \asymp N_k \eps^d,
\end{equation}
uniformly over $i, j \in I_k$ and $k = 1, \dots, K$.
\end{proposition}

\begin{proof}
Within a cluster, the linear approximation factor in \eqref{eq:linear-affinity} is a function of the proximity factor.  This is due to Lemma~\ref{lem:S-approx}.  Formally, let $\rG_{i, \eps}$ denote the degree of $\bx_i$ in the neighborhood graph built by SC, i.e.
\[
\rG_{i, \eps} = \# \{j \in I_k, j \neq i: \bx_j \in B(\bx_i, \eps)\},
\]
Then Proposition~\ref{prp:rD} is a direct consequence of Lemma~\ref{lem:rG-rD}, which relates $\rG_{i, \eps}$ to $\rW_{ij}$ and $\rD_i$, and of Proposition~\ref{prp:rG}, which estimates $\rG_{i, \eps}$.
\end{proof}

\begin{lemma} \label{lem:rG-rD}
We have
$${\bf 1}_{\{\| \bx_i - \bx_j \| \leq \eps/2\}} (\rG_{i, \eps/2} - 1)^{\{m-2\}} \leq \rW_{ij} \leq {\bf 1}_{\{\| \bx_i - \bx_j \| \leq \eps\}} (\rG_{i, \eps} - 1)^{\{m-2\}},$$
and,
$$\rG_{i, \eps/2} (\rG_{i, \eps/2} - 1)^{\{m-2\}}  \leq \rD_i \leq \rG_{i, \eps} (\rG_{i, \eps} - 1)^{\{m-2\}},$$
where $r^{\{m\}} = r (r-1) \cdots (r - m  +1)$.
\end{lemma}
Note that $r^{\{m\}} \leq r^m$, and $r^{\{m\}} \geq (r/3)^m$ for $r \geq m$.
\begin{proof}
We focus on the first expression, as the second expression is obtained by summing the first one over $j \in I_k, \, j\neq i$, where $k$ is such that $i \in I_k$.
Therefore, fix $i, j \in I_k$.  The upper bound on $\rW_{ij}$ comes from the fact that
$$\diam(\bx_i, \bx_j, \bx_1, \dots, \bx_{m-2}) \leq \eps \quad \Rightarrow \quad \bx_1, \dots, \bx_{m-2} \in B(\bx_i, \eps).$$
The lower bound comes from
$$\bx_1, \dots, \bx_{m-2} \in B(\bx_i, \eps/2) \quad \Rightarrow \quad \diam(\bx_i, \bx_1, \dots, \bx_{m-2}) \leq \eps,$$
and the fact that,
$$\bx_1, \dots, \bx_{m-2} \in B(S_k,\tau) \cap B(\bx_i, \eps/2) \quad \Rightarrow \quad \bx_1, \dots, \bx_{m-2} \in B(T_{\bs_i}, \eta),$$
where $\bs_i$ is the point on $S_k$ closest to $\bx_i$.  Indeed, take $\bx \in B(S_k,\tau) \cap B(\bx_i, \eps/2)$ and let $\bs \in S_k$ such that $\|\bx -\bs\| \leq \tau$.  By the triangle inequality, $\|\bs - \bs_i\| \leq \eps/2 + 2\tau$, so that, by Lemma~\ref{lem:S-approx}, $\bs \in B(T_{\bs_i}, \kappa (\eps/2 + 2\tau)^2)$.  Therefore, $\bx \in B(T_{\bs_i}, \kappa (\eps/2 + 2\tau)^2 + \tau)$.  We then conclude with the fact that $\tau \leq \eps$ and $\eta \geq \tau + \rho_\sN \eps^2$, with $\rho_\sN \to \infty$.
\end{proof}

Note that $N \leq K \zeta N_k$, which together with (\ref{eq:eps}) implies
\begin{equation} \label{eq:N-cond}
N_k \eps^d (1 \wedge (\eps/\tau))^{D-d} \geq \rho_\sN^2/(K \zeta) \log N, \quad \forall k =1, \dots, K.
\end{equation}
The following bound on $\rG_{i, \eps}$ is slightly more general than needed at this point.
\begin{proposition} \label{prp:rG}
Assume that \eqref{eq:N-cond} holds.  Then with probability at least $1 - N^{-\rho_\sN^2/(K \zeta)}$,
\begin{equation} \label{eq:rG}
\rG_{i, \eps} \asymp N_k \eps^d (1 \wedge (\eps/\tau))^{D-d},
\end{equation}
uniformly over $i \in I_k$ and $k = 1,\dots, K$.
\end{proposition}
\begin{proof}
This is done in the proof of~\cite[Eq.~(A4)]{pairwise} and we repeat the arguments here for future reference.
Let $\Psi_k$ denote the uniform distribution on $B(S_k, \tau)$.  By definition, for any (measurable) set $A$,
\begin{equation} \label{eq:psi}
\Psi_k(A) = \frac{\vol_D(A \cap B(S_k, \tau))}{\vol_D(B(S_k, \tau))}.
\end{equation}
Since $\rG_{i, \eps}$ is the sum of independent Bernoulli random variables, by Lemma~\ref{lem:hoeff}, it suffices to bound it in expectation.  Using Lemma~\ref{lem:vol1}, we have
$$\expect{\rG_{i, \eps}} = N_k \Psi_k(B(\bx_i, \eps)) \asymp N_k \eps^d (1 \wedge (\eps/\tau))^{D-d}.$$
Applying Lemma~\ref{lem:hoeff} and (\ref{eq:N-cond}), we then get
$$\pr{\rG_{i, \eps} > 16 \, \E \rG_{i, \eps}} \ \vee \ \pr{\rG_{i, \eps} < \E \rG_{i, \eps}/8} \leq N^{-2(\rho_\sN^2/(K \zeta))}.$$
We then apply the union bound and use the fact that $N \cdot N^{-2(\rho_\sN^2/(K \zeta))} \leq N^{-\rho_\sN^2/(K \zeta)}$, since $\rho_\sN^2 \to \infty$.
\end{proof}

We now turn to bounding the size of the edge weights $W_{ij}$ under finite separation.  We do so by comparing them with the edge weights under infinite separation.
\begin{proposition} \label{prp:W}
With probability at least $1 - N^{-\rho_\sN}$,
\begin{equation} \label{eq:W}
(W_{ij} - \rW_{ij})^{1/(m-2)} \prec {\bf 1}_{\{\|\bx_i - \bx_j\| \leq \eps\}} N \eps^d / \rho_\sN.
\end{equation}
uniformly over $i \in I_k, j\in I_\ell$ and $k,\ell = 1, \dots, K$.
\end{proposition}

\begin{proof}
If $k = \ell$, $W_{ij} - \rW_{ij}$ is the sum of $\alpha_{d}(\bx_i, \bx_j, \bx_{i_1}, \dots, \bx_{i_{m-2}})$ over (distinct) $i_1, \dots, i_{m-2}$ that are not all in $I_k$.  When $k \neq \ell$, $\rW_{ij} = 0$ and $W_{ij}$ is again the same sum except this time over all (distinct) $i_1, \dots, i_{m-2}$.  Both situations are similar and we focus on the latter.  We assume that $\|\bx_i - \bx_j\| \leq \eps$, for otherwise the bound is trivially satisfied.  Note that this implies that $\rho_\sN \eta \leq \delta - 2 \tau \leq \eps$.

Define
\begin{equation} \label{eq:G-def}
G_{i, \eps} = \# \{j \neq i: \bx_j \in B(\bx_i, \eps)\},
\end{equation}
which is the equivalent of $\rG_{i, \eps}$ under finite separation, as well as
$$
H_{i,\eps,\eta}(L) = \# \{j \neq i: \bx_j \in B(\bx_i, \eps) \cap B(L, \eta)\},
$$
and
$$
H^*_{i,j,\eps,\eta} = \max_{M} H_{i,\eps, \eta}(L_M),
$$
where the maximum is over all $M \subset \{1, \dots, N\}$, of size $|M| = d+1$ such that $\bx_j \in B(L_M, \eta)$.
Then Proposition~\ref{prp:W} is a direct consequence of Lemma~\ref{lem:W-GH}, which relates $G_{i, \eps}$ and $H^*_{i,j,\eps,\eta}$ to $W_{ij}$, and of Propositions~\ref{prp:G} and~\ref{prp:H}, which bound $G_{i, \eps}$ and $H^*_{i,j,\eps,\eta}$, respectively.
\end{proof}

\begin{lemma} \label{lem:W-GH}
There is a constant $C > 0$ such that
\begin{equation} \label{eq:W-ub}
W_{ij} \leq (G_{i,\eps} + 1)^{d+1} (H^*_{i,j,\eps,C\eta})^{\{m-d-1\}}.
\end{equation}
\end{lemma}
\begin{proof}
By definition of the affinity~\eqref{eq:linear-affinity} and the triangle inequality, we have
$$W_{ij} \leq \sum_{M} {\bf 1}_{\{\exists L \in L_d: \bx_{n} \in B(\bx_i, \eps) \cap B(L, \eta), \forall n \in M \cup\{i, j\}\}},$$
where the sum is over $M \subset \{1, \dots, N\}$ such that $|M| = m - 2$ and $i, j \notin M$.
%
For a subset $M \subset \{1, \dots, N\}$, of size $|M| = d+1$, let $L_M$ denote the affine subspace spanned by $\{\bx_n, n \in M\}$.
By Lemma~\ref{lem:height}, we may limit ourselves to subspaces $L$ that are generated by $d+1$ data points, obtaining
\begin{eqnarray}
W_{ij}
& \leq & \sum_{M} {\bf 1}_{\{\bx_{n} \in B(\bx_i, \eps), \forall n \in M\}}  \label{eq:W-prod-sum} \\
& & \hspace{.2in} \times \sum_{M'} {\bf 1}_{\{\bx_{n} \in B(\bx_i, \eps) \cap B(L_{M}, C \eta), \forall n \in M' \cup \{i,j\}\}} \nonumber,
\end{eqnarray}
where $M$ is any subset of $\{1, \dots, N\}$ of size $d+1$, and $M'$ is any subset of $\{1, \dots, N\} \setminus (M \cup \{i, j\})$ such that $M' \cup M \cup \{i, j\}$ is of size $m$.  Such an $M$ is of size at most $m-d-1$ and does not contain $i$ or $j$.  For any $M$, $B(\bx_i, \eps) \cap B(L_{M}, C \eta)$ contains at most $H^*_{i,j,\eps,C\eta}$ data points other than $\bx_i$ and $\bx_j$, so that the second sum is bounded by $(H^*_{i,j,\eps,C\eta})^{m-d-1}$ independently of $M$.  Similarly, $B(\bx_i, \eps)$ contains at most $G_{i,\eps} + 1$ points, so the first sum is bounded by $(G_{i,\eps} + 1)^{d+1}$.  The result follows.
\end{proof}

%
\begin{proposition} \label{prp:G}
Assume that \eqref{eq:N-cond} holds.  Then with probability at least $1 - N^{-\rho_\sN^2/(K \zeta)}$,
\begin{equation} \label{eq:G}
G_{i, \eps} \prec N \eps^d (1 \wedge (\eps/\tau))^{D-d},
\end{equation}
uniformly over $i = 1, \dots, N$.
\end{proposition}
\begin{proof}
We have
$$\expect{G_{i, \eps}} = \sum_\ell N_\ell \Psi_\ell(B(\bx_i, \eps)).$$
Now, by Lemma~\ref{lem:vol1}, for all $\ell$ such that $\dist(\bx_i, S_\ell) \leq \eps + \tau$,
$$\Psi_\ell(B(\bx_i, \eps)) \prec \eps^d (1 \wedge (\eps/\tau))^{D-d}.$$
Hence,
$$
\expect{G_{i, \eps}} \prec N \eps^d (1 \wedge (\eps/\tau))^{D-d}.
$$
We then use Lemma~\ref{lem:hoeff} and (\ref{eq:N-cond}).
\end{proof}

\begin{proposition} \label{prp:H}
With probability at least $1 - N^{-\rho_\sN}$,
\begin{equation} \label{eq:H}
H^*_{i,j,\eps,\eta} \prec \frac{N \eps^d}{\rho_\sN},
\end{equation}
uniformly over $i \in I_k$, $j \in I_\ell$ and $k \neq \ell$ in $\{1, \dots, K\}$.
\end{proposition}
\begin{proof}
For $L \in \cA_d$, $H_{i,\eps,\eta}(L)$ is a sum of independent Bernoulli random variables, with expectation
$$
\expect{H_{i,\eps,\eta}(L)} = \sum_\ell N_\ell \Psi_\ell(B(\bx_i, \eps) \cap B(L, \eta)).
$$
Take $\ell$ such that $B(S_\ell, \tau) \cap B(\bx_i, \eps) \cap B(L, \eta) \neq \emptyset$, and let $\bx$ be in that set and $\bs$ be the point on $S_\ell$ closest to $\bx$.  Then by the triangle inequality and the fact that $\eps \geq \eta \geq \tau$,
$$
B(S_\ell, \tau) \cap B(\bx_i, \eps) \cap B(L, \eta) \subset B(S_\ell, \tau) \cap B(\bs, 3\eps) \cap B(L_\bs, 3 \eta),$$
where $L_\bs$ is the translate of $L$ passing through $\bs$.
Therefore,
$$\Psi_\ell(B(\bx_i, \eps) \cap B(L, \eta)) \leq {\bf 1}_{\{\dist(\bx_i, S_\ell) \leq \eps + \tau\}} \cdot \sup_{\bs \in S_\ell} \Psi_\ell(B(\bs, 3\eps) \cap B(L_\bs, 3\eta)).$$
Our focus is on $L$ such that $\bx_i, \bx_j \in B(L, \eta)$, which transfers as $\bx_i, \bx_j \in B(L_\bs, 3 \eta)$ by the triangle inequality.  Since $\bx_i$ and $\bx_j$ belong to different clusters, for a given $\ell$, at least one of them does not belong to $\cX_\ell$.  Hence, by Lemma~\ref{lem:L-angle} and the fact that $\delta \gg \eta \geq \tau + \kappa \eps^2$,
$\theta_1(L, T_{\bs}) \succ \delta/\eps$ uniformly over $\bs \in S_\ell$ and $\ell$.  (Remember that $\theta_1(L,T)$ denotes the largest principal angle between $L$ and $T$.)  Together with Lemma~\ref{lem:vol-angle}, we thus get
$$
\Psi_\ell(B(\bs, 3\eps) \cap B(L_\bs, 3\eta)) \leq C \eps^d (\eta/\delta).
$$
Hence, by the fact that $\delta \geq \rho_\sN \eta$, we have
$$
\expect{H_{i,\eps,\eta}(L)} \leq C N \eps^d (\eta/\delta) \leq C N \eps^d/\rho_\sN.
$$
With Lemma~\ref{lem:hoeff} and (\ref{eq:eps}), we then get
$$\sup_L \pr{H_{i,\eps,\eta}(L) > 16 C N \eps^d/\rho_\sN} \leq N^{-2\rho_\sN}.$$
Hence, by the union bound,
\begin{equation} \label{eq:HX}
\pr{H^*_{i,j,\eps,\eta} > 16 C N \eps^d/\rho_\sN} \leq N^{d+1} \cdot N^{-2\rho_\sN}.
\end{equation}
The right hand side is bounded by $N^{-\rho_\sN}$ eventually.
\end{proof}

We now turn to verifying \ref{a1}-\ref{a4}.
\begin{itemize} \setlength{\itemsep}{.1in}
\item Verifying {\ref{a4}:}  (\ref{eq:rD}) immediately implies \ref{a4} with $Q = C^{m}$ for some constant $C > 0$.

\item Verifying {\ref{a3}:}  Take $k = 1, \dots, K$.  By (\ref{eq:rW}), (\ref{eq:rD}) and (\ref{eq:W}),
$$
\sum_{i,j \in I_k} \frac{W_{ij}^2}{\rD_{i} \rD_{j}} \prec \eps^{-2} ( 1 + (\rho_\sN/\zeta)^{-2 (m-2)}) \prec \eps^{-2},
$$
and also,
$$
\frac{1}{\rD_i} \sum_{j \notin I_k} W_{ij} \prec (N/N_k) \eps^{-d} (\rho_\sN/\zeta)^{-(m-1)}.
$$
Since $N/N_k \leq N$, $\eps \succ N^{-1/d}$ and $m \log(\rho_\sN/\zeta) \gg \log N$, we may take $\nu_2 = (\rho_\sN/\zeta)^{-m/2}$.

\item Verifying {\ref{a2}:}  Take $k, \ell = 1, \dots, K$, with $k \neq \ell$.  Then by (\ref{eq:rW}), (\ref{eq:rD}) and (\ref{eq:W}),
$$\sum_{i \in I_k} \sum_{j \in I_\ell} \frac{W_{ij}^2}{\rD_{i} \rD_{j}} \prec \eps^{-2d} (\rho_\sN/\zeta)^{-2 (m-2)}.$$
Since $\eps \succ N^{-1/d}$ and $m \log(\rho_\sN/\zeta) \gg \log N$, we may take $\nu_1 = (\rho_\sN/\zeta)^{-m}$.

\item Verifying {\ref{a1}:}
As suggested in \cite{Ng02}, we approach this through a lower bound on the Cheeger constant.  Let $\brZ_k$ be the matrix obtained from $\brW_k$ following SC.  That $\brZ_k$ has eigenvalue 1 with multiplicity 1 results from the graph being fully connected~\cite{Chung97}.  The Cheeger constant of $\brW_k$ is defined as:
$$h_k = \min_{|I| \leq N_k/2} \frac{\sum_{i \in I} \sum_{j \in I_k \setminus I} \rW_{ij}}{\sum_{i \in I} \rD_i},$$
where the minimum is over all subsets $I \subset I_k$ of size $|I| \leq N_k/2$.
The spectral gap of $\brZ_k$ is then at least $h_k^2/2$.
By (\ref{eq:rW})-(\ref{eq:rD}), there is a constant $C > 0$ such that,
$$h_k \geq C^{-m} (N_k \eps^d)^{-1} \min_{|I| \leq N_k/2}  \frac{\sum_{i \in I} \sum_{j \in I_k \setminus I} {\bf 1}_{\{\| \bx_i - \bx_j \| \leq \eps/2\}}}{|I|}.$$
From here, the proof is identical to that of~\cite[Eq.~(A1)]{pairwise}, which bounds the minimum from below by $1/N_k$, so that $h_k \geq C^{-m} N_k^{-1}.$
\end{itemize}

\subsection{Proof of Proposition~\ref{prop:intersect}}

From the proof of Theorem~\ref{th:linear}, it suffices to verify that \ref{a2} and \ref{a3} still hold under the conditions of Proposition~\ref{prop:intersect}, and in view of (\ref{eq:delta-cap}), we may focus on $W_{ij}$ for $i \in I_k$ and $j \in I_\ell$, with $k \neq \ell$, such that $\| \bx_i - \bx_j \| \leq \eps$ and with $\bx_j$ close to an intersection, specifically, for some $p \neq \ell$,
$$\dist(\bx_j, S_\ell \cap S_{p}) \leq \nu, \ \text{ where } \nu := (\sin \theta_{\rm int})^{-1} (\eps \wedge \rho_\sN \eta).$$
In fact, we show that, under the conditions of Proposition~\ref{prop:intersect}, with probability at least $1 -\gamma_\sN$, there is no such a pair of points $(\bx_i, \bx_j)$.
For fixed $(k, \ell, p)$, the probability that $\bx_i \sim \Psi_k$ and $\bx_j \sim \Psi_\ell$ satisfy these conditions is
\begin{equation} \label{eq:psi-cap}
\expect{\Psi_k(B(\bx_j, \eps)) {\bf 1}_{\{\bx_j \in B(S_\ell \cap S_p, \nu)\}}},
\end{equation}
after integrating over $\bx_i$.
By Lemma~\ref{lem:vol1},
$$
\Psi_k(B(\bx_j, \eps)) \prec \eps^d.
$$
where the implicit constant depends only on $\kappa, d$.  Moreover, by condition (\ref{eq:d-cap}),
$$
\Psi_\ell(B(S_\ell \cap S_p, \nu)) \prec \nu^{d-d_{\rm int}}.
$$
Therefore, using the union bound, the probability that there is such a pair of points is of order not exceeding
$$\sum_{k, \ell} N_k N_\ell \cdot \eps^d \nu^{d-d_{\rm int}} = N^2 \eps^d \nu^{d-d_{\rm int}} = \gamma_N\to 0.$$


\subsection{Proof of Propositions~\ref{prop:linear-outliers-1} and~\ref{prop:linear-outliers-2}}

Without loss of generality, we assume that $\delta_0$ is small and that $\eta \leq \eps/10$.
Let $\Psi_0$ be the uniform distribution on $(0,1)^D \setminus \bigcup_k B(S_k, \delta_0)$.  By Lemma~\ref{lem:vol1}, this set has $D$-volume of order $1 - O(K \delta_0^{D-d})$, with $K \delta_0^{D-d}$ small since $K$ is fixed.  Therefore, for $A \subset (0,1)^D$,
$$\Psi_0(A) \asymp \vol_D\left(A \setminus \bigcup_k B(S_k, \delta_0)\right).$$
Let $I_0 \subset \{1, \dots, N\}$ index the outliers and let $N_0$ be the number of outliers.

In view of how the procedures \ref{o1} and \ref{o2} work, we need to bound the degrees of non-outliers from below and the degrees of outliers from above.
The following lower bound holds
\begin{equation} \label{eq:N-linear-cond}
N_k \eps^d (1 \wedge (\eta/\tau))^{D-d} \geq (\rho_\sN/(K \zeta)) \log N, \ \forall k = 1, \dots, K.
\end{equation}
For \ref{o1}, it comes from \eqref{eq:m}-\eqref{eq:eps} and the fact that, for all $k \neq 0$, $N_k \geq N/(K \zeta \rho_\sN)$, since
$
N \leq K \zeta N_k + N_0, \ \text{ implying } \ N_k \geq (N -N_0)/(K \zeta),
$
and $N -N_0 \geq N/\rho_\sN$ in our assumptions.
For \ref{o2}, it comes from \eqref{eq:eps-eta-O2} and \eqref{eq:eps-lb} (and the inequality holds with $\rho_\sN$ in place of $\rho_\sN/(K \zeta)$).  In the same vein,
\begin{equation} \label{eq:N0-Nk}
N_k (1 \wedge (\eta/\tau))^{D-d} \gg N \eta^{D-d}, \quad \forall k=1, \dots,K.
\end{equation}
We prove a result that is more general than what we need now.
\begin{proposition} \label{prp:D}
Assume \eqref{eq:N-linear-cond} and \eqref{eq:N0-Nk}.  Then with probability at least $1 - N^{-\rho_\sN/(K \zeta)}$,
\begin{equation} \label{eq:D}
N_k \eps^d (1 \wedge (\eta/\tau))^{D-d} \prec D_{i}^{1/(m-1)} \prec N_k \eps^d (1 \wedge (\eta/\tau))^{(D-d)(1-\frac{d+1}{m-1})},
\end{equation}
uniformly over $i \in I_k$, $k \neq 0$; and also,
\begin{eqnarray}
D_{i}^{1/(m-1)} &\prec& (N -N_0) \eps^d (1 \wedge (\eta/\tau))^{(D-d)(1-\frac{d+1}{m-1})} \xi^{1-\frac{d+1}{m-1}} {\bf 1}_{\{\delta_0 \leq \eps + \tau\}} \nonumber \\
&& \qquad + N \eps^d \eta^{(D-d)(1-\frac{d+1}{m-1})}, \label{eq:D-outlier}
\end{eqnarray}
uniformly over $i  \in I_0$, where $\xi = 1$ if $\tau \geq \eps$, and $\xi = 1 \wedge (\eta/\delta_0)$, otherwise.
\end{proposition}

\begin{proof}
Define
$$H_{i,\eps,\eta} = \max_{L \in \cA_d} H_{i,\eps,\eta}(L).$$
Proposition~\ref{prp:D} is a direct consequence of Lemma~\ref{lem:DGH} which relates $D_i$ to $G_{i,\eps}$ (defined in Section~\ref{sec:proof-linear}) and $H_{i,\eps,\eta}$, and of Propositions~\ref{prp:G-outlier} and~\ref{prp:H-outlier} (together with \eqref{eq:N0-Nk}), which bound $G_{i,\eps}$ and $H_{i,\eps,\eta}$, respectively.
\end{proof}

\begin{lemma} \label{lem:DGH}
There is a constant $C > 0$ such that
\begin{equation} \label{eq:DGH}
H_{i,\eps/2,\eta}^{\{m-1\}} \leq D_i \leq G_{i,\eps}^{\{d+1\}} (H^*_{i,\eps,C\eta})^{\{m-d-2\}}.
\end{equation}
\end{lemma}

\begin{proof}
We get the upper bound by following the arguments in the proof of (\ref{eq:W}).
For the lower bound, we simply have
\begin{eqnarray*}
D_i
& \geq & \sum_{M: |M| = m-1} {\bf 1}_{\{\exists L \in L_d: \bx_{j} \in B(\bx_i, \eps/2) \cap B(L, \eta), \forall j \in M\}} \\
& \geq & H_{i,\eps/2,\eta}^{\{m-1\}}.
\end{eqnarray*}
\end{proof}

The bounds for $G_{i,\eps}$ and $H_{i,\eps,\eta}$ that follow are more general than needed at this point.  In particular, the case of large $\tau$ will only be useful in Section~\ref{sec:proof-estim}.

\begin{proposition} \label{prp:G-outlier}
Assume \eqref{eq:N-linear-cond} holds with $\eps$ in place of $\eta$.  Then with probability at least $1 - N^{-\rho_\sN/(K \zeta)}$,
\begin{equation} \label{eq:G-noutlier}
N_k \eps^d (1 \wedge (\eps/\tau))^{D-d} \prec G_{i, \eps} \prec N_k \eps^d (1 \wedge (\eps/\tau))^{D-d} + N_0 \eps^D.
\end{equation}
uniformly over $i \in I_k$ and $k = 1, \dots, K$.
Also,
\begin{equation} \label{eq:G-outlier}
G_{i, \eps} \prec (N -N_0) \eps^d (1 \wedge (\eps/\tau))^{D-d} {\bf 1}_{\{\delta_0 \leq \eps + \tau\}} + N \eps^D.
\end{equation}
uniformly over $i \in I_0$
\end{proposition}
\begin{proof}
The proof is similar to that of Proposition~\ref{prp:rG}.  We bound $G_{i, \eps}$ in expectation.  Suppose $i \in I_k$ with $k \neq 0$.  Then by Lemma~\ref{lem:vol1}
$$\expect{G_{i, \eps}} \geq N_k \Psi_kB(\bx_i, \eps) \asymp N_k \eps^d (1 \wedge (\eps/\tau))^{D-d}.$$
For the upper bound, by Lemma~\ref{lem:vol1} and the simple bound
$$\Psi_0(B(\bx_i, \eps)) \prec \eps^D,$$
we have
$$\expect{G_{i, \eps}} = \sum_\ell N_\ell \Psi_\ell(B(\bx_i, \eps)) \prec (N - N_0) \eps^d (1 \wedge (\eps/\tau))^{D-d} + N_0 \eps^D,$$
with $N -N_0 \leq (K \zeta) N_k$ for any $k \neq 0$.
As in as in Proposition~\ref{prp:rG}, we then use Lemma~\ref{lem:hoeff} together with \eqref{eq:N-linear-cond} and the union bound, to conclude the proof of \eqref{eq:G-noutlier}.
The proof of \eqref{eq:G-outlier} is identical, except that, when $\delta_0 > \tau + \eps$, we have $\Psi_\ell(B(\bx_i, \eps)) = 0$ if $\ell \neq 0$ and $i \in I_0$.
\end{proof}

\begin{proposition} \label{prp:H-outlier}
If \eqref{eq:N-linear-cond} holds, then with probability at least $1 - N^{-\rho_\sN/(K \zeta)}$,
\begin{equation} \label{eq:H-noutlier}
H_{i,\eps/2,\eta} \succ N_k \eps^d (1 \wedge (\eta/\tau))^{D-d}, \quad H^*_{i,\eps/2,\eta} \prec N_k \eps^d (1 \wedge (\eta/\tau))^{D-d} + N_0 \eps^d \eta^{D-d},
\end{equation}
uniformly over $i \in I_k$ and $k \neq 0$; and also,
\begin{equation} \label{eq:H-outlier}
H^*_{i,\eps/2,\eta} \prec (N -N_0) \eps^d (1 \wedge (\eta/\tau))^{D-d} \xi {\bf 1}_{\{\delta_0 \leq \eps + \tau\}} + N \eps^d \eta^{D-d}.
\end{equation}
uniformly over $i \in I_0$
\end{proposition}

\begin{proof}
First assume that $i \in I_k$ with $k \neq 0$.
For the lower bound in \eqref{eq:H-noutlier}, let $L$ be a subspace such that
$$\Psi_k(B(\bx_i, \eps) \cap B(L, \eta)) \succ \eps^d (1 \wedge (\eta/\tau))^{D-d},$$
which exists by the lower bound in Lemma~\ref{lem:vol2}.
We have $H_{i,\eps,\eta} \geq H_{i,\eps,\eta}(L)$, and the term on the right hand side is a sum of independent Bernoulli random variables with expectation
$$\expect{H_{i,\eps,\eta}(L)} = N_k \Psi_k(B(\bx_i, \eps) \cap B(L, \eta)) \succ N_k  \eps^d (1 \wedge (\eta/\tau))^{D-d}.$$
We then apply Lemma~\ref{lem:hoeff}, using (\ref{eq:N-linear-cond}), and the union bound.
For the upper bound in \eqref{eq:H-noutlier}, the arguments are the same as in the proof of \eqref{eq:H}, except for the following bound in expectation, valid for any $L \in \cA_d$,
\begin{eqnarray*}
\expect{H_{i,\eps,\eta}(L)}
&=& \sum_\ell N_\ell \Psi_\ell(B(\bx_i, \eps) \cap B(L, \eta)) \\
&\prec& (N -N_0) \eps^d (1 \wedge (\eta/\tau))^{D-d} + N_0 \eps^d \eta^{D-d},
\end{eqnarray*}
by Lemmas~\ref{lem:vol2} and~\ref{lem:vol2-out}.

Now, assume that $i \in I_0$.  Again, the arguments are the same as in the proof of \eqref{eq:H}, except that the bounds in expectation are different.  Specifically, if $\delta_0 > \eps + \tau$, then $\Psi_\ell(B(\bx_i, \eps) \cap B(L, \eta)) = 0, \ \forall \ell \neq 0$, so that, by Lemma~\ref{lem:vol2-out}, for any $L \in \cA_d$,
$$\expect{H_{i,\eps,\eta}(L)} \prec N_0 \eps^d \eta^{D-d}.$$
Otherwise,
$$\expect{H_{i,\eps,\eta}(L)} \prec (N -N_0) \eps^d (1 \wedge (\eta/\tau))^{D-d} \xi + N_0 \eps^d \eta^{D-d}.$$
\end{proof}

We are now in a position to prove Propositions~\ref{prop:linear-outliers-1} and~\ref{prop:linear-outliers-2}.  We first consider \ref{o1}.  By (\ref{eq:D}) and (\ref{eq:D-outlier}), and the fact that $\tau \leq \eta \leq \rho_\sN^{-3/(D-d)}$, we have
\[
\max_i D_{i}^{1/(m-1)} \prec (N -N_0) \eps^d \prec (N/\rho_\sN) \eps^d.
\]
On the one hand, by (\ref{eq:D}), $D_{i}^{1/(m-1)} \succ N_k \eps^d \succ (N/\rho_\sN) \eps^d$, uniformly over $i \in I_k, \ \forall k \neq 0$.  Hence, since $\rho_\sN \to \infty$, no non-outlier is identified as an outlier.
On the other hand, by (\ref{eq:D-outlier}), for any $i \in I_0$,
$$D_i^{1/(m-1)} \prec N \eps^d (\xi^{1 - \frac{d+1}{m-1}} + \eta^{D-d-\frac{d+1}{m-1}}) \ll N \eps^d/\rho_\sN^{2},$$
since $\xi \prec \eta/\delta_0 \prec \rho_\sN^{-3}$ and $\eta \leq \eps \leq \rho_\sN^{-3/(D-d)}$.  Hence, all outliers are identified as such.

We now consider \ref{o2}.
On the one hand, by (\ref{eq:D}) and \eqref{eq:eps-lb}, and the expression for $\eps$ and $\eta$, we have
\[
D_{i}^{1/(m-1)} \succ N_k \eps^d (1 \wedge (\eta/\tau))^{D-d} \succ \rho_\sN^3 \log N \succ \rho_\sN^2 N \eps^d \eta^{D-d},
\]
uniformly over $k \neq 0$ and $i \in I_k$.  Hence, no non-outlier is identified as an outlier.
On the other hand, by (\ref{eq:D-outlier}), for any $i \in I_0$,
$$D_i^{1/(m-1)} \prec N \eps^d \eta^{D-d-\frac{d+1}{m-1}} \prec N \eps^d \eta^{D-d},$$
which comes from $m \gg \log(N)/\log(\rho_\sN)$.  Hence, all outliers are identified as such.

\section{Proofs for the Estimation of Parameters}
\label{sec:proof-estim}

\subsection{Proof of Proposition~\ref{prop:tau-1}}

Recalling the definition of $G_{i, \eps}$ in (\ref{eq:G-def}), we have
$${\rm Cor}(\eps) = \sum_{i} G_{i, \eps}.$$

Let $\eps_r = \rho_\sN^{-r}$ and let $r_0$ be the integer defined by $\eps_{r_0+1} < \tau \leq \eps_{r_0}$.
Define
$$r_\sN^* :=  ((1-d/D) r_0 + (d/D) r_\sN)) \wedge r_\sN,$$
and note that, for $r \leq r_\sN^*$, (\ref{eq:N-linear-cond}) with $\eps$ in place of $\eta$ is satisfied for $\eps_r$.
As there are only order $\log N$ such $r$'s, Proposition~\ref{prp:G-outlier} and the union bound imply that, with probability at least $1 - \log(N) N^{-\rho_\sN/(K \zeta)}$,
$$
(N/\rho_\sN)^2 \eps_r^d (1 \wedge (\eps_r/\tau))^{D-d} \prec {\rm Cor}(\eps_r) \prec (N/\rho_\sN)^2 \eps_r^d (1 \wedge (\eps_r/\tau))^{D-d},
$$
uniformly over $r \leq r_\sN^*$.  Note that we used the fact that $N^2 \eps_r^D \ll (N/\rho_\sN)^2 \eps_r^d (1 \wedge (\eps_r/\tau))^{D-d}$, which holds since $r, r_0 \geq 3$.
When this is the case,
$$A_r = \left\{\begin{array}{ll}
2\log N - d r \log \rho_\sN + O(1), & r \leq r_0; \\
2\log N - D r \log \rho_\sN - (D - d) \log \tau + O(1), & r > r_0.
\end{array}\right.$$
In particular, for $r \leq r_\sN^*$,
$$\frac{A_r - A_{r+1}}{\log \rho_\sN} = \left\{\begin{array}{ll}
d + o(1), & r \leq r_0 - 1; \\
D + o(1), & r \geq r_0 + 1.
\end{array}\right.$$
From the first part, we see that $\rhat \geq r_0 \wedge (r_\sN - \lceil 2D/d \rceil)$, since $d \leq D-1$ and $\rho_\sN \to \infty$.  To use the second part, note that $r_0 + 2 \leq r_\sN^*$ if, and only if, $r_0 \leq r_\sN - \lceil 2D/d \rceil$.
If this is the case, $\rhat \leq r_0+1$.  From this follows the statement in  Proposition~\ref{prop:tau-1}.

\subsection{Proof of Proposition~\ref{prop:tau-2}}

We follow the proof of Proposition~\ref{prop:tau-1}.
We assume that $\dhat = d$, which happens with probability tending to one.  Let $\eta_s = \rho_\sN^{-\rhat-s}$ and $s_0 = r_0 - \rhat$.  Define
$$s_\sN^* :=  ((2Dd + d -2)/(D-d) + s_0) \wedge (\rhat - 1),$$
and note that, for $s \leq s_\sN^*$, (\ref{eq:N-linear-cond}) is satisfied for $\eps_{\rhat}$ and $\eta_s$.  Indeed, using the fact that $\eps_{\rhat} \geq (\log(N)/N)^{1/d} \rho_\sN^{2D + 1}$ and $\tau \leq \rho_\sN^{-r_0}$, we get
\begin{eqnarray*}
N_k \eps_{\rhat}^d (1 \wedge (\eta_s/\tau))^{D-d}
&\geq& (N/(K \zeta) \rho_\sN) (\log(N)/N) \rho_\sN^{(2D+1) d} (1 \wedge \rho_\sN^{(s_0-s)(D-d)}) \\
&=& \rho_\sN \log(N) \cdot \rho_\sN^{-2 + (2D+1) d - (D-d)(s-s_0)_+},
\end{eqnarray*}
and the exponent in $\rho_\sN$ is non-negative by the upper bound on $s$.
As there are only order $\log N$ such $s$'s, Proposition~\ref{prp:D} and the union bound imply that, with probability at least $1 - \log(N) N^{-\rho_\sN/(K \zeta)}$,
\begin{eqnarray*}
{\rm Cor}(\eps_{\rhat}, \eta_s) &\succ& (N/\rho_\sN)^2 \zeta^{-1} \eps_{\rhat}^d (1 \wedge (\eta_s/\tau))^{D-d}, \\
{\rm Cor}(\eps_{\rhat}, \eta_s) &\prec& (N/\rho_\sN)^2 \eps_{\rhat}^d (1 \wedge (\eta_s/\tau))^{D-d - (d+1)/(m-1)},
\end{eqnarray*}
uniformly over $s \leq s_\sN^*$.  Note that we used the fact that
$$N^2 \eps_{\rhat}^d \eta_s^{D-d} \ll (N/\rho_\sN)^2 \eps_{\rhat}^d (1 \wedge (\eta_s/\tau))^{D-d}.$$
When this is the case,
$$B_s = 2\log N - d \rhat \log \rho_\sN + O(1),$$
when $s \leq s_0$, and
$$B_s = 2\log N - D \rhat \log \rho_\sN + (-(D-d) + O(1/m)) (s \log \rho_\sN + \log \tau) + O(1),$$
when $s > s_0$.
In particular, for $s \leq s_\sN^*$,
$$\frac{B_s - B_{s+1}}{\log \rho_\sN} = \left\{\begin{array}{ll}
o(1), & s \leq s_0 - 1; \\
D-d + o(1), & s = s_0 + 1.
\end{array}\right.$$
From here the arguments are parallel to those used in Proposition~\ref{prop:tau-1}.


\section*{Acknowledgements}

GC was at the University of Minnesota, Twin Cities, for part of the project.
The authors would like to thank the Institute for Mathematics and its Applications (IMA), in particular Doug Arnold and Fadil Santosa, for holding a stimulating workshop on multi-manifold modeling that GL co-organized, and EAC and GL participated in.  The authors also thank Jason Lee for providing the last dataset in Figure~\ref{fig:artificial_data}.
Finally, the authors are grateful to two anonymous referees and Associate Editor for providing constructive feedback and criticism, which helped improve the presentation of the paper.
This work was partially supported by grants from the National Science Foundation (DMS-06-12608, DMS-09-15160, DMS-09-15064) and a grant from the Office of Naval Research (N00014-09-1-0258).

\bibliographystyle{abbrv} 

s

\end{document}